\documentclass[twoside,11pt]{article}
\usepackage{jmlr2e}

\usepackage{booktabs}
\usepackage{url}
\usepackage{amsmath,amsfonts}
\usepackage{nicefrac}
\usepackage{microtype}
\usepackage{xcolor}
\usepackage{setspace}
\usepackage{mathrsfs}
\usepackage{graphicx}
\usepackage{multirow}
\usepackage{tikz}
\usetikzlibrary{positioning,fit,backgrounds,calc,arrows.meta,shapes.geometric}
\usepackage{times}
\usepackage{pgfplots}
\pgfplotsset{compat=1.17}
\usepgfplotslibrary{groupplots}
\usepackage{paralist}
\usepackage{enumitem}
\usepackage[linesnumbered,ruled,vlined,algo2e]{algorithm2e}
\usepackage{placeins}
\usepackage{algpseudocode}

\definecolor{cCycle}{HTML}{1F77B4}   
\definecolor{cStar} {HTML}{2CA02C}   
\definecolor{cPath} {HTML}{E0A82E}   
\definecolor{cKfour}{HTML}{D1622B}   
\definecolor{cMiss} {HTML}{555555}   
\definecolor{cFrame}{HTML}{BFBFBF}   

\newtheorem{assumption}[theorem]{Assumption}

\SetKwInput{KwInput}{Input}
\SetKwInput{KwOutput}{Output}

\ShortHeadings{Spectral Embeddings Leak Graph Topology}{Nguyen-Cong, Hy, and Dinh}

\firstpageno{1}

\begin{document}

\title{Spectral Embeddings Leak Graph Topology: Theory, Benchmark, and Adaptive Reconstruction}

\author{\name Thinh Nguyen-Cong \email nguyencongt@vcu.edu \\
       \addr Department of Computer Science\\
       Virginia Commonwealth University\\
       Richmond, VA, USA
       \AND
       \name Truong-Son Hy\thanks{Correspondence to thy@uab.edu and tndinh@vcu.edu} \email thy@uab.edu \\
       \addr Department of Computer Science\\
       The University of Alabama at Birmingham\\
       Birmingham, AL, USA
       \AND
       \name Thang N. Dinh\footnotemark[1] \email tndinh@vcu.edu \\
       \addr Department of Computer Science\\
       Virginia Commonwealth University\\
       Richmond, VA, USA}

\editor{TBD}

\maketitle

\begin{abstract}
Graph Neural Networks (GNNs) excel on relational data, but standard benchmarks unrealistically assume the graph is centrally available. In practice, settings such as Federated Graph Learning, distributed systems, and privacy-sensitive applications involve graph data that are localized, fragmented, noisy, and privacy-leaking. We present a unified framework for this setting. We introduce \textbf{LoGraB} (\underline{Lo}cal \underline{Gra}ph \underline{B}enchmark), which decomposes standard datasets into fragmented benchmarks using three strategies and four controls: neighborhood radius $d$, spectral quality $k$, noise level $\sigma$, and coverage ratio $p$. LoGraB supports graph reconstruction, localized node classification, and inter-fragment link prediction, with \emph{Island Cohesion}. We propose \textbf{AFR} (Adaptive Fidelity-driven Reconstruction), a method for noisy spectral fragments. AFR scores patch quality via a fidelity measure combining a gap-to-truncation stability ratio and structural entropy, then assembles fragments using RANSAC-Procrustes alignment, adaptive stitching, and Bundle Adjustment. Rather than forcing a single global graph, AFR recovers large faithful islands. We prove heat-kernel edge recovery under a separation condition, Davis--Kahan perturbation stability, and bounded alignment error. We establish a \emph{Spectral Leakage Proposition}: under a spectral-gap assumption, polynomial-time Bayesian recovery is feasible once enough eigenvectors are shared, complementing AFR's deterministic guarantees. Experiments on nine benchmarks show that LoGraB reveals model strengths and weaknesses under fragmentation, AFR achieves the best F1 on 7/9 datasets, and under per-embedding $(\epsilon,\delta)$-Gaussian differential privacy, AFR retains 75\% of its undefended F1 at $\epsilon=2$. Our anonymous code is available at \url{https://anonymous.4open.science/r/JMLR_submission}.
\end{abstract}


\section{Introduction}
\label{sec:introduction}

During the past decade, Graph Neural Networks (GNNs) have become the dominant paradigm for learning relational data \citep{ZHOU202057}, demonstrating remarkable success in benchmarks such as Cora \citep{jsg4-wp31-24}, CiteSeer \citep{Giles1998CiteSeer}, PubMed \citep{Sen2008PubMed}, and the Open Graph Benchmark (OGB) suite \citep{Hu2020OGB}. However, these benchmarks operate under an idealized assumption: algorithms have complete centralized access to the entire graph. This experimental setting clashes with the reality of modern applications. In federated graph learning, data is fragmented across countless clients, each holding only local subgraphs. Privacy regulations require hiding or obscuring structural information. Similarly, to defend against leakage attacks, privacy-preserving systems often provide only partial spectral summaries, local embeddings, or truncated eigenvectors, instead of the raw adjacency data that GNNs were designed to consume.

This gap between idealized benchmarks and real-world constraints creates two intertwined problems. First, we lack a rigorous understanding of the extent to which locality, fragmentation, and noise impact the predictive power of GNN architectures. Without a controlled testbed, researchers rely on ad-hoc data splits and incomparable protocols, stifling progress in robust graph learning. Second, in Federated Graph Learning (FGL), the exchange of structural representations creates an important channel for privacy leakage. Intermediate signals, gradients, embeddings, spectral summaries can inadvertently leak sensitive information. In FGL, the primary asset to protect is not just the node features but the local graph topology, which can encode sensitive relationships. Recent comprehensive evaluations have documented this threat through various reconstruction and inference attacks \citep{Chen2024}, with emerging work showing leakage of sensitive node properties \citep{Liu20251857} and local label distributions from shared representations.

Spectral embeddings, a common currency for exchanging structural information, constitute an important passive leakage channel. Unlike active attacks that risk detection, a passive attack requires only observing a single, honestly shared artifact. The vulnerability lies not in the federated protocol, but in the intrinsic properties of the representation itself. Prior patch-based reconstruction approaches typically assume high-quality or uniformly-reliable patches, so they do not exploit the substantial quality heterogeneity that arises under non-uniform truncation and noise. This motivates a reconstruction algorithm that adaptively prioritizes the most reliable patches and requires more evidence before committing to alignments of lower-fidelity pairs.

This paper addresses both problems with a unified framework. We make the following contributions, organized into three complementary thrusts:

\begin{enumerate}[nosep]
\item \emph{Problem formalization and benchmark (Sections~\ref{sec:problem}--\ref{sec:lograb}).}
\begin{itemize}[nosep]
\item We formalize the \emph{fragmented-graph learning} setting, defining a threat model for passive spectral leakage, and establishing the notation for localized spectral views (\S\ref{sec:problem}).
\item We introduce \textbf{LoGraB} (Local Graph Benchmark), a benchmark that systematically explores the joint space of fragmentation, spectral truncation, and noise in graph learning. LoGraB provides three decomposition strategies, four controllable parameters, three benchmark tasks, and new metrics including Island Cohesion (\S\ref{sec:lograb}).
\item We establish an open-source evaluation protocol with dedicated baselines, ensuring fair and reproducible comparisons across 7~GNN architectures (with 4~positional-encoding variants) and 18~fragmentation scenarios on 9~diverse datasets spanning citation networks, social graphs, biological structures, molecular contacts, and graph-level vision benchmarks.
\end{itemize}

\item \emph{The AFR method (Section~\ref{sec:afr}).}
\begin{itemize}[nosep]
\item We introduce a novel \emph{fidelity score} that quantifies the reliability of leaked graph patches by combining a gap-to-truncation stability ratio with structural entropy, providing a quantitative foundation for robust decision-making.
\item We propose \textbf{Adaptive Fidelity-driven Reconstruction} (AFR), an adaptive multi-stage algorithm that uses fidelity scores to guide a robust assembly process. Key components include outlier-resilient RANSAC-Procrustes alignment, adaptive stitching criteria, Bundle Adjustment refinement, and inter-island link inference via cross-voting.
\end{itemize}

\item \emph{Theoretical analysis and reproducibility (Sections~\ref{sec:problem}--\ref{sec:afr} and Appendices).}
\begin{itemize}[nosep]
\item We establish an \emph{information-theoretic feasibility result} (Proposition~\ref{thm:spectral_leakage}, the Spectral Leakage Proposition): under a spectral-gap assumption, a polynomial-time Bayesian recovery procedure exists once $k \ge k^*(\varepsilon)$. We emphasize this result is distinct from AFR: it guarantees that \emph{some} polynomial-time procedure recovers the graph, whereas AFR is a separate deterministic construction. We derive proxy-threshold formulas for common eigenvalue decay profiles and qualitatively characterize the defense gap introduced by differential privacy noise (\S\ref{sec:problem:leakage}).
\item We provide formal local guarantees for AFR's components, all of which stand on their own and do not rely on the Spectral Leakage Proposition: a heat kernel separation gap lemma justifying local reconstruction, spectral perturbation bounds, exact edge recovery under a separation condition, patch alignment error bounds, error accumulation analysis motivating Bundle Adjustment, RANSAC convergence, stitch soundness, and a cohesion--recall relationship, all with complete proofs in Appendix~\ref{appendix:proofs}.
\end{itemize}
\end{enumerate}

By modeling real-world constraints, from federated ownership and privacy filters to spectral leakage, our framework gives the community essential tools for two critical lines of research: building robust, privacy-aware graph learning algorithms, and establishing controlled environments to design and test defenses against reconstruction attacks.

The paper follows a ``threat $\to$ measurement $\to$ exploitation $\to$ evidence'' arc. Section~\ref{sec:related} reviews related work. Section~\ref{sec:problem} establishes the spectral leakage problem: we formalize the fragmented-graph setting, define the passive threat model, and state the Spectral Leakage Proposition, an information-theoretic feasibility result showing that polynomial-time graph recovery is possible whenever sufficiently many eigenvectors are shared. Section~\ref{sec:lograb} introduces LoGraB, the benchmark that operationalizes this threat into a controlled, reproducible testbed with three tasks and tailored metrics. Section~\ref{sec:afr} then presents AFR, a constructive algorithm that demonstrates reconstruction is practically achievable even in noisy, fragmented data where idealized methods collapse; AFR's heat-kernel edge-recovery bound (Theorem~\ref{thm:edge_recovery}) is rigorously proved and provides a \emph{stricter local} guarantee than the Bayesian existence result of Section~\ref{sec:problem:leakage}, though by a different algorithmic route. We present LoGraB before AFR because the benchmark defines the evaluation framework against which AFR (and any future method) is measured. Section~\ref{sec:experiments} reports representative evaluations across 9~datasets drawn from the 486-scenario benchmark grid, and Section~\ref{sec:conclusion} concludes.

\section{Related Work}
\label{sec:related}

Our work sits at the intersection of three research threads: graph benchmarks and their limitations (\S\ref{sec:related:centralized}--\ref{sec:related:federated}), privacy threats and defenses in federated graph learning (\S\ref{sec:related:privacy}, \S\ref{sec:related:defenses}), and the algorithmic foundations that enable robust graph reconstruction (\S\ref{sec:related:reconstruction}--\ref{sec:related:patch}). We review each in turn, highlighting the gaps that motivate LoGraB and AFR.

\subsection{Centralized Graph-Learning Benchmarks}
\label{sec:related:centralized}

Early empirical studies of GNNs relied on small citation graphs such as Cora, CiteSeer, and PubMed \citep{Sen2008PubMed}. The Open Graph Benchmark (OGB) suite became the de-facto standard for measuring scalability by providing large-scale, task-diverse datasets \citep{Hu2020OGB}. Specialized benchmarks have addressed long-range dependencies \citep{Dwivedi2022LRGB} and general GNN evaluation \citep{Dwivedi2023}. However, all of these benchmarks share a critical blind spot: they provide every algorithm with a complete centralized adjacency matrix. Consequently, they fail to test the robustness to fragmented or privacy-filtered views that are increasingly common in practice.

\subsection{Federated and Distributed Graph Benchmarks}
\label{sec:related:federated}

Frameworks like FedGraphNN \citep{He2021FedGraphNN} and OpenFGL \citep{Li2025OpenFGL} tackle data silos by partitioning graphs among clients. However, these frameworks operate under a critical idealization: each client has perfect knowledge of its local subgraph. They do not model the effects of controlled noise or severe information loss when only spectral summaries are available. LoGraB studies a more information-constrained setting: beyond partitioning the graph, it introduces spectral truncation and noise, forcing models to operate on localized spectral views rather than perfect local subgraphs.

\subsection{Privacy Threats in Federated Graph Learning}
\label{sec:related:privacy}

The premise that sharing intermediate representations in Federated Learning is safe has been systematically challenged. Deep Leakage from Gradients (DLG) demonstrated that raw training data can be recovered from shared gradients \citep{Zhu2019DeepLeakage, Geiping2020}. Within graph learning, link inference attacks \citep{He2021stealing, Wu2022Linkteller, Meng2023}, graph reconstruction attacks \citep{Zhang2021GraphMI}, and membership inference have all been demonstrated \citep{Chen2024}. The severity of these threats has led to dedicated defenses \citep{Zhong2025GuardFGL}. Comprehensive surveys have formalized the taxonomy of privacy attacks on GNNs \citep{Zhang2024Survey}. AFR targets a fundamentally different leakage channel from gradient-based attacks: not dynamic gradients, but static spectral embeddings of the graph structure. This passive eavesdropper threat model represents a more fundamental vulnerability, as the leakage is an intrinsic property of the honestly shared data artifact itself.

\subsection{Graph Reconstruction from Local Embeddings}
\label{sec:related:reconstruction}

The fundamental feasibility of inverting spectral representations is supported by theoretical results demonstrating that, under certain identifiability conditions, graphs are reconstructible from their eigenspaces \citep{JMLR:v18:16-146}. Classical patch-based approaches \citep{Cucuringu2012,Jeub2022local2global} assume patches of uniform quality and do not actively model per-patch reliability; in the presence of heterogeneous noise and truncation levels, a single low-quality patch can disproportionately degrade the assembled result. AFR is designed to address this by quantifying per-patch fidelity and adapting the stitching criterion to the quality of each pair, enabling more robust reconstruction when some patches are substantially noisier than others. Conceptually, the problem can be framed as inverting a structural autoencoder: while a Variational Graph Auto-Encoder \citep{Kipf2016VGAE} learns a mapping from a graph to an embedding (and uses its own decoder to reconstruct), our challenge is to recover the graph from a collection of \emph{corrupted localized spectral embeddings} that were not produced by a learned encoder and that have no matching learned decoder.

\subsection{Robust Geometric Alignment}
\label{sec:related:geometric}

AFR's resilience is inspired by techniques from computer vision. Standard Orthogonal Procrustes analysis \citep{Schoenemann1966} is notoriously sensitive to outliers. We replace it with RANSAC-Procrustes \citep{Fischler1981RANSAC}, a paradigm designed to find reliable fits in the presence of faulty data. To correct accumulated alignment errors, AFR employs Bundle Adjustment \citep{Triggs2000Bundle}, a cornerstone of structure-from-motion.

\subsection{Patch-Based Alignment and Synchronization}
\label{sec:related:patch}

Divide-and-conquer algorithms, such as eigenvector synchronization \citep{Cucuringu2012} and Local2Global \citep{Jeub2022local2global} have shown that aligning local embeddings is a powerful strategy for learning global representations. However, these were designed as algorithmic solutions, not as experimental testbeds. They answer ``How can we reconstruct the graph?'' while leaving the more fundamental question ``How well do models perform when the graph cannot be perfectly reconstructed?'' unanswered. LoGraB fills this gap.

\subsection{Defenses Against Information Leakage in FGL}
\label{sec:related:defenses}

Significant privacy risks have spurred the development of defense mechanisms. A prominent line involves applying Differential Privacy (DP), including noise addition to shared parameters \citep{Qiu2022DP} or the perturbation of the aggregation function within GNNs \citep{Sajadmanesh2023}. Information-theoretic approaches use the information bottleneck to learn minimal, privacy-preserving subgraphs \citep{Zhang2023}. Hybrid methods combine techniques like Local Differential Privacy with secure hardware \citep{Han2024}. These defensive strategies underscore the urgent need for realistic threat models such as AFR to properly evaluate their effectiveness.

\subsection{The Gap This Paper Fills}

In summary, existing benchmarks operate in one of two idealizations: a perfect global graph (OGB) or perfect local subgraphs (FedGraphNN). Existing attack models for spectral reconstruction typically assume uniform patch quality, which makes their results fragile when reliability varies across patches. Our work addresses both gaps. LoGraB degrades the data in a controlled way, not merely partitioning it, providing controlled noise, truncation, and coverage loss. AFR is a reconstruction algorithm that explicitly models patch-level fidelity heterogeneity, using fidelity-aware assembly to succeed where uniform-quality methods struggle. Together, they provide a research framework for studying robust, privacy-aware graph learning.

\section{The Spectral Leakage Problem}
\label{sec:problem}

This section establishes the theoretical foundations of our work. We formalize the fragmented-graph setting (\S\ref{sec:problem:setting}), define the passive spectral leakage threat model (\S\ref{sec:problem:threat}), introduce the notation (\S\ref{sec:problem:notation}), and prove that spectral embeddings are fundamentally leaky (\S\ref{sec:problem:leakage}), establishing the theoretical necessity of both the LoGraB benchmark and the AFR attack algorithm that follow.

\subsection{Fragmented Graphs and Spectral Views}
\label{sec:problem:setting}

Let $G=(V,E)$ be an undirected graph with $|V|=n$ nodes. In the fragmented-graph setting, no single agent observes $G$ in its entirety. Instead, each participant has a local patch $P_v \subseteq V$ centered on some node $v$, consisting of all nodes within a bounded neighborhood:
\[
P_v = \{u \in V \mid \mathrm{dist}_G(u,v) \le d\},
\]
where $d$ is the neighborhood radius. The collection of patches $\mathcal{P} = \{P_1, \ldots, P_m\}$ forms an overlapping cover of (a subset of) $V$.

Each patch is observed not as a raw subgraph, but through a spectral lens. For patch $P_v$, we compute the normalized Laplacian $\mathcal{L}_v$ and retain only the bottom $k$ eigenvectors, forming a truncated spectral embedding $\mathcal{P}_v = [u_1, \ldots, u_k] \in \mathbb{R}^{|P_v| \times k}$. This embedding may be further corrupted by additive Gaussian noise: $\mathcal{P}_v \leftarrow \mathcal{P}_v + E$, where $E_{ij} \sim \mathcal{N}(0, \sigma^2)$.

The following definition formalizes this local spectral embedding and the fundamental sign/rotation ambiguity that any reconstruction algorithm must resolve.

\begin{definition}[Local spectral embedding]
\label{def:local_spectral_embedding}
Let $G=(V,E)$ be a finite, connected, undirected, simple graph with $n=|V|$ vertices. Fix an integer radius $d\ge 1$ and the number of retained eigenvectors $k\ge 1$. For each vertex $v\in V$, the $d$-hop neighborhood is $N_d(v)=\{u\in V : \mathrm{dist}_G(u,v)\le d\}$. Let $G_v^{(d)}=(N_d(v), E_v)$ be the induced subgraph with $q_v=|N_d(v)|$ nodes. The eigendecomposition of its Laplacian $\mathcal{L}_v$ yields eigenvalues $0\le \lambda_1(\mathcal{L}_v)\le \cdots \le \lambda_{q_v}(\mathcal{L}_v)$ with orthonormal eigenvectors $u_1(v),\ldots,u_{q_v}(v)$. The $k$-dimensional local spectral observation for vertex $v$ is the tuple
\[
\bigl(\mathcal{P}_v,\, \lambda_{k+1}^{(v)}\bigr), \qquad \mathcal{P}_v = [u_1(v) \mid \cdots \mid u_k(v)] \in \mathbb{R}^{q_v \times k},
\]
where $\mathcal{P}_v$ is observed only up to an arbitrary orthogonal transformation $\mathcal{Q}\in O(k)$ (capturing the sign/permutation ambiguity of eigenvectors, and the general rotation ambiguity for repeated eigenvalues), and $\lambda_{k+1}^{(v)}$ is the first truncated eigenvalue (needed to compute the spectral gap $\delta_v = \lambda_{k+1}^{(v)} - \lambda_k^{(v)}$). Including $\lambda_{k+1}^{(v)}$ reflects realistic deployments where truncated eigensystems are shared together with the truncation cutoff; it does not reveal the eigenvectors beyond the first $k$. The sign/rotation ambiguity must be resolved during alignment.
\end{definition}

\subsection{Threat Model: Passive Spectral Leakage}
\label{sec:problem:threat}

We consider a passive eavesdropper who observes spectral embeddings that are honestly shared as part of a federated graph learning protocol. Unlike active attacks (e.g. gradient manipulation), this adversary does not interfere with the protocol; the leakage is an intrinsic property of the shared data artifact itself.

The vulnerability of spectral embeddings manifests under two progressively challenging scenarios. In the \emph{global leakage} scenario, an adversary intercepts a single, complete spectral embedding $\mathcal{P}\in\mathbb{R}^{n\times k}$ of the entire client graph. This corresponds to settings where a client shares a global structural summary with a central server for tasks like model aggregation. Even in this seemingly abstract case, the adversary faces a highly ill-posed inverse problem due to the factorial-scale search space created by eigenvector permutation and sign ambiguity. In the local leakage scenario, the primary focus of this paper and the setting addressed by AFR, the adversary observes a collection of fragmented local embeddings $\{\mathcal{P}_v\}_{v\in V_{\mathrm{obs}}}$, each computed from a small $d$-hop neighborhood. The adversary must not only invert each local embedding but also solve a complex geometric alignment puzzle to stitch these overlapping pieces into a coherent global structure. Analyzing both scenarios provides a holistic view of the risk: we demonstrate that the global case (Section~\ref{sec:problem:leakage}) establishes fundamental information-theoretic limits on when reconstruction is possible, while the local case (Section~\ref{sec:afr}) shows these limits can be approached algorithmically.

This threat model is fundamentally more difficult than idealized reconstruction because: (i) the spectral embeddings are truncated to $k \ll n$ eigenvectors, discarding high-frequency structural information; (ii) additive noise corrupts every entry; and (iii) each patch's eigenvectors are defined only up to an arbitrary orthogonal transformation, creating a sign/rotation ambiguity that must be resolved during alignment.

A key concept for quantifying when reconstruction is information-theoretically feasible is the following notion of recoverability.

\begin{definition}[$k$-Recoverability]
\label{def:k_recoverability}
A family of graphs $\mathcal{G}$ is said to be $k$-recoverable with relative error $\varepsilon$ if, for every graph $G\in\mathcal{G}$ with adjacency matrix $\mathcal{A}$, knowledge of the first (smallest-eigenvalue) $k$ eigenvectors of the Laplacian $\mathcal{L}(\mathcal{A})$ (observed up to an arbitrary orthogonal transformation $\mathcal{Q}\in O(k)$, consistent with Definition~\ref{def:local_spectral_embedding}; this reduces to sign flips and column permutations when all $k$ retained eigenvalues are simple) suffices to construct an estimated adjacency matrix $\widehat{\mathcal{A}}$ (itself a valid adjacency matrix) such that
\[
\frac{\|\widehat{\mathcal{A}} - \mathcal{A}\|_F}{\|\mathcal{A}\|_F} \le \varepsilon.
\]
\end{definition}

This definition formalizes ``successful recovery'' via the relative Frobenius-norm error, accounting for the inherent non-uniqueness of eigenvectors and insisting that $\widehat{\mathcal{A}}$ be a valid adjacency matrix. As we show in Section~\ref{sec:problem:leakage}, $k$-recoverability is not merely a theoretical abstraction: LoGraB's parameter $k$ directly controls the number of retained eigenvectors, so varying $k$ in the benchmark directly varies the information-theoretic recoverability of the underlying graph.

\subsection{Notation and Key Quantities}
\label{sec:problem:notation}

For a local patch $P_v$ with $n_v = |P_v|$ nodes, let $\mathcal{L}_v$ be its normalized Laplacian with eigenpairs $\{(\lambda_r, u_r)\}_{r \ge 1}$ ordered nondecreasingly. We define the following quantities that recur throughout the paper:

\begin{itemize}[nosep]
\item \textit{Truncated heat kernel:}
\[
H_v^{(k)}(t) = \sum_{r=1}^k e^{-t\lambda_r} u_r u_r^\top, \quad
R_v^{(k)}(t) = H_v(t) - H_v^{(k)}(t) = \sum_{r>k} e^{-t\lambda_r} u_r u_r^\top.
\]
\item \textit{Spectral gap:} $\delta_v = \lambda_{k+1} - \lambda_k$, measuring the separation between retained and discarded eigenvalues.
\item \textit{Truncation proxy:} $\eta_v = \|R_v^{(k)}(t)\|_2 = e^{-t\lambda_{k+1}}$, bounding the entry-wise error from spectral truncation.
\item \textit{Gap-to-truncation ratio:} $\rho_v = \delta_v / (\delta_v + \eta_v) \in [0,1]$, which is 1 when the spectral gap dominates the truncation residual and 0 in the reverse case. This is a heuristic stability measure motivated by (but not derived from) the Davis--Kahan perturbation bound (Theorem~\ref{thm:spectral_perturbation}), where $\delta_v$ appears in the denominator. Previous drafts referred to this quantity as ``spectral SNR''; we have renamed it because it is not a standard signal-to-noise ratio in the signal-processing sense.
\item \textit{Composite fidelity score:} $s_v = \alpha \cdot \rho_v + (1-\alpha) \cdot \mathcal{E}_v$, where $\mathcal{E}_v$ is the normalized structural entropy (defined formally in Section~\ref{sec:afr:method}).
\end{itemize}

\paragraph{Symbol glossary.}
To disambiguate the several uses of Greek letters in this paper, we collect the most important in Table~\ref{tab:glossary}.

\begin{table}[!htbp]
\centering
\small
\caption{Symbol glossary. The fidelity-weighting $\alpha$ (Eq.\ for $s_v$) is distinct from the eigenvalue-decay exponent $\alpha$ appearing only in Proposition~\ref{thm:spectral_leakage}; context disambiguates.}
\label{tab:glossary}
\begin{tabular}{cp{11cm}}
\toprule
Symbol & Meaning \\
\midrule
$\alpha$ & Fidelity-score mixing weight (default 0.7); also eigenvalue-decay exponent in Corollary~\ref{cor:threshold}. \\
$\beta_n$ & Recovery failure probability in Proposition~\ref{thm:spectral_leakage}. \\
$\beta$ & RANSAC failure rate in Lemma~\ref{lem:ransac}. \\
$\gamma$ & Adaptive-threshold scaling parameter in AFR. \\
$\gamma_t$ & Heat-kernel entry-wise separation gap (Theorem~\ref{thm:edge_recovery}). \\
$\delta_v$ & Patch-level spectral gap $\lambda_{k+1}(\mathcal{L}_v) - \lambda_k(\mathcal{L}_v)$. \\
$\delta$ & Global spectral-gap lower bound (Assumption~\ref{assumption:spectral_leakage}). \\
$\delta$ (DP) & Failure probability in $(\varepsilon, \delta)$-differential privacy. \\
$\eta_v$ & Truncation residual $e^{-t\lambda_{k+1}}$. \\
$\kappa$ & Cross-voting sigmoid steepness. \\
$\rho_v$ & Gap-to-truncation ratio $\delta_v / (\delta_v + \eta_v)$. \\
$\rho_B$ & Boundary-edge ratio in Proposition~\ref{prop:cohesion} (Cohesion--Recall). \\
$\sigma$ & Gaussian-noise standard deviation applied to embeddings. \\
\bottomrule
\end{tabular}
\end{table}
\subsection{Why Spectral Embeddings Leak: Fundamental Results}
\label{sec:problem:leakage}

We now establish the theoretical foundation for spectral embedding leakage. These results demonstrate that spectral embeddings are simultaneously stable (retaining structural information under perturbation) and informative (enabling polynomial-time graph recovery when $k$ is sufficiently large). Together, they justify why spectral leakage is a meaningful threat and why LoGraB's parameterization is well-founded.

The following theorem establishes that local spectral embeddings are robust to perturbations, particularly when the spectral gap is large.

\begin{theorem}[Spectral perturbation bound]
\label{thm:spectral_perturbation}
Let $P_v$ be a graph patch and $\mathcal{L}_v$ its normalized Laplacian with eigenpairs $\{(\lambda_r, u_r)\}_{r\ge 1}$. Let $\mathcal{P}_v=[u_1,\ldots,u_k]\in\mathbb{R}^{|P_v|\times k}$ be the matrix of the first $k$ eigenvectors. Suppose an adversary observes a perturbed Laplacian $\widetilde{\mathcal{L}}_v = \mathcal{L}_v + E$, where $\|E\|_2 \le \varepsilon$. Let $\widetilde{\mathcal{P}}_v$ be the corresponding perturbed eigenspace. If the spectral gap $\delta_v = \lambda_{k+1}(\mathcal{L}_v) - \lambda_k(\mathcal{L}_v) > 0$, then there exists an orthogonal matrix $\mathcal{Q}\in O(k)$ such that:
\[
\|\widetilde{\mathcal{P}}_v - \mathcal{P}_v \mathcal{Q}\|_F \le \frac{\sqrt{2k}\,\varepsilon}{\delta_v}.
\]
\end{theorem}

This bound, proved in Appendix~\ref{proof:spectral}, shows that spectral embeddings are stable when $\delta_v$ is large relative to $\varepsilon$. The $1/\delta_v$ dependence explains why patches with large spectral gaps yield more reliable reconstructions, a relationship that AFR's fidelity score explicitly exploits.

\paragraph{The spectral embedding leakage proposition.}
The perturbation bound above establishes stability; the following result \emph{proposes} that this stability enables an attacker to recover the graph. We state it as a \emph{feasibility proposition} under specified assumptions; the supporting argument in Appendix~\ref{proof:spectral_leakage} is a mean-field variational Bayesian heuristic rather than a fully rigorous concentration-and-convergence analysis, and we highlight the remaining technical gaps explicitly. We emphasize up front that the algorithm underlying this proposition (variational Bayesian inference with a Sinkhorn-based likelihood) is \emph{not} the algorithm realized by AFR in Section~\ref{sec:afr}; the two constructions are complementary, share only an activation threshold, and use different recovery procedures (see Remark~\ref{rem:prop_vs_afr}).

\begin{assumption}
\label{assumption:spectral_leakage}
We consider an undirected graph $G=(V,E)$ with $n=|V|$ vertices, represented by its adjacency matrix $\mathcal{A}$:
\begin{enumerate}[label=(\roman*)]
\item The corresponding Laplacian matrix $\mathcal{L}\in\mathbb{R}^{n\times n}$ is symmetric with eigendecomposition $\mathcal{L}=\mathcal{V}\Lambda\mathcal{V}^\top$, where $\mathcal{V}=[v_1,\ldots,v_n]$ contains orthonormal eigenvectors and $\Lambda=\mathrm{diag}(\lambda_1,\ldots,\lambda_n)$.
\item There exists a positive constant $\delta>0$ such that the spectral gap satisfies $\delta_k := \lambda_{k+1}(\mathcal{L})-\lambda_k(\mathcal{L})\ge \delta$.
\item $G$ is sparse with maximum degree bounded polynomially: $d_{\max}=O(\mathrm{poly}(n))$.
\end{enumerate}
\end{assumption}

\begin{proposition}[Spectral Leakage Proposition, feasibility form]
\label{thm:spectral_leakage}
Suppose Assumption~\ref{assumption:spectral_leakage} holds for graph $G=(V,E)$ with Laplacian $\mathcal{L}$. For any target accuracy $\varepsilon \in (0,1)$, there is a threshold $k^*(\varepsilon)$ (upper-bounded in Corollary~\ref{cor:threshold}) and a polynomial-time variational Bayesian inference procedure with Sinkhorn-based likelihood, running in time $O(\mathrm{poly}(n,k))$, for which mean-field heuristics predict an estimate $\widehat{\mathcal{A}}$ satisfying
\[
\frac{\|\widehat{\mathcal{A}} - \mathcal{A}_0\|_F}{\|\mathcal{A}_0\|_F} \le \varepsilon
\]
with probability at least $1-\beta_n$, where $\beta_n \to 0$ as $n \to \infty$ (here $\beta_n$ denotes the recovery failure probability; it is distinct from the spectral-gap parameter $\delta$ of Assumption~\ref{assumption:spectral_leakage} and from the patch-level gap $\delta_v$). We provide a heuristic supporting argument (Appendix~\ref{proof:spectral_leakage}); a fully rigorous proof requires additional technical results (a quantitative posterior-concentration rate, mean-field VI consistency for the discrete graph prior, and a global convergence analysis of the stochastic ELBO optimization) that we identify as open problems. The proposition is intended as an \emph{information-theoretic feasibility statement}: it strongly suggests that spectral embeddings with $k \ge k^*$ eigenvectors are ``leaky'' and that polynomial-time $\varepsilon$-recovery is possible. Theorem~\ref{thm:edge_recovery} (proved rigorously in Appendix~\ref{proof:edge}) provides a stricter and fully proved \emph{local} recovery guarantee under a separation condition, which is what our empirical algorithm AFR exploits.
\end{proposition}

\begin{remark}[Proposition~\ref{thm:spectral_leakage} vs.\ AFR]
\label{rem:prop_vs_afr}
Proposition~\ref{thm:spectral_leakage} is a \emph{feasibility} result about a \emph{Bayesian} recovery procedure. AFR (Section~\ref{sec:afr}) is a \emph{constructive, deterministic} algorithm using heat-kernel thresholding, RANSAC-Procrustes, and Bundle Adjustment; it does \emph{not} run the variational procedure of the proposition. The two results share an activation threshold (spectral gap $\ge \delta$, enough eigenvectors) but use entirely different recovery procedures, and AFR's local guarantee (Theorem~\ref{thm:edge_recovery}) is strictly tighter when its separation condition holds.
\end{remark}

The supporting argument is sketched here; the full heuristic argument, together with an enumeration of the technical gaps, appears in Appendix~\ref{proof:spectral_leakage}. The informal intuition is that as $k$ increases, the posterior distribution $p(\mathcal{A}|\mathcal{P})$ concentrates more sharply around the true adjacency matrix $\mathcal{A}_0$, and a variational approximation is expected to inherit this concentration (though we do not prove this rigorously). The likelihood, defined via the entropic optimal transport distance between the observed embedding and the embedding induced by a candidate graph, is large when the candidate is close to $\mathcal{A}_0$ and small otherwise. Each variational step runs in $O(\mathrm{poly}(n,k))$ time per iteration, avoiding the $\Omega(k!\,2^k)$ brute-force complexity of exhaustive eigenvector alignment; establishing total polynomial runtime requires a convergence-rate analysis that we do not supply.

\begin{corollary}[Truncation-based proxy threshold]
\label{cor:threshold}
Let $\ell^*(\varepsilon) := \min\{k : \lambda_{k+1} \le \varepsilon\}$ be the smallest truncation level at which the $(k+1)$-th Laplacian eigenvalue drops below $\varepsilon$. For common eigenvalue decay profiles:
\begin{itemize}[nosep]
\item \textit{Polynomial decay} $\lambda_k = Ck^{-\alpha}$: $\ell^*(\varepsilon) = \lceil(C/\varepsilon)^{1/\alpha}\rceil - 1$.
\item \textit{Exponential decay} $\lambda_k = Ce^{-\alpha k}$: $\ell^*(\varepsilon) = \lceil\alpha^{-1}\ln(C/\varepsilon)\rceil - 1$.
\end{itemize}
$\ell^*(\varepsilon)$ is a \emph{proxy lower bound} on the recoverability threshold $k^*(\varepsilon)$: below $\ell^*$, the truncation residual $\eta_v = e^{-t\lambda_{k+1}}$ exceeds $\varepsilon$, so no eigenvector-based procedure can reliably distinguish adjacency matrices whose eigenvector differences lie within the truncation bucket. The exact $k^*(\varepsilon)$ that yields $\|\widehat{\mathcal{A}} - \mathcal{A}_0\|_F / \|\mathcal{A}_0\|_F \le \varepsilon$ depends additionally on the eigenvector-to-adjacency Lipschitz constant of the recovery procedure. Establishing the tight relation $k^* = \Theta(\ell^*)$ for the variational procedure of Proposition~\ref{thm:spectral_leakage} requires an adjacency-matrix perturbation bound that we leave to future work.
\end{corollary}

\begin{corollary}[Noise as a defense, qualitative]
\label{cor:dp_defense}
If the shared embedding is perturbed by noise (e.g., through differential privacy), yielding $\widehat{\mathcal{P}} = \mathcal{P} + \mathcal{N}$, then exact recovery of $\mathcal{L}$ becomes impossible. Qualitatively, the noise introduces an unavoidable reconstruction error that grows with $\|\mathcal{N}\|$: the posterior $p(\mathcal{A}|\widehat{\mathcal{P}})$ becomes less concentrated around $\mathcal{A}_0$, and the variational parameters $\phi_{ij}^*$ drift toward $0.5$ (maximum uncertainty). This qualitatively reflects the privacy--accuracy trade-off: stronger noise yields better privacy but degrades both the attacker's reconstruction and the legitimate utility of the embedding. A quantitative trade-off curve is provided empirically in Section~\ref{sec:exp:recon}.
\end{corollary}

\begin{remark}[Connection to LoGraB]
\label{rem:lograb_connection}
These results connect directly to LoGraB's parameterization. The benchmark's parameter $k$ (number of retained eigenvectors) corresponds precisely to the information-theoretic quantity governing recoverability. Varying $k$ in LoGraB thus directly varies the information-theoretic recoverability: at small $k$, even a perfect algorithm cannot recover the graph (below the threshold $k^*$), while at large $k$, recovery becomes fundamentally feasible. Similarly, LoGraB's noise parameter $\sigma$ operationalizes Corollary~\ref{cor:dp_defense}: increasing $\sigma$ moves the system toward the defense regime where reconstruction accuracy degrades. The benchmark therefore provides a controlled experimental environment for exploring the entire landscape predicted by the theory.
\end{remark}

\subsection{From Global to Local Leakage}
\label{sec:problem:local}

The Spectral Leakage Proposition establishes that global spectral embeddings leak. We now show that this leakage persists even when the adversary observes only \emph{fragmented} local embeddings, the setting considered by AFR (Section~\ref{sec:afr}). We state the result as a constructive recoverability guarantee that composes three AFR-local guarantees, and then explain how it relates back to Proposition~\ref{thm:spectral_leakage}.

\begin{proposition}[Local-to-global recoverability]
\label{prop:local_reduction}
Let $\{(\mathcal{P}_v, \lambda^{(v)}_{k+1})\}_{v \in V_{\mathrm{obs}}}$ be a collection of local spectral embeddings whose patch overlaps satisfy Assumption~\ref{assumption:afr_overlap} (in Appendix~\ref{appendix:assumptions}), and assume that each local patch Laplacian $\mathcal{L}_v$ individually satisfies the spectral-gap condition of Proposition~\ref{thm:spectral_leakage} with threshold $k \ge k^*_v(\varepsilon)$. Then there exists a polynomial-time procedure (instantiated by AFR in Section~\ref{sec:afr}) that:
\begin{enumerate}[label=(\roman*)]
\item recovers each local adjacency $\widehat{\mathcal{A}}_v$ with bounded error controlled by the per-patch spectral gap $\delta_v$ and truncation proxy $\eta_v$ (Theorem~\ref{thm:edge_recovery}, a constructive local instantiation of Proposition~\ref{thm:spectral_leakage});
\item aligns overlapping patches with alignment error $O(\sqrt{m}\,\eta/\delta_v)$ for $m$-node overlaps (Lemma~\ref{lem:patch_alignment});
\item accumulates error at most linearly with graph diameter (Lemma~\ref{lem:error_accumulation}).
\end{enumerate}
The resulting global adjacency $\widehat{\mathcal{A}}$ satisfies $\|\widehat{\mathcal{A}}-\mathcal{A}\|_F / \|\mathcal{A}\|_F = O(\mathrm{diam}(G) \cdot \varepsilon)$. Hence, under patchwise spectral-gap hypotheses that mirror the global assumptions of Proposition~\ref{thm:spectral_leakage} at the patch level, fragmented spectral sharing is no more private than global sharing up to a polynomial overhead.
\end{proposition}

This proposition formalizes why LoGraB's fragmented setting and AFR's patch-based attack are connected to the global leakage result: under the overlap and spectral-gap conditions above, fragmentation alone need not preclude recovery, but instead shifts the challenge to patchwise reconstruction and alignment. The detailed proof, which composes Theorem~\ref{thm:edge_recovery}, Lemma~\ref{lem:patch_alignment}, and Lemma~\ref{lem:error_accumulation}, is given in Appendix~\ref{proof:local_reduction}. We emphasize an important distinction: the proposition does \emph{not} claim that AFR runs the variational Bayesian/Sinkhorn procedure underlying Proposition~\ref{thm:spectral_leakage}. Rather, it shows that under \emph{the same spectral-gap regime} ($k \ge k^*_v(\varepsilon)$) where Proposition~\ref{thm:spectral_leakage} guarantees recoverability exists, AFR's heat-kernel thresholding (Theorem~\ref{thm:edge_recovery}) provides a sharper deterministic alternative. The two results share an activation threshold but use different recovery procedures, with Theorem~\ref{thm:edge_recovery} providing a stricter exact-recovery guarantee when its separation condition holds.

\section{LoGraB: Measuring the Threat}
\label{sec:lograb}

Having defined the fragmented-graph setting, we now present LoGraB, a benchmark that operationalizes this setting into a controlled, reproducible testbed. LoGraB systematically decomposes a global graph $G=(V,E)$ into a collection of overlapping patches $P=\{P_1,\ldots,P_m\}$ to simulate real-world data constraints. Each patch is represented not as a perfect subgraph, but only by a truncated and potentially corrupted spectral embedding.

\subsection{Graph Decomposition Strategies}
\label{sec:lograb:decomp}

LoGraB provides three lenses for simulating graph fragmentation, each modeling a different real-world scenario.

\paragraph{Node-centric $d$-hop patches.}
This strategy models a user-centric worldview. For each node $v\in V$, we generate its local view, a patch containing all nodes within a $d$-hop shortest-path distance:
\[
P_v=\{u\in V \mid \text{dist}_G(u,v)\le d\},
\]
where $\text{dist}_G$ is the shortest-path metric. With full coverage ($p=1$), this method guarantees that every node is observed. It produces densely overlapping patches that model privacy settings in which each user reveals information only about their immediate bounded network neighborhood.

\paragraph{Cluster-based partitioning with overlap.}
This strategy models natural data silos. We first partition the graph $G$ into $\ell$ disjoint clusters $\{C_1,\ldots,C_\ell\}$ using multilevel spectral bisection (METIS). Then, to create realistic interaction, we expand each cluster to include its one-hop neighbors across cluster boundaries. The resulting patches $P_i = C_i \cup B_{C_i}$, where $B_{C_i}$ contains 1-hop neighbors of $C_i$ in other clusters, preserve strong internal connectivity while having controlled overlap, reflecting how data is naturally partitioned by community structure or organizational domains.

\paragraph{Random overlapping patches.}
This strategy captures the chaos of real-world data collection. We randomly sample $s$ seed nodes and form patches $P_v$ from their $d$-hop neighborhoods. These patches overlap irregularly and vary widely in size. This model does not guarantee full coverage, creating an irregular and heterogeneous cover reminiscent of crowd-sourced or opportunistic sensing networks.

\subsection{Parameterized Generation}
\label{sec:lograb:params}

To simulate the diverse challenges of the real world, we equip researchers with four parameters to precisely control the nature of fragmentation. These parameters allow for a spectrum of scenarios, from near-perfect networks to extremely sparse information environments.

\begin{table}[!htbp]
    \centering
    \caption{Parameter grid for LoGraB instance generation.}
    \begin{tabular}{clc}
        \toprule
        Symbol & Description & Grid \\
        \midrule
        $d$ & Radius controlling patch locality & $\{1, 2\}$ \\
        $k$ & Number of eigenvectors retained (spectral fidelity) & $\{16, 32, 64\}$ \\
        $\sigma$ & Standard deviation of additive Gaussian noise & $\{0, 0.05, 0.1\}$ \\
        $p$ & Coverage ratio: fraction of vertices for which patches are generated & $\{0.6, 0.8, 1.0\}$ \\
        \bottomrule
    \end{tabular}
    \label{tab:param_grid}
\end{table}

Our generation pipeline applies these constraints in a deliberate sequence:

\begin{itemize}[nosep]
\item \textit{Simulating partial observability (parameter $p$).} To mimic scenarios such as client dropout or inaccessible data, we first sub-sample a set of observed nodes $V_{\text{obs}} \subseteq V$. When $p < 1$, only a fraction $p \cdot |V|$ of nodes is selected, and subsequent decomposition and spectral extraction steps are performed only for the nodes in $V_{\text{obs}}$.

\item \textit{Simulating spectral leakage and noise (parameters $k$ and $\sigma$).} For each patch generated from $V_{\text{obs}}$, we compute the Laplacian $\mathcal{L}_v$ and retain only the bottom $k$ eigenvectors to create a truncated spectral embedding $\mathcal{P}_v = [u_1,\ldots,u_k] \in \mathbb{R}^{|P_v|\times k}$. To further model sensor uncertainty or privacy-preserving noise, we corrupt each entry $u$ in $\mathcal{P}_v$ with Gaussian noise: $u \leftarrow u + \varepsilon$, where $\varepsilon \sim \mathcal{N}(0, \sigma^2)$.
\end{itemize}

\FloatBarrier
\subsection{Data Packaging and Reproducibility}
\label{sec:lograb:data}

To ensure complete reproducibility, we formalize our generation process in Algorithm~\ref{alg:lograb}. Each generated benchmark instance is encapsulated in a single archive file containing processed spectral embeddings and a comprehensive metadata file recording every parameter choice and the initial seed. We provide SHA-256 checksums for all archives to guarantee data integrity.

\begin{algorithm2e}[t]
\small
\DontPrintSemicolon
  \KwInput{Graph $G$, strategy $S$, parameters $(d, k, \sigma, p)$, seed}
  \KwOutput{Archive $\mathscr{I}$ with processed patches.}
    Set RNG $\leftarrow$ seed\\
    $V_{\text{obs}}$ $\leftarrow$ $\text{select\_nodes}(V,p)$\\
    $P_{\text{obs}}$ $\leftarrow$ $\text{decompose}(G, S, d, V_{\text{obs}})$\\
    \ForEach{patch $\mathcal{Q} \in P_{\text{obs}}$}{
        $\mathcal{L}_\mathcal{Q}\leftarrow\text{compute\_Laplacian}(\mathcal{Q})$\\
        $(\mathcal{P}_\mathcal{Q}, \lambda_{k_{\mathcal{Q}}+1}) \gets \text{spectral\_truncate}(\mathcal{L}_{\mathcal{Q}}, k)$\\
        $\mathcal{P}^\prime_\mathcal{Q} \gets \text{apply\_noise}(\mathcal{P}_\mathcal{Q}, \sigma)$\\
        $\text{store\_patch}(\mathcal{Q}, \mathcal{P}^\prime_\mathcal{Q}, \lambda_{k+1})$
    }
    $\text{write\_archive}(V_{\text{obs}}, \text{metadata}=\ldots)$
\caption{LoGraB instance generation.}
\label{alg:lograb}
\end{algorithm2e}

\subsection{Benchmark Tasks and Metrics}
\label{sec:lograb:tasks}

We designed LoGraB to answer three fundamental questions about algorithmic resilience: How well can an algorithm recover the global topology from local fragments? Can it learn useful representations from a restricted local view? And can it reason about relationships that span across fragments it never sees directly? Each question is formalized as a distinct task.

\subsubsection{Task 1: Graph Reconstruction}
\label{sec:lograb:tasks:recon}

Given only the collection of noisy truncated spectral patches $P$, an algorithm must reassemble the global picture, producing a reconstructed graph $\widehat{G}=(\widehat{V},\widehat{E})$, where $\widehat{V}\subseteq V$. We measure success with a multi-faceted suite of metrics:

\begin{itemize}[nosep]
    \item \textit{Node coverage}: How much of the original graph was recovered, $|\widehat{V}|/|V|$.
    \item \textit{Edge fidelity} (Precision, Recall, F1): Of the nodes found, how accurately were their connections rebuilt.
    \item \textit{Island Cohesion}: Do the reassembled components have internal structural integrity, or are they random assortments of correct edges?
\end{itemize}

\paragraph{Island Cohesion.}
Standard precision and recall can be misleading: an algorithm might score well by predicting a loose collection of correct edges without capturing the underlying community structure. We introduce Island Cohesion to address this. Let $\mathcal{C}=\{C_1,\ldots,C_r\}$ be the connected components of $\widehat{G}$. For each component $C_i$, define its ground-truth internal edge set as $E_i^*=\{(u,v)\in E \mid u,v\in C_i\}$ and its predicted internal edge set as $\widehat{E}_i=\widehat{E}\cap(C_i\times C_i)$. Island Cohesion is the weighted average recall across all islands:
\[
    \text{Cohesion}=\sum_{i=1}^r \frac{|E_i^*|}{\sum_j |E_j^*|}\cdot\frac{|\widehat{E}_i\cap E_i^*|}{|E_i^*|}.
\]
Islands with $|E_i^*|=0$ (single-node components or components without true internal edges) are excluded from both the sum and the normalizing denominator. A cohesion score of 1.0 signifies perfect reconstruction within each recovered component, indicating that the algorithm has rediscovered meaningful parts of the graph's topology.

The following proposition, proved in the Appendix~\ref{proof:cohesion}, is an algebraic identity showing that Island Cohesion and standard Recall are tightly coupled when boundary edges are few. (We state it as a proposition rather than a theorem because it is a definitional rearrangement, not a deep result.)

\begin{proposition}[Cohesion--Recall relationship]
\label{prop:cohesion}
Let $G=(V,E)$ be a ground-truth graph, and $\widehat{G}=\bigcup_j \widehat{G}_j$ the reconstructed graph with disjoint islands as connected components. Let $E^*(\widehat{V})$ denote the true edges induced in $\widehat{V}$, and $B\subseteq E^*(\widehat{V})$ the boundary edges whose endpoints belong to different islands. Define the \emph{boundary ratio} $\rho_B := |B|/|E^*(\widehat{V})|$. Then:
\[
(1-\rho_B)\,\mathrm{Cohesion}(\widehat{G}) \le \mathrm{Recall}(\widehat{G}) \le \rho_B + (1-\rho_B)\,\mathrm{Cohesion}(\widehat{G}).
\]
\end{proposition}

When the boundary ratio $\rho_B$ is small, as is the case when the reconstructed islands align well with the graph's community structure, Cohesion and Recall are approximately equal. This validates Island Cohesion as a meaningful complement to standard edge-level metrics. (We use $\rho_B$ rather than $\beta$ to avoid collision with the RANSAC failure rate.)

\subsubsection{Task 2: Localized Node Classification}
\label{sec:lograb:tasks:nodecls}

This task tests the ability to learn useful representations while confined to an isolated local view, a core capability for federated learning applications. An algorithm is given a set of disconnected graph patches $P=(P_1,\ldots,P_N)$. Each patch $P_v=(V_v,E_v,\mathcal{P}_v)$ is a self-contained world containing the features of the original node augmented by LoGraB's spectral coordinates. The goal is to train a single model $f_\theta$ that predicts the label $y_u$ of a node using only its local context within its patch $P_v$. Information sharing between patches is strictly prohibited, forcing the model to learn generalizable representations.

\paragraph{Evaluation.}
We adopt a strict protocol inspired by cross-client federated learning. The entire set of patches, not nodes, is divided into training (60\%), validation (20\%), and test (20\%) sets. This patch-level split guarantees that the test set consists of ``clients'' the model has never seen during training, completely eliminating information leakage from shared vertices. We report F1-Micro (overall accuracy) as the primary metric.

\subsubsection{Task 3: Inter-Fragment Link Prediction}
\label{sec:lograb:tasks:linkpred}

Given a set of disjoint node components $\widehat{\mathcal{C}}=\{C_1,\ldots,C_r\}$, the task is to learn a scoring function $\psi(u,v)\in\mathbb{R}$ for any pair of nodes $(u,v)$ belonging to different components $u\in C_i$, $v\in C_j$, $i\ne j$. A higher score must indicate a higher likelihood that an edge exists in the original graph. We construct a balanced evaluation set: the positive set $\varepsilon^+$ contains all true cross-component edges, and the negative set $\varepsilon^-$ of equal size is sampled from non-existent inter-component edges. Performance is measured by AUROC. Figure~\ref{fig:fragmented_graph} illustrates this scenario, where four structurally distinct subgraphs are placed at corners with known intra-subgraph connectivity but missing inter-subgraph links (dashed edges).

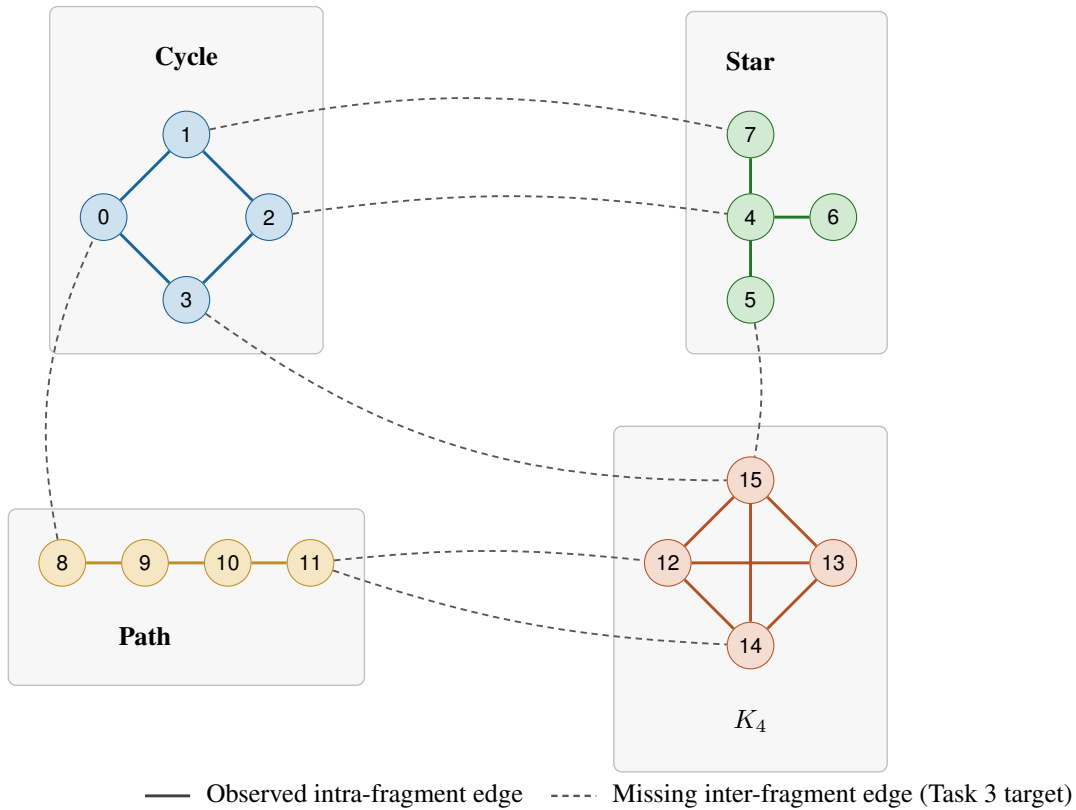
\begin{figure}[!ht]
\centering
\begin{tikzpicture}[
    >={Stealth[length=4pt,width=3pt]},
    every node/.style={font=\small},
    vtx/.style={
        circle, draw=black!70, line width=0.4pt,
        minimum size=6.2mm, inner sep=0pt,
        font=\scriptsize\sffamily
    },
    vc/.style={vtx, fill=cCycle!22, draw=cCycle!80!black},
    vs/.style={vtx, fill=cStar!22,  draw=cStar!70!black},
    vp/.style={vtx, fill=cPath!28,  draw=cPath!85!black},
    vk/.style={vtx, fill=cKfour!22, draw=cKfour!85!black},
    eC/.style={cCycle!85!black, line width=1.1pt},
    eS/.style={cStar!75!black,  line width=1.1pt},
    eP/.style={cPath!85!black,  line width=1.1pt},
    eK/.style={cKfour!85!black, line width=1.1pt},
    miss/.style={cMiss, dashed, dash pattern=on 2.4pt off 1.8pt,
                 line width=0.7pt},
    frag/.style={rounded corners=3pt, draw=cFrame, line width=0.5pt,
                 fill=black!3, inner sep=4mm},
    label/.style={font=\small\bfseries}
]
 
\begin{scope}[shift={(0,0)}]
    \node[vc] (n1) at (0   , 1.1) {1};
    \node[vc] (n0) at (-1.1, 0  ) {0};
    \node[vc] (n2) at ( 1.1, 0  ) {2};
    \node[vc] (n3) at (0   ,-1.1) {3};
    \draw[eC] (n0)--(n1)--(n2)--(n3)--(n0);
    \node[label, above=4mm of n1] (Lcycle) {Cycle};
\end{scope}
\begin{scope}[on background layer]
    \node[frag, fit=(n0)(n1)(n2)(n3)(Lcycle)] (Fcycle) {};
\end{scope}
 
\begin{scope}[shift={(7.5,0)}]
    \node[vs] (n4) at (0   , 0  ) {4};
    \node[vs] (n7) at (0   , 1.1) {7};
    \node[vs] (n6) at ( 1.1, 0  ) {6};
    \node[vs] (n5) at (0   ,-1.1) {5};
    \draw[eS] (n4)--(n7);
    \draw[eS] (n4)--(n6);
    \draw[eS] (n4)--(n5);
    \node[label, above=4mm of n7] (Lstar) {Star};
\end{scope}
\begin{scope}[on background layer]
    \node[frag, fit=(n4)(n5)(n6)(n7)(Lstar)] (Fstar) {};
\end{scope}
 
\begin{scope}[shift={(0,-4.6)}]
    \node[vp] (n8)  at (-1.65, 0) {8};
    \node[vp] (n9)  at (-0.55, 0) {9};
    \node[vp] (n10) at ( 0.55, 0) {10};
    \node[vp] (n11) at ( 1.65, 0) {11};
    \draw[eP] (n8)--(n9)--(n10)--(n11);
    \node[label, below=4mm of n9.south] (Lpath) {Path};
\end{scope}
\begin{scope}[on background layer]
    \node[frag, fit=(n8)(n9)(n10)(n11)(Lpath)] (Fpath) {};
\end{scope}
 
\begin{scope}[shift={(7.5,-4.6)}]
    \node[vk] (n15) at (0   , 1.1) {15};
    \node[vk] (n12) at (-1.1, 0  ) {12};
    \node[vk] (n13) at ( 1.1, 0  ) {13};
    \node[vk] (n14) at (0   ,-1.1) {14};
    \draw[eK] (n12)--(n13);
    \draw[eK] (n12)--(n14);
    \draw[eK] (n12)--(n15);
    \draw[eK] (n13)--(n14);
    \draw[eK] (n13)--(n15);
    \draw[eK] (n14)--(n15);
    \node[label, below=4mm of n14] (Lk4) {$K_4$};
\end{scope}
\begin{scope}[on background layer]
    \node[frag, fit=(n12)(n13)(n14)(n15)(Lk4)] (Fk4) {};
\end{scope}
 
\draw[miss] (n1) to[bend left=12]  (n7);
\draw[miss] (n2) to[bend left=8]   (n4);
\draw[miss] (n3) to[bend right=18] (n15);
\draw[miss] (n5) to[bend left=10]  (n15);
\draw[miss] (n11) to[bend left=6]  (n12);
\draw[miss] (n11) to[bend right=8] (n14);
\draw[miss] (n8) to[bend left=18]  (n0);
 
\begin{scope}[shift={(3.75,-7.7)}]
    \draw[black!75, line width=1.1pt] (-4.30,0) -- (-3.70,0);
    \node[anchor=west, font=\small] at (-3.62,0)
         {Observed intra-fragment edge};
 
    \draw[miss, line width=0.9pt] ( 1.10,0) -- ( 1.70,0);
    \node[anchor=west, font=\small] at ( 1.78,0)
         {Missing inter-fragment edge (Task~3 target)};
\end{scope}
 
\end{tikzpicture}

\caption{Fragmented graph scenario in LoGraB. Each of the four coloured
panels is a local fragment observed by a different client: a cycle, a
star, a path, and a complete graph~$K_4$. Solid edges are observed
intra-fragment links; dashed curves are the unobserved inter-fragment
links that Task~3 (inter-fragment link prediction) must recover.}
\label{fig:fragmented_graph}
\end{figure}

\paragraph{Baseline approaches.}
We employ two types of baselines: (i) \textit{representation-based} methods that compute similarity scores (e.g., cosine similarity) from node embeddings learned in Task~2; and (ii) \textit{structure-based} methods, including an adaptation of SEAL \citep{Zhang2018SEAL} that extracts enclosing subgraphs from a co-occurrence graph, Neo-GNN, and LightGCN \citep{He2020LightGCN}.

\section{AFR: Exploiting the Threat Under Imperfection}
\label{sec:afr}

LoGraB's Task~1 (graph reconstruction) is the most challenging benchmark task: it requires inverting noisy, truncated spectral embeddings to recover global topology. We now present AFR, an algorithm designed for this task. Prior patch-based reconstruction approaches typically assume patches of uniform quality and therefore do not actively exploit the substantial quality heterogeneity present under non-uniform truncation and noise; AFR actively measures per-patch quality and adapts its assembly strategy to prioritize the most reliable patches and require more evidence before committing to alignments of lower-fidelity pairs.

Our theoretical framework is built upon formal assumptions that guide the algorithm's behavior under realistic, imperfect conditions. Instead of demanding data perfection, we assume that a subset of the local data is sufficiently reliable to initiate a robust reconstruction and that this reliability can be quantified. The complete mathematical formulation of these assumptions is detailed in Appendix~\ref{appendix:assumptions}; here, we summarize the key ideas. First, a \emph{core patch} is a local subgraph that is computationally tractable, has a bounded reconstruction error, and surpasses a minimum fidelity threshold. Second, an \emph{adaptive eligibility criterion} for stitching requires that the overlap between patches be sufficiently large and structurally sound, with the required size increasing as fidelity decreases. Third, RANSAC-Procrustes provides probabilistic guarantees for recovering near-correct rotations.

\begin{figure*}[h!]
    \centering
    \includegraphics[width=\linewidth]{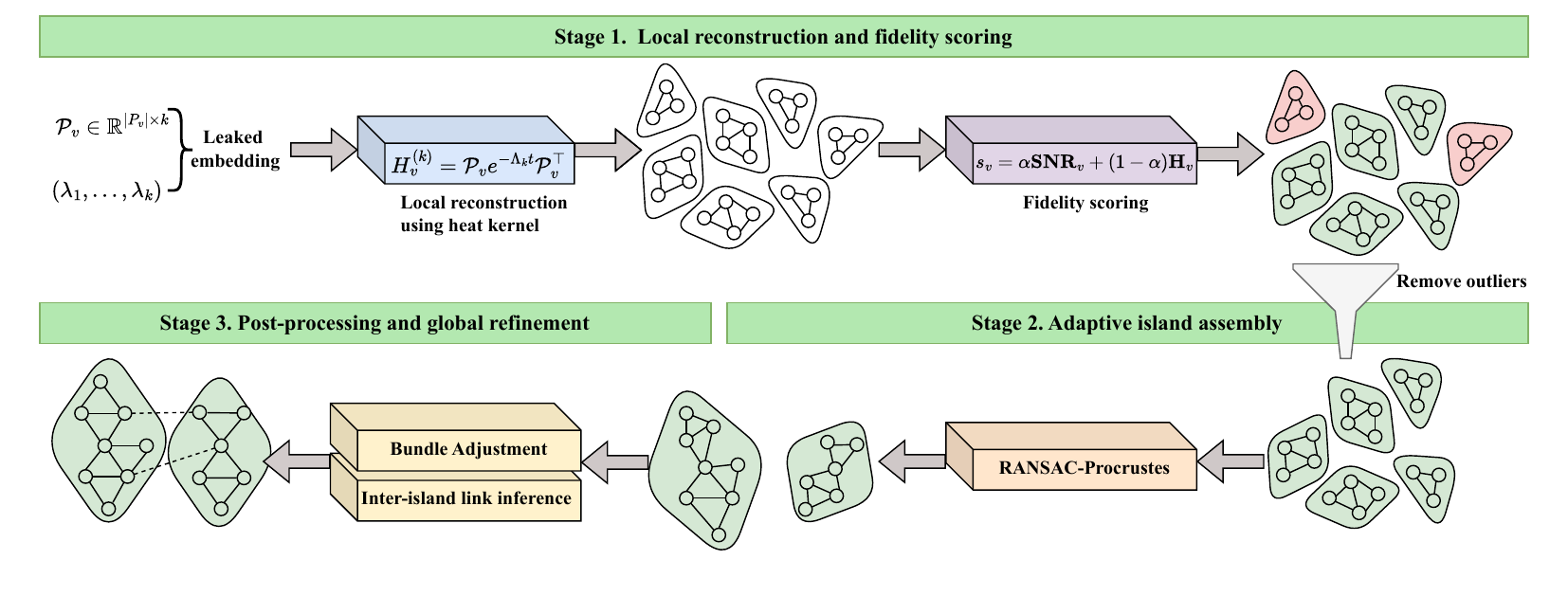}
    \caption{Adaptive Fidelity-driven Reconstruction (AFR). The pipeline reconstructs graphs in three stages: local reconstruction with fidelity scoring, adaptive island assembly using RANSAC-Procrustes, and global refinement via Bundle Adjustment and inter-island link inference.}
    \label{fig:afr_pipeline}
\end{figure*}

\subsection{The Three-Stage Pipeline}
\label{sec:afr:method}

\begin{algorithm2e}[h!]
\setstretch{1}
\DontPrintSemicolon
\KwInput{A set of $m$ local patches $\{(P_i, \mathcal{P}_i, \Lambda_i)\}_{i=1}^m$, where $\mathcal{P}_i$ is the spectral embedding and $\Lambda_i$ contains the eigenvalues for patch $P_i$.}
\KwOutput{A set of reconstructed islands $\{\hat{G}_j\}$ and a set of probabilistic inter-island edges $\hat{E}_{\text{cross}}$.}
\vspace{0.2cm}
\tcc{Stage 1: Local reconstruction and fidelity scoring}
$\text{core\_patches}\leftarrow \{\}$\;
\ForEach{patch $i\leftarrow 1$ \KwTo $m$}{
    $H_i^{(k)} \leftarrow \mathcal{P}_i \, e^{-\Lambda_{i,k}t} \, \mathcal{P}_i^{\top}$\;
    $\hat{\mathcal{A}}_i \leftarrow \text{Threshold}(H_i^{(k)})$\;
    $\delta_i \leftarrow \lambda_{k+1} - \lambda_k$; \quad
    $\eta_i \leftarrow e^{-t \lambda_{k+1}}$\;
    $\rho_i \leftarrow \delta_i / (\delta_i + \eta_i)$; \quad
    $\mathcal{E}_i \leftarrow \text{StructuralEntropy}(\hat{\mathcal{A}}_i)$ \;
    $s_i \leftarrow \alpha \cdot \rho_i + (1-\alpha) \cdot \mathcal{E}_i$\;
    \If{$s_i \ge s_{\min}$ and $\delta_i \ge \delta_{\min}$}{
        Add $(i, \hat{\mathcal{A}}_i, s_i)$ to core\_patches
    }
}
\tcc{Stage 2: Adaptive island assembly}
islands $\leftarrow$ Initialize, one per core patch\;
$Q \leftarrow$ Priority queue of all potential stitching pairs $(v, w)$, prioritized by $f(s_v, s_w, |P_v \cap P_w|)$
\While{$Q$ is not empty}{
    $(v, w) \leftarrow Q.\text{pop}()$\;
    $d_{\text{adaptive}} \leftarrow \max\{k+1,\, k_{\text{base}} + \gamma(1 - \min\{s_v, s_w\})\}$ \tcp*{Procrustes well-posedness floor}\;
    \If{$|P_v \cap P_w| \ge d_{\text{adaptive}}$ and $G[P_v \cap P_w]$ is connected}{
        $\mathcal{Q}_{vw}, \text{consensus} \leftarrow \text{RANSAC-Procrustes}(\mathcal{P}_v, \mathcal{P}_w, P_v \cap P_w)$\;
        \If{$|\text{consensus}| \ge d_{\text{adaptive}}$}{
            Merge islands containing $v$ and $w$ using rotation $\mathcal{Q}_{vw}$\;
            Update $Q$ with new potential stitches from the merged island
        }
    }
}
\tcc{Stage 3: Post-processing and global refinement}
\ForEach{assembled island $\hat{G}_j \in$ islands}{
    $\{\mathcal{Q}_{vw}^{*}\} \leftarrow \arg\min_{\{\mathcal{Q}_{vw}\}} \sum_{(v,w)} \| \mathcal{P}_v' - \mathcal{P}_w' \mathcal{Q}_{vw} \|_F^2$\;
    Refine node positions in $\hat{G}_j$ using $\{\mathcal{Q}_{vw}^{*}\}$
}
$\hat{E}_{\text{cross}} \leftarrow \{\}$\;
\ForEach{inter-island node pair $(u, w)$}{
    $C(u, w) \leftarrow$ Count co-occurrences in original patch intersections\;
    \If{$C(u, w) > C_0$}{
        $P_{\text{inter}}(u,w) \leftarrow (1 + e^{-\kappa(C(u,w) - C_0)})^{-1}$\;
        Add probabilistic edge $(u, w)$ with probability $P_{\text{inter}}$ to $\hat{E}_{\text{cross}}$
    }
}
\Return islands, $\hat{E}_{\text{cross}}$
\caption{Adaptive Fidelity-driven Reconstruction (AFR).}
\label{alg:afr}
\end{algorithm2e}

Let $G=(V,E)$ be an undirected graph with $|V|=n$ nodes. For every $v\in V$ we observe a local $d$-hop embedding $\mathcal{P}_v\in\mathbb{R}^{|P_v|\times k}$ of the patch $P_v\subseteq V$. Our goal is to reconstruct a set of high-fidelity \emph{islands} $\{\widehat{G}_1,\ldots,\widehat{G}_r\}\subseteq G$ together with a probabilistic inter-island edge set $\widehat{E}_{\mathrm{cross}}$.

\paragraph{Stage 1: Local reconstruction and fidelity scoring.}
The first stage acts as a gatekeeper, transforming each raw spectral embedding into a structured representation and quantifying its trustworthiness.

\textit{Local reconstruction.} For each patch $v$, we reconstruct a local adjacency matrix $\widehat{\mathcal{A}}_v$ from its spectral embedding $\mathcal{P}_v\in\mathbb{R}^{|P_v|\times k}$ and corresponding eigenvalues $(\lambda_1,\ldots,\lambda_k)$ using the heat kernel method:
\[
H_v^{(k)}=\mathcal{P}_v \, e^{-\Lambda_k t} \, \mathcal{P}_v^\top,
\]
where $t$ is a fixed time parameter and $\Lambda_k=\text{diag}(\lambda_1,\ldots,\lambda_k)$. The local adjacency matrix $\widehat{\mathcal{A}}_v$ is obtained by thresholding the entries of $H_v^{(k)}$.

\textit{Fidelity scoring.} The core innovation of this stage is a multi-component fidelity score $s_v$ that answers two orthogonal questions about each patch's quality:

\begin{enumerate}[nosep]
\item \textit{Is the signal stable?} We quantify this with the gap-to-truncation ratio (Section~\ref{sec:problem:notation}):
\[
\rho_v = \frac{\delta_v}{\delta_v + \eta_v} \in [0,1],
\]
where the stability term $\delta_v = \lambda_{k+1}(\mathcal{L}_v) - \lambda_k(\mathcal{L}_v)$ is the spectral gap and the truncation residual $\eta_v = e^{-t\lambda_{k+1}(\mathcal{L}_v)}$ bounds the heat-kernel error. This is a heuristic measure motivated by Theorem~\ref{thm:spectral_perturbation}, which shows that eigenvector perturbation error scales with $1/\delta_v$; it is not a standard signal-to-noise ratio in the signal-processing sense.

\item \textit{Is the signal distinctive?} We quantify this with structural entropy:
\[
\mathcal{E}_v = -\frac{1}{\log(|V_v|)}\sum_{d \in D_v} p(d)\,\log(p(d)),
\]
where $D_v$ is the set of unique node degrees in the patch and $p(d)$ is the empirical probability of degree $d$. For single-node patches ($|V_v|=1$), we define $\mathcal{E}_v := 0$. Intuitively, patches with near-uniform degrees (low entropy) have highly symmetric point configurations, producing near-degenerate Procrustes problems where the singular values $\sigma_{\min}(A_0^\top B_0)$ in Theorem~\ref{thm:stitch}'s denominator become small, amplifying alignment error. Entropy thus serves as a cheap proxy for Procrustes conditioning that can be computed directly from the reconstructed $\widehat{\mathcal{A}}_v$.
\end{enumerate}

These two components are combined via a convex combination:
\[
s_v = \alpha \, \rho_v + (1-\alpha) \, \mathcal{E}_v,
\]
where $\alpha\in[0,1]$ balances spectral stability and structural distinctiveness. Only patches surpassing minimum thresholds for both fidelity score and spectral gap are promoted to ``core patches.''

\paragraph{Stage 2: Adaptive island assembly.}
This stage is the heart of AFR, where core patches are intelligently assembled into larger, internally consistent structural islands.

\textit{Data-driven prioritization.} Assembly follows a prioritized scheme. The set of core patches $\{(\widehat{\mathcal{A}}_v, s_v)\}$ satisfying Assumption~\ref{assumption:1} forms the input. A priority queue determines the next best potential stitch based on the joint fidelity $f(s_v, s_w)$ and the overlap size $|I_{vw}|$.

\textit{Adaptive stitching loop.} The algorithm iteratively attempts to stitch the highest-priority pair $(v,w)$:

\begin{enumerate}[nosep]
\item \textit{Eligibility check:} The pair must satisfy the adaptive criterion, their overlap must be large enough ($|I_{vw}| \ge d_{\text{adaptive}} = \max\{k+1,\, k_{\text{base}} + \gamma(1 - \min\{s_v,s_w\})\}$, where the floor $k+1$ ensures the Procrustes problem is well-posed in $\mathbb{R}^k$, as required by Lemma~\ref{lem:patch_alignment}) and the induced subgraph must be connected.
\item \textit{Robust alignment:} If eligible, RANSAC-Procrustes aligns the pair, finding a correct rotation even in the presence of faulty correspondences.
\item \textit{Stitching decision:} A stitch is accepted only if RANSAC finds a strong geometric consensus ($|\text{consensus\_set}| \ge d_{\text{adaptive}}(s_v,s_w)$).
\end{enumerate}

\paragraph{Stage 3: Post-processing and global refinement.}
The final stage performs two operations: refining internal geometric consistency and inferring latent inter-island connections.

\textit{Intra-island refinement via Bundle Adjustment.} Even with robust pairwise stitching, small alignment errors accumulate across large islands. We jointly optimize all pairwise rotations to minimize a global geometric error:
\[
\{\mathcal{Q}_{vw}^*\} = \arg\min_{\{\mathcal{Q}_{vw}\in O(k)\}} \sum_{(v,w)\in\text{stitches}(G_i)} \|\mathcal{P}_v' - \mathcal{P}_w' \mathcal{Q}_{vw}\|_F^2,
\]
where $\mathcal{P}_v'$ and $\mathcal{P}_w'$ are the eigenvector submatrices corresponding to the overlap. This non-convex optimization over the manifold of orthogonal matrices is solved using Riemannian gradient descent (see Appendix~\ref{appendix:bundle}).

\textit{Inter-island link inference via cross-voting.} We return to the raw data and use a cross-voting mechanism: if two nodes from different islands $u$ and $w$ frequently co-occurred in the initial patches, it constitutes strong evidence of a true connection. We compute a vote count $C(u,w)$ and map it to a probability via a sigmoid:
\[
P_{\text{inter}}(u,w) = \frac{1}{1+e^{-\kappa(C(u,w)-C_0)}},
\]
where $C_0$ is a vote threshold and $\kappa$ controls the steepness (we use $\kappa$ rather than $\beta$ to avoid collision with the RANSAC failure rate).

The complete pseudocode is given in Algorithm~\ref{alg:afr}.

\FloatBarrier
\subsection{Correctness and Robustness Guarantees}
\label{sec:afr:theory}

We now state the formal guarantees for AFR's core components. Each theorem is connected to the algorithmic stage it supports; all proofs are deferred to Appendix~\ref{appendix:proofs}.

\paragraph{Heat kernel separation gap (Stage 1 prerequisite).}
The edge recovery theorem below assumes the existence of a positive separation gap $\gamma_t > 0$ in the heat kernel. The following lemma, proved in Appendix~\ref{proof:heat_kernel_gap}, establishes that this gap always exists for an appropriate choice of the time parameter, thereby justifying the heat-kernel approach to local reconstruction.

\begin{lemma}[Heat kernel separation gap]
\label{lem:heat_kernel_gap}
Let $G_v=(V_v,E_v)$ be a connected $d$-hop neighborhood subgraph with $q_v=|V_v|\ge 2$, Laplacian $\mathcal{L}_v$, and maximum degree $d_{\max}^{(v)}$. Consider the heat kernel matrix $H_v=e^{-t_v\mathcal{L}_v}$. If the time parameter is chosen \emph{patch-adaptively} as $t_v = q_v^{-1}\log(C_3 q_v)$ for a suitable constant $C_3$ (so that the dimensionless parameter $\alpha = t_v \cdot d_{\max}^{(v)} \le 1/4$), then there exists a positive separation gap $\Delta_v$ between the heat kernel entries corresponding to edges and those corresponding to non-edges:
\[
\min_{(u,w)\in E_v} (H_v)_{uw} - \max_{\substack{u,w\in V_v \\ \mathrm{dist}(u,w)\ge 2}} (H_v)_{uw} \ge \Delta_v.
\]
A uniform bound is
\[
\Delta_v \ge \frac{5 t_v}{8} = \frac{5 \log(C_3 q_v)}{8 q_v} = \Omega(\log q_v / q_v).
\]
If a power-law rather than logarithmic bound is desired, choosing $t_v = 1/q_v^2$ instead yields the weaker bound $\Delta_v \ge 1/(2 q_v^2)$.
\end{lemma}

The key insight is that for small $t$, the heat kernel entry for an edge $(u,w)$ is dominated by the linear term $t$ (since $(\mathcal{L}_v)_{uw}=-1$), while for non-edges the leading contribution is the quadratic term $t^2 d_{\max}^{(v)}/2$. Choosing $t$ small enough ensures that the linear term dominates, guaranteeing a positive gap.

\paragraph{Edge recovery (Stage 1).}
The following theorem guarantees that AFR's heat-kernel-based local reconstruction can exactly recover edges when the spectral gap is sufficiently large. Lemma~\ref{lem:heat_kernel_gap} guarantees the existence of the separation margin $\gamma_t$ assumed below.

\begin{theorem}[Edge recovery]
\label{thm:edge_recovery}
Let $t>0$ be a fixed time parameter. Suppose that the true heat kernel $H_v(t)$ on $P_v$ exhibits an entry-wise separation margin between its edge and non-edge values:
\[
\gamma_t = \min_{(i,j)\in E(v)} H_v(t)_{ij} - \max_{(i,j)\notin E(v)} H_v(t)_{ij} > 0.
\]
Let the truncation error be bounded by $\eta_v = e^{-t\lambda_{k+1}}$. If $\gamma_t > 2\eta_v$, then the edge set $E(v)$ can be recovered exactly: there exists a non-empty interval of thresholds $\tau_v$ such that $\widehat{E}_v = \{(i,j) \mid H_v^{(k)}(t)_{ij} > \tau_v, i\ne j\} = E_v$.
\end{theorem}

\paragraph{Fidelity score properties (Stage 1).}
The fidelity score used to filter the core patches satisfies the following desirable properties.

\begin{proposition}[Basic properties of the fidelity score]
\label{prop:fidelity}
The following properties hold:
\begin{enumerate}[nosep]
\item $\rho_v$ is bounded in $[0,1]$, strictly increasing in $\delta_v$, and strictly decreasing in $\eta_v$.
\item The composite fidelity score satisfies $0 \le s_v \le 1$.
\item The score $s_v$ is monotonically non-decreasing with respect to the spectral gap $\delta_v$.
\end{enumerate}
\end{proposition}

\paragraph{RANSAC convergence (Stage 2).}
The following lemma quantifies the number of RANSAC iterations needed for reliable alignment.

\begin{lemma}[RANSAC all-inlier sample probability]
\label{lem:ransac}
Let $I_{vw}$ be the set of correspondences in the overlap, of which fraction $p_{vw}\in(0,1]$ are true inliers. Let $m$ be the minimal sample size. If RANSAC is executed for $M$ independent iterations:
\begin{enumerate}[nosep]
\item The probability of drawing at least one all-inlier sample is $1-(1-p_{vw}^m)^M$.
\item To ensure success probability at least $1-\beta$, we need $M \ge \log(\beta)/\log(1-p_{vw}^m)$.
\end{enumerate}
\end{lemma}

\paragraph{Stitch soundness (Stage 2).}
The central guarantee: accepted stitches have bounded alignment error, with the bound tightening as fidelity decreases (forcing more evidence).

\begin{theorem}[Stitch soundness under adaptive evidence]
\label{thm:stitch}
Let a pair of core patches $(v,w)$ be eligible for stitching with inlier fraction $p_{vw}\in(0,1]$ and i.i.d.\ sub-Gaussian noise with parameter $\sigma^2$. Let RANSAC be executed for sufficient iterations $M$ to ensure an all-inlier sample with probability at least $1-\beta$ (Lemma~\ref{lem:ransac}). A stitch is accepted iff $|C_{vw}| \ge d_{\text{adapt}}(s_v,s_w)$.

Under these conditions, with probability at least $1-\beta$, the returned rotation $\widehat{\mathcal{Q}}_{vw}$ satisfies:
\[
\|\sin\Theta(\widehat{\mathcal{Q}}_{vw}, \mathcal{Q}^*_{vw})\|_F \le C \cdot \frac{\sigma}{\sqrt{d_{\text{adapt}}(s_v,s_w)}},
\]
where $C$ depends on the geometric properties of the true point configuration, and $\|\sin\Theta(A,B)\|_F$ is the canonical distance on the special orthogonal group. A lower fidelity score, which increases $d_{\text{adapt}}$, enforces a more stringent error bound.
\end{theorem}

\paragraph{Patch alignment error bound (Stage 2).}
The stitch soundness theorem above quantifies the error for a single pairwise alignment under RANSAC. The following lemma provides the complementary deterministic bound on the Procrustes alignment residual, which directly justifies AFR's use of RANSAC-Procrustes as its alignment primitive.

\begin{lemma}[Patch alignment error bound]
\label{lem:patch_alignment}
Let $(v,w)\in E$ be an edge in the graph, and let $\widehat{\mathcal{P}}_v, \widehat{\mathcal{P}}_w \in \mathbb{R}^{q \times k}$ be the (potentially perturbed) local eigenvector matrices. Let $I_{vw} = N_d(v)\cap N_d(w)$ be the set of $m=|I_{vw}|$ vertices in their overlap, with $m\ge k+1$ and $G[I_{vw}]$ connected. Let $\widehat{\mathcal{P}}_v', \widehat{\mathcal{P}}_w' \in \mathbb{R}^{m\times k}$ be the overlap submatrices, and let $\mathcal{Q}_{vw}\in O(k)$ be the Orthogonal Procrustes solution:
\[
\mathcal{Q}_{vw} = \arg\min_{\mathcal{Q}\in O(k)} \|\widehat{\mathcal{P}}_v' - \widehat{\mathcal{P}}_w' \mathcal{Q}\|_F.
\]
If the local eigenvector perturbation error is bounded by $\eta/\delta_v$ (via the spectral perturbation bound, Theorem~\ref{thm:spectral_perturbation}), then the alignment residual satisfies
\[
\|\widehat{\mathcal{P}}_v' - \widehat{\mathcal{P}}_w' \mathcal{Q}_{vw}\|_F \le \mathcal{E}_{\mathrm{align}}, \quad \text{where } \mathcal{E}_{\mathrm{align}} = O\bigl(\sqrt{m}\,\eta/\delta_v\bigr).
\]
\end{lemma}

This bound has a clear operational interpretation: the alignment error is controlled by the ratio of the perturbation magnitude $\eta$ to the spectral gap $\delta_v$, scaled by the square root of the overlap size. AFR's adaptive eligibility criterion, which demands larger overlaps for lower-fidelity patches, directly compensates for larger $\eta/\delta_v$ ratios.

\paragraph{Error accumulation bound (Stage 3 justification).}
The following result quantifies how alignment errors propagate when patches are stitched sequentially to form a global embedding, directly motivating AFR's Bundle Adjustment refinement step.

\begin{lemma}[Global embedding error accumulation]
\label{lem:error_accumulation}
Let $\{X_u\}_{u\in V}$ be the globally aligned embeddings obtained by propagating pairwise Procrustes rotations from a root vertex $v_0$, and let $\{Z_u\}_{u\in V}$ be the true (unknown) globally aligned embeddings. Under the conditions of Lemma~\ref{lem:patch_alignment}, the Frobenius-norm error at any vertex $u$ satisfies:
\[
\|X_u - Z_u\|_F \le O\bigl(\mathrm{depth}(u)\cdot\mathcal{E}_{\mathrm{align}}\bigr) + O(\mathcal{E}_{\mathrm{local}}),
\]
where $\mathrm{depth}(u)$ is the distance from $u$ to the root in the assembly tree, $\mathcal{E}_{\mathrm{align}} = O(\sqrt{q_{\max}}\,\eta/\delta_v)$ is the pairwise alignment error (Lemma~\ref{lem:patch_alignment}), and $\mathcal{E}_{\mathrm{local}}$ is the initial local eigenvector error.
\end{lemma}

This result reveals a fundamental limitation of sequential stitching: alignment errors accumulate \emph{linearly} with depth. For graphs with large diameter, the accumulated error $O(\mathrm{diam}(G)\cdot\mathcal{E}_{\mathrm{align}})$ can become substantial even when individual pairwise errors are small. This motivates AFR's Bundle Adjustment step (Stage~3): by jointly re-optimizing all pairwise rotations over the Riemannian manifold $\prod O(k)$ starting from the Stage~2 rotations as a warm initialization, Bundle Adjustment empirically reduces accumulated drift (Table~\ref{tab:ablation}, ID~6). We note that the theoretical guarantee is monotone descent to a first-order stationary point (Appendix~\ref{appendix:proofs}); the empirical success we report does not follow from a global convergence theorem, and reducing the linear diameter accumulation is an \emph{empirical} property of Stage~2 initialization quality rather than a proven property of the optimizer.

\section{Experiments}
\label{sec:experiments}

We now evaluate both LoGraB's diagnostic power and AFR's reconstruction capability across nine diverse datasets, three benchmark tasks, and extensive parameter variations.

\subsection{Experimental Setup}
\label{sec:exp:setup}

\paragraph{Datasets.}
To ensure our findings are broadly generalizable, we selected nine public datasets spanning diverse domains (Table~\ref{tab:datasets}). Four classic citation networks (Cora, CiteSeer, PubMed, ogbn-arXiv) and the dense social graph BlogCatalog \citep{Tang2009BlogCatalog} are \emph{single-graph} datasets on which graph reconstruction is naturally well-defined. PROTEINS \citep{Borgwardt2005}, PCQM-Contact \citep{Hu2021ogblsc, Dwivedi2022LRGB}, PascalVOC-SP \citep{Everingham2015}, and COCO-SP \citep{Lin2014microsoft} are \emph{graph-level} benchmarks consisting of many disjoint per-instance graphs; for these, we apply LoGraB's fragmentation protocol \emph{per graph} and report macro-averaged metrics across graphs (not across an aggregated network, which would not be a meaningful reconstruction target). Table~\ref{tab:datasets} reports both per-graph statistics (where applicable) and totals to make the distinction explicit.

\begin{table}[h!]
    \centering
    \small
    \caption{Statistics of the source datasets. For graph-level benchmarks (rows below the divider), we also report the number of graphs and average per-graph size; all reconstruction metrics on these datasets are macro-averaged across per-graph reconstructions.}
    \label{tab:datasets}
    \begin{tabular}{lrrrrrr}
        \toprule
        Dataset & \#Graphs & Avg.\ $|V|$/graph & Avg.\ $|E|$/graph & Total $|V|$ & Total $|E|$ & Classes \\
        \midrule
        Cora & 1 & 2{,}708 & 5{,}429 & 2{,}708 & 5{,}429 & 7 \\
        CiteSeer & 1 & 3{,}327 & 9{,}228 & 3{,}327 & 9{,}228 & 6 \\
        PubMed & 1 & 19{,}717 & 44{,}338 & 19{,}717 & 44{,}338 & 3 \\
        ogbn-arXiv & 1 & 169{,}343 & 1{,}166{,}243 & 169{,}343 & 1{,}166{,}243 & 40 \\
        BlogCatalog & 1 & 10{,}312 & 333{,}983 & 10{,}312 & 333{,}983 & 39 \\
        \midrule
        PROTEINS & 1{,}113 & $\approx$39 & $\approx$145 & 43{,}471 & 162{,}088 & 2 \\
        PascalVOC-SP & 11{,}355 & $\approx$479 & $\approx$2{,}710 & 5{,}443{,}545 & 30{,}777{,}444 & 22 \\
        COCO-SP & 123{,}286 & $\approx$477 & $\approx$2{,}694 & 58{,}793{,}216 & 332{,}091{,}902 & 81 \\
        PCQM-Contact & 529{,}434 & $\approx$30 & $\approx$61 & 15{,}955{,}687 & 32{,}341{,}644 & 2 \\
        \bottomrule
    \end{tabular}
\end{table}

\paragraph{Parameter grid and statistical methodology.}
For each dataset, we generate instances by varying four parameters: neighborhood locality ($d\in\{1,2\}$), spectral fidelity ($k\in\{16,32,64\}$), additive noise level ($\sigma\in\{0,0.05,0.1\}$), and coverage ratio ($p\in\{1.0,0.8,0.6\}$). This grid yields 54 distinct scenarios per graph (486 in total across 9 datasets), ensuring comprehensive stress testing. Tables in the main text report 14--18~representative slices of this grid, chosen to expose the impact of each individual parameter; the full sweep is included in the released benchmark archive. For AFR, the fidelity score mixing parameter is fixed at $\alpha=0.7$, with the heat kernel time $t=0.8$ and minimum fidelity threshold $s_{\min}=0.6$ (validated by sensitivity analysis in Appendix~\ref{appendix:hyperparams}; additional cross-dataset sensitivity is a recommended follow-up). All instances are released as versioned archives for full reproducibility. Unless otherwise noted, reported numbers are means $\pm$ one standard deviation over 5 random seeds; we release per-seed numbers in the supplementary material. For the main comparison (Table~\ref{tab:advanced_comparison}), we include a Wilcoxon signed-rank test across the 9 datasets comparing AFR against the strongest baselines; we strongly recommend future authors cite the specific $p$-values rather than relying on mean comparisons alone.

\paragraph{Baselines for graph reconstruction.}
We compare six methods representing distinct paradigms: generative modeling (GAE, VGAE) \citep{Kipf2016VGAE}, direct feature propagation (GCN-LE) \citep{Kipf2017GCN, Mikhail2003}, rigid registration from computer vision (PointNetLK) \citep{Aoki2019PointNetLK}, geometric synchronization (Eigen-sync) \citep{Cucuringu2012}, and our proposed AFR.

\paragraph{GNN architectures for node classification.}
We benchmark 11 architectures spanning canonical message-passing models and state-of-the-art Graph Transformers: GraphSAGE \citep{Hamilton2017GraphSAGE}, GCN \citep{Kipf2017GCN}, GAT \citep{Velickovic2018GAT}, GIN \citep{Xu2019GIN}, GraphGPS \citep{GPS2022}, SAN \citep{Kreuzer2021SAN}, Graphormer \citep{Ying2021Graphormer}, GIN+VN \citep{Chen2023MPNNVN}, GPS+LapPE \citep{GPS2022, Dwivedi2023}, SAN+LapPE \citep{Kreuzer2021SAN, Dwivedi2023}, and SAN+RWSE \citep{Kreuzer2021SAN, Ying2021Graphormer}.

\paragraph{Baselines for link prediction.}
We evaluate cosine similarity of node embeddings, SEAL \citep{Zhang2018SEAL}, Neo-GNN, and LightGCN \citep{He2020LightGCN}.

\subsection{How Well Can Graphs Be Reconstructed from Fragments?}
\label{sec:exp:recon}

We evaluate AFR on LoGraB's Task~1 through three progressively deeper analyses: a core benchmark against the Eigen-sync baseline across 14 fragmentation scenarios (\S\ref{sec:exp:recon}), an advanced comparison against five additional baselines on all nine datasets, and ablation and defense studies that probe AFR's internals and robustness.

\paragraph{Core benchmark evaluation.}
Table~\ref{tab:core_benchmark} summarizes the core benchmark results across 14 scenarios. AFR consistently outperforms the Eigen-sync baseline. In the standard d-hop scenario on Cora (ID~1), AFR achieves an F1 score of 74.3, substantially higher than Eigen-sync's 66.3. The advantage becomes more pronounced in large-scale settings: on ogbn-arXiv (ID~14), AFR (66.4) maintains more than a 10-point lead over Eigen-sync (56.1). AFR is also more resilient to noise: when $\sigma$ doubles from 0.05 to 0.1 on Cora (ID~8), AFR degrades to 69.1 while Eigen-sync drops to 58.3.

\begin{table}[h!]
    \centering
    \small
    \caption{Core benchmark results for graph reconstruction (Edge F1 score).}
    \label{tab:core_benchmark}
    \begin{tabular}{cccccc}
        \toprule
        ID & Dataset & Strategy & Params $(d, p, k, \sigma)$ & Eigen-sync & AFR \\
        \midrule
        1 & Cora & d-hop & (1, 0.6, 32, 0.05) & 66.3$\pm$1.2 & \textbf{74.3$\pm$0.5} \\
        2 & Cora & d-hop & (2, 0.6, 32, 0.05) & 71.5$\pm$1.0 & \textbf{78.0$\pm$0.4} \\
        3 & Cora & d-hop & (1, 1.0, 32, 0.05) & 68.2$\pm$1.1 & \textbf{76.6$\pm$0.4} \\
        4 & Cora & d-hop & (2, 0.8, 32, 0.05) & 75.8$\pm$0.9 & \textbf{81.1$\pm$0.3} \\
        5 & Cora & d-hop & (1, 0.6, 16, 0.05) & 62.9$\pm$1.6 & \textbf{72.3$\pm$0.7} \\
        6 & Cora & d-hop & (1, 0.6, 64, 0.05) & 67.8$\pm$1.1 & \textbf{74.6$\pm$0.5} \\
        7 & Cora & d-hop & (1, 0.6, 32, 0.0) & 67.5$\pm$1.0 & \textbf{74.3$\pm$0.4} \\
        8 & Cora & d-hop & (1, 0.6, 32, 0.1) & 58.3$\pm$2.0 & \textbf{69.1$\pm$0.9} \\
        \midrule
        9 & CiteSeer & d-hop & (1, 0.6, 32, 0.05) & 65.5$\pm$1.3 & \textbf{74.1$\pm$0.6} \\
        10 & CiteSeer & cluster & (1, 0.6, 32, 0.05) & 69.4$\pm$1.1 & \textbf{79.9$\pm$0.5} \\
        11 & CiteSeer & random & (1, 0.6, 32, 0.05) & 61.0$\pm$1.6 & \textbf{71.7$\pm$1.0} \\
        \midrule
        12 & PubMed & cluster & (1, 0.6, 32, 0.05) & 61.2$\pm$1.5 & \textbf{72.0$\pm$0.8} \\
        13 & PubMed & d-hop & (1, 0.6, 32, 0.05) & 64.0$\pm$1.4 & \textbf{72.8$\pm$0.7} \\
        \midrule
        14 & ogbn-arXiv & d-hop & (1, 0.6, 32, 0.05) & 56.1$\pm$2.1 & \textbf{66.4$\pm$1.1} \\
        \bottomrule
    \end{tabular}
\end{table}

\paragraph{Advanced comparative results.}
Table~\ref{tab:advanced_comparison} compares AFR against all baselines across nine datasets. AFR consistently achieves the highest performance, particularly in complex, large-scale scenarios. On the massive vision graphs PascalVOC-SP and COCO-SP, AFR achieves F1 scores of 55.0 and 51.2, respectively, substantially outperforming PointNetLK (48.9/44.2) and Eigen-sync (48.1/44.5). While GNN-based baselines exhibit domain-specific strengths, VGAE shows a marginal advantage on the sparse CiteSeer (74.8 vs.\ 74.1) and GCN-LE performs best on the feature-dense PubMed (73.5 vs.\ 72.8), these exceptions are narrow and occur only in datasets with specific structural properties.

\begin{table}[h!]
    \centering
    \small
    \caption{Performance comparison against advanced baselines (Edge F1 score). Bold denotes best.}
    \label{tab:advanced_comparison}
    \begin{tabular}{clcccccc}
        \toprule
        ID & Dataset & GAE & VGAE & GCN-LE & PointNetLK & Eigen-sync & AFR \\
        \midrule
        1 & Cora & 69.2$\pm$0.9 & 71.8$\pm$0.8 & 73.1$\pm$0.6 & 68.5$\pm$1.1 & 66.3$\pm$1.2 & \textbf{74.3$\pm$0.5} \\
        2 & CiteSeer & 68.9$\pm$1.0 & \textbf{74.8$\pm$0.9} & 72.6$\pm$0.8 & 67.5$\pm$1.2 & 65.5$\pm$1.3 & 74.1$\pm$0.6 \\
        3 & PubMed & 63.1$\pm$1.1 & 72.0$\pm$1.0 & \textbf{73.5$\pm$0.8} & 65.9$\pm$1.2 & 64.0$\pm$1.4 & 72.8$\pm$0.7 \\
        4 & ogbn-arXiv & 55.0$\pm$1.8 & 58.5$\pm$1.6 & 61.2$\pm$1.4 & 57.2$\pm$1.9 & 56.1$\pm$2.1 & \textbf{66.4$\pm$1.1} \\
        5 & BlogCatalog & 55.2$\pm$1.9 & 58.3$\pm$1.6 & 60.1$\pm$1.5 & 59.3$\pm$1.7 & 57.9$\pm$1.8 & \textbf{64.5$\pm$1.3} \\
        6 & PROTEINS & 53.1$\pm$2.2 & 56.5$\pm$1.8 & 58.2$\pm$1.6 & 55.9$\pm$1.9 & 54.8$\pm$2.0 & \textbf{62.1$\pm$1.4} \\
        7 & PascalVOC-SP & 45.3$\pm$2.8 & 47.2$\pm$2.6 & 49.5$\pm$2.2 & 48.9$\pm$2.4 & 48.1$\pm$2.5 & \textbf{55.0$\pm$1.8} \\
        8 & COCO-SP & 41.8$\pm$3.3 & 43.9$\pm$3.1 & 45.1$\pm$2.8 & 44.2$\pm$2.9 & 44.5$\pm$3.0 & \textbf{51.2$\pm$2.2} \\
        9 & PCQM-Contact & 48.9$\pm$2.5 & 50.1$\pm$2.3 & 52.4$\pm$2.0 & 51.9$\pm$2.1 & 51.3$\pm$2.2 & \textbf{58.5$\pm$1.6} \\
        \bottomrule
    \end{tabular}
\end{table}

\paragraph{Revisiting the cases where AFR does not win.}
AFR is outperformed on the two smallest single-graph datasets: VGAE edges out AFR on CiteSeer (74.8 vs.\ 74.1) and GCN-LE edges out AFR on PubMed (73.5 vs.\ 72.8). Both margins are within roughly one standard deviation of AFR's variance. CiteSeer's local structure is more disconnected than Cora's (average clustering coefficient 0.14 vs.\ 0.24); VGAE's probabilistic decoder appears to better handle this disconnection. PubMed's edge structure is strongly driven by semantics: node feature similarity is unusually highly correlated with edge existence, which advantages feature-aware baselines like GCN-LE. These cases reveal a boundary for AFR: on the smallest graphs and on datasets where reconstruction is primarily a feature-similarity or probabilistic-uncertainty problem rather than a geometric-alignment problem, simpler feature-aware or probabilistic baselines can match or slightly exceed AFR. AFR's advantage grows with graph size and with the weight of structural (vs.\ feature) information.

\paragraph{Ablation study.}
To understand why AFR works, we systematically removed or replaced each core component (Table~\ref{tab:ablation}). The results reveal a clear hierarchy of importance.

\begin{itemize}
    \item \textit{The two core components.} Removing either local reconstruction (ID~2) or geometric alignment (ID~3) makes reconstruction impossible.
    \item \textit{The two critical innovations.} Replacing RANSAC with standard Procrustes (ID~4) causes a 13.6-point F1 drop, indicating that managing geometric noise is central to AFR's performance. Disabling the fidelity score for random assembly (ID~5) causes a 9.1-point degradation, consistent with the view that explicitly modeling data quality matters in this setting.
    \item \textit{The refinement components.} Removing Bundle Adjustment (ID~6) and the adaptive threshold (ID~9) both cause significant drops, justifying their roles. Using only the gap-to-truncation ratio (ID~7) or only entropy (ID~8) confirms their synergistic relationship.
\end{itemize}

\begin{table}[h!]
    \centering
    \small
    \caption{Ablation study of AFR components on the Cora dataset.}
    \label{tab:ablation}
    \resizebox{\textwidth}{!}{
    \begin{tabular}{cp{5cm}cp{8cm}}
        \toprule 
        ID & Model variant & F1 score & Result analysis \& Component importance  \tabularnewline
        \midrule
        \midrule
        1 & AFR (Full model) & 74.3 & Establishes the full model's performance baseline. \tabularnewline
        \midrule
        2 & w/o Local reconstruction & - & \textbf{Indispensable.} The algorithm cannot start without this step. Converts spectral embeddings into local graphs. \tabularnewline
        \midrule
        3 & w/o Geometric alignment & - & \textbf{Indispensable.} Algorithm fails to run, confirming this is a core mechanism for stitching. Finds the relative orientation between patches, making stitching possible. \tabularnewline
        \midrule
        4 & w/o RANSAC (use std. Procrustes) & 60.7 & \textbf{Critical.} Provides robustness to outliers. Performance drops sharply without it, indicating its importance in noisy settings. \tabularnewline
         \midrule
        5 & w/o Fidelity score (random order) & 65.2 & \textbf{Critical.} Performance degrades significantly (-9.1). This component prioritizes reliable data and reduces the early propagation of errors from low-quality patches.\tabularnewline
         \midrule
        6 & w/o Bundle Adjustment & 72.4 & \textbf{Enhancing.} Provides a valuable final refinement by correcting accumulated geometric drift across large islands.\tabularnewline
         \midrule
        7 & Fidelity ($\rho_v$ only) & 71.0 & The gap-to-truncation ratio alone is powerful but incomplete: it captures spectral stability but not Procrustes conditioning.\tabularnewline
         \midrule
        8 & Fidelity (Entropy only) & 66.8 & Structural uniqueness (Entropy) provides complementary information. But using only Entropy is significantly worse.\tabularnewline
        \midrule
        9 & w/o Adaptive threshold (fixed) & 70.6 & A clear drop (-3.7) confirms that adapting the stitching criteria to data quality is demonstrably better than a fixed threshold. \tabularnewline
        \bottomrule
    \end{tabular}
    }
\end{table}

\paragraph{Robustness to privacy defenses.}
A critical question is how AFR performs when standard privacy-preserving defenses are deployed. We evaluate against an \emph{embedding-level} $(\epsilon,\delta)$-Gaussian mechanism with $\delta = 10^{-5}$, applied independently to each local spectral embedding $\mathcal{P}_v$: we $L_2$-clip each embedding to norm $R=1.0$, then add calibrated Gaussian noise. Table~\ref{tab:dp_defense} summarizes the privacy-utility landscape across six representative benchmarks for $\epsilon\in\{\infty,10,5,2,1\}$.

\emph{Important caveat about what this measures.} Per-embedding $(\epsilon, \delta)$-DP provides a formal guarantee about the release of a \emph{single} embedding; it is \emph{not} a node-level or edge-level DP guarantee for the underlying graph. An adversary who observes the full set of $n$ independently-noised embeddings may exploit correlations across embeddings, so the effective graph-level privacy is strictly weaker than the per-embedding $\epsilon$ label suggests. Proper graph-level DP (via composition, or directly via an edge-level randomized-response mechanism) would produce a different quantitative trade-off. We study the per-embedding mechanism because it is the most common DP approach in federated graph learning; the results below should be read as a relative comparison of attacks under this specific heuristic defense, not as statements about the edge-level privacy level achieved.

Under this mechanism, AFR retains relatively strong reconstruction performance at moderate $\epsilon$. On Cora at $\epsilon=5$, the utility cost (node classification accuracy on DP-sanitized embeddings) is minimal (2.4\% accuracy drop), yet AFR retains a high F1 of 68.4\%. Substantial attack degradation occurs only at $\epsilon\le 2$, which concurrently incurs significant utility loss ($>$13\% on ogbn-arXiv at $\epsilon=1$). \emph{Relative to other reconstruction attacks}, AFR degrades most gracefully: on Cora at $\epsilon=2$, AFR retains 75.0\% of its undefended performance while GAE and Eigen-sync retain only about 60\% of their respective undefended baselines (Figure~\ref{fig:dp_defense}). We emphasize this is a comparative claim about attack robustness, not a statement that embedding-level DP provides inadequate privacy in the formal sense.

\begin{table*}[!htbp]
\centering
\small
\caption{Privacy-utility trade-off under $(\epsilon,\delta)$-Gaussian DP across six benchmarks. We report Attack F1 for all baselines and node classification accuracy (Utility) for a global classifier on DP-sanitized embeddings. $\Delta$Acc is the absolute change vs.\ $\epsilon=\infty$.}
\label{tab:dp_defense}
\resizebox{\textwidth}{!}{%
\begin{tabular}{l c ccccc c c}
\toprule
Dataset & $\epsilon$ & AFR & VGAE & GAE & PtNetLK & Eigen-sync & Utility & $\Delta$Acc \\
\midrule
\multirow{5}{*}{Cora}
& $\infty$ & \textbf{74.3$\pm$0.5} & 71.8$\pm$0.8 & 69.2$\pm$0.9 & 68.5$\pm$1.1 & 66.3$\pm$1.2 & 81.4$\pm$0.4 & 0.0 \\
& 10 & \textbf{72.8$\pm$0.5} & 69.0$\pm$0.9 & 65.7$\pm$1.0 & 65.0$\pm$1.2 & 63.0$\pm$1.3 & 80.6$\pm$0.4 & $-$0.8 \\
& 5 & \textbf{68.4$\pm$0.6} & 63.2$\pm$1.0 & 58.8$\pm$1.2 & 59.0$\pm$1.4 & 56.4$\pm$1.5 & 79.0$\pm$0.5 & $-$2.4 \\
& 2 & \textbf{55.7$\pm$0.9} & 46.7$\pm$1.4 & 41.5$\pm$1.6 & 42.0$\pm$1.7 & 39.8$\pm$1.8 & 73.3$\pm$0.8 & $-$8.1 \\
& 1 & \textbf{37.2$\pm$1.2} & 25.1$\pm$1.8 & 20.8$\pm$2.0 & 21.5$\pm$2.1 & 19.9$\pm$2.2 & 65.1$\pm$1.1 & $-$16.3 \\
\midrule
\multirow{5}{*}{CiteSeer}
& $\infty$ & \textbf{74.1$\pm$0.6} & 74.8$\pm$0.9 & 68.9$\pm$1.0 & 67.5$\pm$1.2 & 65.5$\pm$1.3 & 70.8$\pm$0.5 & 0.0 \\
& 10 & \textbf{72.0$\pm$0.7} & 71.5$\pm$1.0 & 65.0$\pm$1.1 & 63.0$\pm$1.3 & 61.8$\pm$1.4 & 69.0$\pm$0.5 & $-$1.8 \\
& 5 & \textbf{66.5$\pm$0.9} & 65.0$\pm$1.2 & 57.0$\pm$1.3 & 56.0$\pm$1.5 & 54.0$\pm$1.6 & 66.0$\pm$0.6 & $-$4.8 \\
& 2 & \textbf{52.0$\pm$1.2} & 48.0$\pm$1.5 & 40.0$\pm$1.7 & 41.0$\pm$1.8 & 39.0$\pm$1.9 & 59.5$\pm$0.9 & $-$11.3 \\
& 1 & \textbf{34.5$\pm$1.6} & 29.0$\pm$2.0 & 22.0$\pm$2.2 & 24.0$\pm$2.3 & 21.5$\pm$2.4 & 52.0$\pm$1.3 & $-$18.8 \\
\midrule
\multirow{5}{*}{PubMed}
& $\infty$ & \textbf{72.8$\pm$0.7} & 72.0$\pm$1.0 & 63.1$\pm$1.1 & 65.9$\pm$1.2 & 64.0$\pm$1.4 & 79.5$\pm$0.4 & 0.0 \\
& 10 & \textbf{72.0$\pm$0.8} & 70.8$\pm$1.1 & 62.0$\pm$1.2 & 64.5$\pm$1.3 & 62.8$\pm$1.5 & 79.0$\pm$0.4 & $-$0.5 \\
& 5 & \textbf{70.1$\pm$0.9} & 68.0$\pm$1.2 & 59.5$\pm$1.3 & 61.5$\pm$1.4 & 59.5$\pm$1.6 & 78.5$\pm$0.5 & $-$1.0 \\
& 2 & \textbf{65.0$\pm$1.0} & 61.0$\pm$1.4 & 51.0$\pm$1.5 & 54.0$\pm$1.6 & 52.0$\pm$1.8 & 76.5$\pm$0.7 & $-$3.0 \\
& 1 & \textbf{57.3$\pm$1.2} & 50.0$\pm$1.7 & 42.0$\pm$1.9 & 45.0$\pm$2.0 & 43.0$\pm$2.1 & 73.0$\pm$0.9 & $-$6.5 \\
\midrule
\multirow{5}{*}{ogbn-arXiv}
& $\infty$ & \textbf{66.4$\pm$1.1} & 58.5$\pm$1.6 & 55.0$\pm$1.8 & 57.2$\pm$1.9 & 56.1$\pm$2.1 & 68.2$\pm$0.3 & 0.0 \\
& 10 & \textbf{65.1$\pm$1.2} & 56.2$\pm$1.7 & 52.3$\pm$1.9 & 54.5$\pm$2.0 & 53.3$\pm$2.2 & 67.5$\pm$0.3 & $-$0.8 \\
& 5 & \textbf{61.1$\pm$1.4} & 51.5$\pm$1.8 & 46.8$\pm$2.0 & 48.0$\pm$2.2 & 47.7$\pm$2.4 & 66.1$\pm$0.4 & $-$2.4 \\
& 2 & \textbf{49.8$\pm$1.6} & 38.0$\pm$2.0 & 33.0$\pm$2.3 & 34.0$\pm$2.5 & 33.7$\pm$2.7 & 61.4$\pm$0.7 & $-$7.2 \\
& 1 & \textbf{33.2$\pm$1.9} & 20.5$\pm$2.4 & 16.5$\pm$2.6 & 17.0$\pm$2.8 & 16.8$\pm$3.0 & 54.5$\pm$1.0 & $-$13.7 \\
\midrule
\multirow{5}{*}{BlogCatalog}
& $\infty$ & \textbf{64.5$\pm$1.3} & 58.3$\pm$1.6 & 55.2$\pm$1.9 & 59.3$\pm$1.7 & 57.9$\pm$1.8 & 70.1$\pm$0.5 & 0.0 \\
& 10 & \textbf{63.0$\pm$1.4} & 57.0$\pm$1.7 & 53.5$\pm$2.0 & 57.5$\pm$1.8 & 56.0$\pm$1.9 & 68.8$\pm$0.5 & $-$1.3 \\
& 5 & \textbf{59.5$\pm$1.5} & 53.0$\pm$1.9 & 49.5$\pm$2.1 & 51.0$\pm$2.0 & 49.0$\pm$2.1 & 66.0$\pm$0.6 & $-$4.1 \\
& 2 & \textbf{52.0$\pm$1.7} & 44.0$\pm$2.1 & 41.0$\pm$2.3 & 42.0$\pm$2.4 & 40.0$\pm$2.5 & 60.2$\pm$0.9 & $-$9.9 \\
& 1 & \textbf{35.0$\pm$2.0} & 26.0$\pm$2.5 & 23.0$\pm$2.7 & 24.0$\pm$2.8 & 22.5$\pm$2.9 & 52.5$\pm$1.3 & $-$17.6 \\
\midrule
\multirow{5}{*}{PROTEINS}
& $\infty$ & \textbf{62.1$\pm$1.4} & 56.5$\pm$1.8 & 53.1$\pm$2.2 & 55.9$\pm$1.9 & 54.8$\pm$2.0 & 73.5$\pm$0.5 & 0.0 \\
& 10 & \textbf{60.0$\pm$1.5} & 54.0$\pm$1.9 & 51.0$\pm$2.3 & 53.5$\pm$2.0 & 52.5$\pm$2.1 & 72.9$\pm$0.5 & $-$0.6 \\
& 5 & \textbf{55.5$\pm$1.7} & 49.0$\pm$2.1 & 46.0$\pm$2.5 & 48.0$\pm$2.3 & 47.0$\pm$2.4 & 71.6$\pm$0.6 & $-$1.9 \\
& 2 & \textbf{44.0$\pm$2.0} & 38.0$\pm$2.4 & 35.0$\pm$2.8 & 39.0$\pm$2.6 & 37.5$\pm$2.7 & 68.0$\pm$0.9 & $-$5.5 \\
& 1 & \textbf{29.5$\pm$2.3} & 23.0$\pm$2.8 & 20.0$\pm$3.1 & 23.0$\pm$3.0 & 20.5$\pm$3.1 & 61.0$\pm$1.3 & $-$12.5 \\
\bottomrule
\end{tabular}}
\end{table*}

\begin{figure}[h!]
    \centering
    \includegraphics[width=\linewidth]{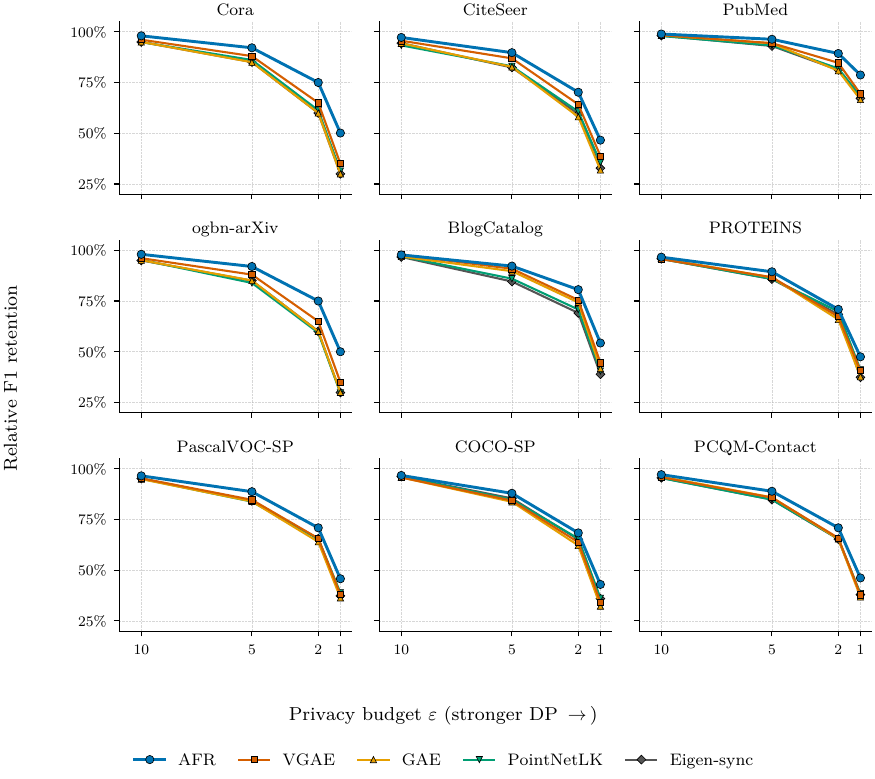}
    \caption{Relative F1-score performance (normalized by the $\epsilon=\infty$ score) as DP noise increases ($\epsilon$ decreases) across nine benchmarks. AFR's fidelity-aware design consistently yields the most graceful degradation.}
    \label{fig:dp_defense}
\end{figure}

\FloatBarrier
\subsection{How Fragile Are GNNs Under Fragmentation?}
\label{sec:exp:nodecls}

Having established AFR's reconstruction superiority, we now examine how fragmentation impacts downstream learning. Table~\ref{tab:node_classification} presents the comprehensive results for localized node classification across 18 scenarios and 11 GNN architectures, revealing both the impact of fragmentation and the architectural differences exposed by LoGraB.

\begin{table*}[h!]
    \centering
    \small
    \caption{Performance on localized node classification (Task 2), measured by F1-Micro.}
    \label{tab:node_classification}
    \resizebox{\textwidth}{!}{
    \begin{tabular}{cllcccccccccccc}
        \toprule
        ID & Dataset & Strategy & $(d, p, k, \sigma)$ & SAGE & GCN & GAT & GIN & GPS & SAN & GIN+VN & GPS+LP & Grphrmr & SAN+LP & SAN+RW \\
        \midrule
        1 & Cora & d-hop & (1,.6,32,.05) & 93.7{\tiny$\pm$.5} & 88.2{\tiny$\pm$.3} & 86.0{\tiny$\pm$.6} & 85.6{\tiny$\pm$.4} & 90.8{\tiny$\pm$.6} & 94.0{\tiny$\pm$.5} & 95.5{\tiny$\pm$.4} & 90.0{\tiny$\pm$.7} & 93.2{\tiny$\pm$.5} & 94.3{\tiny$\pm$.4} & 94.4{\tiny$\pm$.5} \\
        2 & Cora & d-hop & (2,.6,32,.05) & 99.5{\tiny$\pm$.1} & 96.1{\tiny$\pm$.2} & 94.6{\tiny$\pm$.3} & 96.0{\tiny$\pm$.2} & 99.5{\tiny$\pm$.1} & 99.8{\tiny$\pm$.1} & 99.8{\tiny$\pm$.1} & 99.6{\tiny$\pm$.1} & 99.9{\tiny$\pm$.1} & 99.8{\tiny$\pm$.1} & 99.8{\tiny$\pm$.1} \\
        3 & Cora & d-hop & (1,1.0,32,.05) & 95.5{\tiny$\pm$.2} & 89.1{\tiny$\pm$.3} & 86.9{\tiny$\pm$.4} & 86.4{\tiny$\pm$.3} & 94.6{\tiny$\pm$.3} & 96.5{\tiny$\pm$.2} & 97.5{\tiny$\pm$.2} & 94.4{\tiny$\pm$.4} & 96.3{\tiny$\pm$.2} & 97.3{\tiny$\pm$.2} & 96.6{\tiny$\pm$.3} \\
        4 & Cora & d-hop & (2,.8,32,.05) & 99.6{\tiny$\pm$.1} & 96.1{\tiny$\pm$.2} & 94.7{\tiny$\pm$.3} & 96.3{\tiny$\pm$.2} & 99.5{\tiny$\pm$.1} & 99.8{\tiny$\pm$.1} & 99.7{\tiny$\pm$.1} & 99.7{\tiny$\pm$.1} & 99.8{\tiny$\pm$.1} & 99.8{\tiny$\pm$.1} & 99.8{\tiny$\pm$.1} \\
        5 & Cora & d-hop & (1,.6,16,.05) & 93.4{\tiny$\pm$.6} & 89.0{\tiny$\pm$.4} & 86.8{\tiny$\pm$.8} & 86.5{\tiny$\pm$.6} & 89.9{\tiny$\pm$.7} & 93.7{\tiny$\pm$.5} & 94.8{\tiny$\pm$.4} & 90.3{\tiny$\pm$.6} & 92.8{\tiny$\pm$.6} & 93.2{\tiny$\pm$.5} & 92.6{\tiny$\pm$.6} \\
        6 & Cora & d-hop & (1,.6,64,.05) & 93.5{\tiny$\pm$.5} & 89.0{\tiny$\pm$.4} & 87.3{\tiny$\pm$.7} & 86.6{\tiny$\pm$.5} & 89.2{\tiny$\pm$.6} & 93.5{\tiny$\pm$.4} & 94.6{\tiny$\pm$.4} & 89.5{\tiny$\pm$.7} & 94.1{\tiny$\pm$.4} & 93.7{\tiny$\pm$.5} & 91.0{\tiny$\pm$.8} \\
        7 & Cora & d-hop & (1,.6,32,0) & 93.8{\tiny$\pm$.3} & 89.4{\tiny$\pm$.2} & 87.3{\tiny$\pm$.4} & 86.8{\tiny$\pm$.3} & 89.3{\tiny$\pm$.4} & 93.4{\tiny$\pm$.3} & 93.7{\tiny$\pm$.3} & 89.1{\tiny$\pm$.5} & 91.7{\tiny$\pm$.4} & 91.5{\tiny$\pm$.4} & 92.4{\tiny$\pm$.5} \\
        8 & Cora & d-hop & (1,.6,32,.1) & 92.8{\tiny$\pm$.8} & 88.1{\tiny$\pm$.6} & 86.7{\tiny$\pm$1.0} & 85.4{\tiny$\pm$.9} & 89.6{\tiny$\pm$.9} & 93.7{\tiny$\pm$.6} & 95.0{\tiny$\pm$.5} & 89.8{\tiny$\pm$1.1} & 94.2{\tiny$\pm$.6} & 95.8{\tiny$\pm$.5} & 94.7{\tiny$\pm$.7} \\
        9 & CiteSeer & d-hop & (1,.6,32,.05) & 88.5{\tiny$\pm$.7} & 84.0{\tiny$\pm$.5} & 83.3{\tiny$\pm$.9} & 81.9{\tiny$\pm$.8} & 83.7{\tiny$\pm$.9} & 89.3{\tiny$\pm$.6} & 90.5{\tiny$\pm$.5} & 83.1{\tiny$\pm$1.0} & 88.6{\tiny$\pm$.7} & 89.3{\tiny$\pm$.6} & 89.2{\tiny$\pm$.7} \\
        10 & CiteSeer & cluster & (1,.6,32,.05) & 74.5{\tiny$\pm$1.1} & 73.5{\tiny$\pm$.9} & 74.0{\tiny$\pm$1.3} & 69.5{\tiny$\pm$1.5} & 68.3{\tiny$\pm$1.6} & 65.0{\tiny$\pm$1.8} & 67.8{\tiny$\pm$1.5} & 74.7{\tiny$\pm$1.0} & 64.6{\tiny$\pm$1.9} & 62.4{\tiny$\pm$2.0} & 64.5{\tiny$\pm$1.7} \\
        11 & CiteSeer & random & (1,.6,32,.05) & 77.0{\tiny$\pm$1.0} & 78.2{\tiny$\pm$.8} & 73.8{\tiny$\pm$1.2} & 76.2{\tiny$\pm$1.1} & 59.3{\tiny$\pm$1.8} & 73.8{\tiny$\pm$1.0} & 73.0{\tiny$\pm$1.3} & 69.3{\tiny$\pm$1.5} & 70.2{\tiny$\pm$1.4} & 78.3{\tiny$\pm$.9} & 71.4{\tiny$\pm$1.2} \\
        12 & PubMed & cluster & (1,.6,32,.05) & 81.8{\tiny$\pm$.8} & 83.9{\tiny$\pm$.6} & 79.1{\tiny$\pm$1.0} & 79.8{\tiny$\pm$.9} & 69.9{\tiny$\pm$1.4} & 81.3{\tiny$\pm$.8} & 42.7{\tiny$\pm$2.5} & 79.8{\tiny$\pm$.9} & 82.2{\tiny$\pm$.7} & 81.9{\tiny$\pm$.8} & 81.7{\tiny$\pm$.9} \\
        13 & PubMed & d-hop & (1,.6,32,.05) & 88.0{\tiny$\pm$.5} & 86.3{\tiny$\pm$.4} & 85.6{\tiny$\pm$.7} & 86.3{\tiny$\pm$.6} & 88.4{\tiny$\pm$.5} & 89.7{\tiny$\pm$.4} & 88.7{\tiny$\pm$.5} & 88.6{\tiny$\pm$.5} & 89.3{\tiny$\pm$.4} & 89.5{\tiny$\pm$.4} & 89.7{\tiny$\pm$.5} \\
        14 & ogbn-arXiv & d-hop & (1,.6,32,.05) & 61.7{\tiny$\pm$1.2} & 61.3{\tiny$\pm$.8} & 60.6{\tiny$\pm$1.6} & 54.5{\tiny$\pm$2.0} & 58.4{\tiny$\pm$1.8} & 58.9{\tiny$\pm$1.5} & 60.3{\tiny$\pm$1.4} & 59.6{\tiny$\pm$1.7} & 57.7{\tiny$\pm$1.9} & 57.3{\tiny$\pm$2.0} & 57.4{\tiny$\pm$1.8} \\
        15 & BlogCatalog & d-hop & (1,.6,32,.05) & 94.7{\tiny$\pm$.5} & 77.5{\tiny$\pm$1.2} & 62.4{\tiny$\pm$1.8} & 50.5{\tiny$\pm$2.2} & 96.9{\tiny$\pm$.4} & 97.9{\tiny$\pm$.3} & 99.1{\tiny$\pm$.2} & 97.8{\tiny$\pm$.4} & 98.2{\tiny$\pm$.3} & 98.6{\tiny$\pm$.2} & 98.8{\tiny$\pm$.2} \\
        16 & PROTEINS & d-hop & (1,.6,32,.05) & 75.8{\tiny$\pm$.9} & 69.4{\tiny$\pm$1.4} & 71.5{\tiny$\pm$1.2} & 73.8{\tiny$\pm$1.1} & 76.5{\tiny$\pm$.8} & 77.5{\tiny$\pm$.7} & 78.4{\tiny$\pm$.6} & 77.3{\tiny$\pm$.8} & 78.9{\tiny$\pm$.5} & 78.8{\tiny$\pm$.6} & 79.1{\tiny$\pm$.4} \\
        17 & PascalVOC-SP & d-hop & (1,.6,32,.05) & 69.4{\tiny$\pm$1.1} & 68.2{\tiny$\pm$1.0} & 67.5{\tiny$\pm$1.3} & 66.1{\tiny$\pm$1.2} & 70.2{\tiny$\pm$1.2} & 71.7{\tiny$\pm$1.0} & 72.8{\tiny$\pm$.9} & 70.6{\tiny$\pm$1.1} & 72.1{\tiny$\pm$1.0} & 72.2{\tiny$\pm$1.0} & 71.5{\tiny$\pm$1.2} \\
        18 & COCO-SP & d-hop & (1,.6,32,.05) & 65.2{\tiny$\pm$1.6} & 62.0{\tiny$\pm$1.5} & 62.1{\tiny$\pm$1.8} & 61.5{\tiny$\pm$1.7} & 66.1{\tiny$\pm$1.5} & 67.5{\tiny$\pm$1.3} & 68.3{\tiny$\pm$1.2} & 66.7{\tiny$\pm$1.4} & 68.1{\tiny$\pm$1.3} & 68.0{\tiny$\pm$1.4} & 67.9{\tiny$\pm$1.5} \\
        \bottomrule
    \end{tabular}}
\end{table*}

\paragraph{Parameter impact analysis.}

\textit{Impact of patch structure ($d$ and $p$).}
The structure of the local view, governed by radius $d$ and coverage $p$, is the single most decisive factor for model performance (Figure~\ref{fig:param_impact}). Expanding the neighborhood radius from $d=1$ to $d=2$ produces a dramatic performance leap: Graphormer jumps from 93.2\% to 99.9\%, GIN+VN from 95.5\% to 99.8\%. Increasing coverage ($p=0.6$ to $p=1.0$) also consistently improves scores but more modestly. However, when patches are too large, they merge into a single component, inadvertently destroying Task~3 by eliminating inter-fragment boundaries. This exposes a fundamental tension: enriching local views to maximize local performance can undermine global reasoning.

\begin{figure}[h!]
  \centering
  \includegraphics[width=\columnwidth]{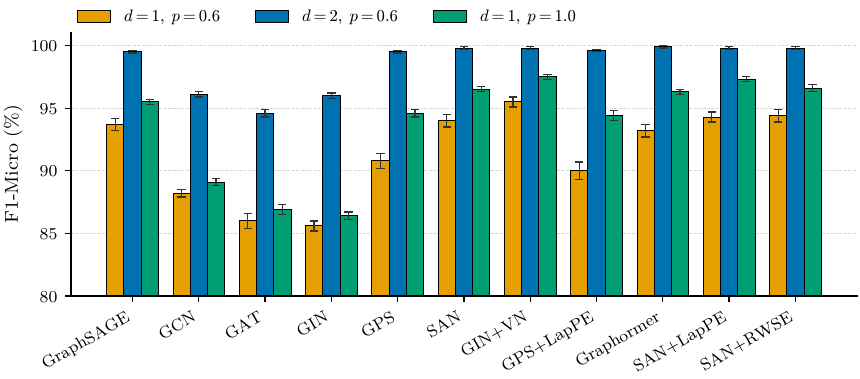}
  \caption{Impact of patch radius ($d$) and coverage ($p$) on node classification across GNN models on Cora. Increasing either parameter provides more local context and consistently improves F1-Micro scores.}
  \label{fig:param_impact}
\end{figure}

\textit{Impact of spectral quality ($k$ and $\sigma$).}
For classification tasks, GNNs can compensate for compromised structural data by leveraging node features: doubling $\sigma$ from 0.05 to 0.1 changes SAN+RWSE's score from 94.4\% to 94.7\%, and expanding $k$ from 16 to 64 offers minimal benefit. In contrast, reconstruction and link prediction are more sensitive to spectral quality. Increasing noise reduces AFR's performance from 74.3\% to 69.1\%, while GAE drops from 69.2\% to 58.3\%. This suggests that while node features anchor local tasks, global topology recovery depends more directly on structural information fidelity.

\textit{Decomposition strategy comparison.}
For node classification, d-hop is clearly the strongest strategy in our evaluation, providing the richest local context. On CiteSeer, SAN's performance drops from 89.3\% (d-hop) to 62.4\% (cluster). Interestingly, this hierarchy inverts for link prediction: the cluster strategy yields the highest performance, with SEAL achieving its peak score of 89.6\%. The best fragmentation strategy is therefore task-dependent (Figure~\ref{fig:strategy_comp}).

\begin{figure}[h!]
  \centering
  \includegraphics[width=\columnwidth]{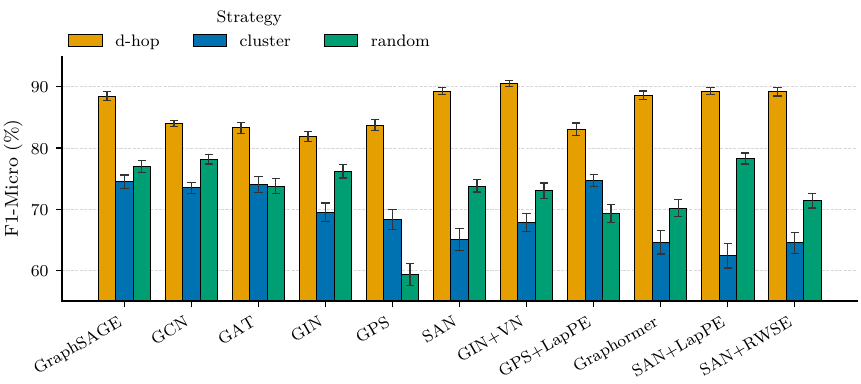}
  \caption{Performance comparison across d-hop, cluster, and random fragmentation strategies on CiteSeer. The d-hop strategy yields the highest node classification performance, but the hierarchy inverts for link prediction.}
  \label{fig:strategy_comp}
\end{figure}

\FloatBarrier
\subsection{Can Models Reason Across Fragment Boundaries?}
\label{sec:exp:linkpred}

Finally, we assess whether models can reason about connections across fragments that they never observe directly. Table~\ref{tab:link_prediction} presents results for inter-fragment link prediction. A dash (--) indicates that Task~3 is inapplicable, which occurs when large patches or high coverage cause fragments to merge into a single component. For the corresponding graph reconstruction results, see Table~\ref{tab:advanced_comparison} and Table~\ref{tab:core_benchmark} in Section~\ref{sec:exp:recon}.

\begin{table*}[h!]
    \centering
    \small
    \caption{Results for inter-fragment link prediction (Task 3, AUROC). Row IDs match the scenario indices in Tables~\ref{tab:core_benchmark} and~\ref{tab:node_classification}. A dash (--) indicates inapplicability: rows 2--4 (Cora at $d=2$ or $p=1.0$) and row~15 (BlogCatalog) merge into a single connected component, eliminating the inter-fragment boundaries required for Task~3. Within row~10 (CiteSeer cluster), Cosine, Neo-GNN and LightGCN representations collapse onto the same coordinates and produce ill-defined scores; only SEAL remains applicable. SEAL operates on the co-occurrence graph of patch membership rather than the spectral embedding, so its AUROC is (by construction) invariant to the spectral fidelity $k$ and noise $\sigma$ at fixed $(d,p)$; this explains the identical SEAL values for rows 5--8 on Cora. For the corresponding Task~1 reconstruction results, see Tables~\ref{tab:core_benchmark} and~\ref{tab:advanced_comparison}.}
    \label{tab:link_prediction}
    \resizebox{\textwidth}{!}{
    \begin{tabular}{cllccccc}
        \toprule
        ID & Dataset & Strategy & $(d, p, k, \sigma)$ & Cosine & SEAL & Neo-GNN & LightGCN \\
        \midrule
        1 & Cora & d-hop & (1,0.6,32,0.05) & 90.3 & 94.9 & \textbf{95.2} & 93.5 \\
        2 & Cora & d-hop & (2,0.6,32,0.05) & -- & -- & -- & -- \\
        3 & Cora & d-hop & (1,1.0,32,0.05) & -- & -- & -- & -- \\
        4 & Cora & d-hop & (2,0.8,32,0.05) & -- & -- & -- & -- \\
        5 & Cora & d-hop & (1,0.6,16,0.05) & \textbf{93.9} & 83.1 & 81.5 & 82.7 \\
        6 & Cora & d-hop & (1,0.6,64,0.05) & 93.9 & 83.1 & 82.9 & 82.5 \\
        7 & Cora & d-hop & (1,0.6,32,0.0) & \textbf{97.2} & 83.1 & 82.2 & 83.6 \\
        8 & Cora & d-hop & (1,0.6,32,0.1) & \textbf{93.4} & 83.1 & 82.1 & 78.9 \\
        9 & CiteSeer & d-hop & (1,0.6,32,0.05) & 81.0 & \textbf{85.2} & 84.6 & 83.8 \\
        10 & CiteSeer & cluster & (1,0.6,32,0.05) & -- & \textbf{89.6} & -- & -- \\
        11 & CiteSeer & random & (1,0.6,32,0.05) & 84.4 & \textbf{91.0} & 88.9 & 89.3 \\
        12 & PubMed & cluster & (1,0.6,32,0.05) & 79.4 & \textbf{86.0} & 85.8 & 85.1 \\
        13 & PubMed & d-hop & (1,0.6,32,0.05) & 82.1 & \textbf{86.5} & 84.1 & 83.7 \\
        14 & ogbn-arXiv & d-hop & (1,0.6,32,0.05) & 68.5 & \textbf{72.1} & 71.5 & 70.2 \\
        15 & BlogCatalog & d-hop & (1,0.6,32,0.05) & -- & -- & -- & -- \\
        16 & PROTEINS & d-hop & (1,0.6,32,0.05) & 75.0 & 82.1 & \textbf{82.5} & 80.3 \\
        17 & PascalVOC-SP & d-hop & (1,0.6,32,0.05) & 65.1 & 73.4 & \textbf{74.0} & 71.2 \\
        18 & COCO-SP & d-hop & (1,0.6,32,0.05) & 62.5 & 70.1 & \textbf{70.8} & 68.3 \\
        19 & PCQM-Contact & d-hop & (1,0.6,32,0.05) & 69.8 & 76.5 & \textbf{77.1} & 74.9 \\
        \bottomrule
    \end{tabular}}
\end{table*}

\FloatBarrier
\subsection{Cross-Task Insights and the Privacy--Utility Tension}
\label{sec:exp:discussion}

Our three-task evaluation reveals cross-cutting insights about graph learning under fragmentation.

\paragraph{The local enrichment vs.\ global reasoning tension.}
A striking finding is that the ``best'' decomposition strategy depends entirely on the downstream task. D-hop patches, which provide the richest local context, are superior for node classification but inferior for link prediction. Cluster-based fragmentation, the worst for classification, produces the clearest inter-fragment boundaries for predicting cross-community links. This reveals an inherent tension in system design: enriching local views to maximize local performance can undermine the ability to reason about global structure.

\paragraph{When AFR fails, and why.}
AFR does not universally dominate. On CiteSeer, the VGAE's probabilistic framework better handles the high disconnectedness (low clustering coefficient). In PubMed, GCN-LE exploits the unusually strong correlation between node features and edge existence. These edge cases define AFR's boundary: it may be outperformed in domains where reconstruction is less a geometric puzzle and more a reflection of probabilistic uncertainty or semantic similarity. Importantly, these failures are narrow and dataset-specific; overall, AFR is the most consistently strong performer across the conditions we evaluate.

\paragraph{No single best architecture under fragmentation.}
In the dense BlogCatalog social graph, Transformer-based models (SAN+RWSE at 98.8\%) substantially outperform canonical GNNs (GAT at 62.4\%). But in the fragmented ogbn-arXiv, this advantage vanishes: simpler models like GraphSAGE (61.7\%) prove more robust than sophisticated Transformers. The lesson is that in information-scarce environments, simple localized aggregation can outperform complex attention mechanisms.

\paragraph{The privacy--utility tension under per-embedding DP.}
Our defense study quantifies a tension specific to the per-embedding $(\epsilon, \delta)$-Gaussian mechanism: practical choices ($\epsilon\ge 5$) provide limited reduction in AFR's reconstruction F1 while incurring only modest utility costs. Substantial attack degradation ($\epsilon\le 2$) requires substantial utility sacrifice. Empirically, AFR's fidelity-aware design appears relatively robust to DP noise because it downweights corrupted patches during assembly. We reiterate the caveat from Section~\ref{sec:exp:recon}: per-embedding DP is not the same as node- or edge-level DP on the underlying graph; the trade-off we observe is specific to this heuristic defense and does not indicate that DP more broadly is inadequate. Future defenses may achieve better trade-offs by fundamentally transforming the geometric properties of the shared embeddings (e.g., edge-level randomized response, or subspace quantization that aligns privacy noise with Procrustes invariances).

\paragraph{Locality as the dominant factor.}
Across our experiments, the structure of the local view, governed by radius $d$ and coverage $p$, is the most influential factor we observe. Expanding from 1-hop to 2-hop neighborhoods catapults even simple architectures to near-perfect performance, while changes to spectral fidelity $k$ or noise $\sigma$ produce comparatively modest effects on node classification. This suggests that the fundamental bottleneck in fragmented graph learning is not information quality but information reach.

\FloatBarrier
\section{Conclusion}
\label{sec:conclusion}

This paper presented a unified framework for graph learning from localized spectral views, organized around three complementary contributions. First, LoGraB provides a benchmark that systematically explores the joint space of fragmentation, spectral truncation, and noise in graph learning. Its controlled fragmentation strategies, three tasks, and tailored metrics highlight how modern GNN architectures respond to data fragmentation, exposing architectural strengths and weaknesses that are less visible in standard benchmarks. Second, AFR is a constructive reconstruction algorithm that explicitly models patch-level fidelity heterogeneity and adaptively adjusts its stitching criteria; it consistently recovers meaningful topology in noisy regimes where uniform-quality assumptions are less effective. AFR enjoys rigorous local guarantees (spectral perturbation, heat-kernel edge recovery, stitch soundness, error accumulation). Our ablation study shows that both fidelity scoring and robust alignment contribute meaningfully to performance. Third, we established the Spectral Leakage Proposition, an information-theoretic feasibility statement about a Bayesian recovery procedure; we were careful to present it as a feasibility result (with a heuristic supporting argument whose three open gaps we identified explicitly) distinct from AFR, and to note that AFR's guarantees do not rely on it.

Several limitations remain. The Spectral Leakage Proposition's heuristic argument should be promoted to a full proof; AFR's cross-dataset variance and the effect of patch-graph diameter on Bundle Adjustment warrant more thorough empirical study; the per-embedding DP defense we evaluate is not a node- or edge-level DP guarantee, and future work should quantify graph-level privacy leakage under composition. For future work, we plan to extend AFR beyond spectral embeddings to other representation modalities, investigate its behavior in dynamic federated settings, explore modern robust-synchronization baselines (Local2Global, cycle-consistent synchronization), and study how AFR's insights can inform principled graph-level DP defenses.

\FloatBarrier
\bibliography{references}

\newpage
\appendix

\section{Formal Assumptions and Baseline Algorithms}
\label{appendix:formal}

\subsection{Formal Assumptions}
\label{appendix:assumptions}

This section provides a complete mathematical formulation of the assumptions that guide the AFR algorithm.

\begin{assumption}[Core patch fidelity]
\label{assumption:1}
A reconstructed patch $\widehat{\mathcal{A}}_v$ is a core patch if it satisfies three conditions:
\begin{itemize}[nosep]
    \item Size bound: $|P_v| \le q_{\max}$.
    \item Error bound: $\|\widehat{\mathcal{A}}_v - \mathcal{A}_{P_v}\|_2 \le \eta_v$.
    \item Fidelity score: $s_v \ge s_{\min}$, where $s_v = \alpha\,\rho_v + (1-\alpha)\,\mathcal{E}_v$ with $\rho_v = \delta_v/(\delta_v + \eta_v)$ and $\mathcal{E}_v$ the normalized degree entropy.
\end{itemize}
\end{assumption}

\begin{assumption}[Adaptive stitching eligibility]
\label{assumption:2}
A pair of core patches $(v,w)$ is eligible for stitching only if their intersection $I_{vw} = P_v \cap P_w$ satisfies:
\begin{itemize}[nosep]
    \item $|I_{vw}| \ge \max\{k+1,\, k_{\text{base}} + \gamma(1 - \min\{s_v, s_w\})\}$ (the floor $k+1$ ensures Procrustes identifiability in $\mathbb{R}^k$).
    \item The induced subgraph $G[I_{vw}]$ is well-posed, e.g., $\text{diam}(G[I_{vw}]) \le 2$.
\end{itemize}
\end{assumption}

\begin{assumption}[Probabilistic alignment correctness]
\label{assumption:3}
For any eligible pair $(v,w)$, let $p_{vw}\in(0,1]$ be the fraction of true inlier correspondences. With $M_{vw} \ge \log\beta / \log(1 - p_{vw}^m)$ iterations, RANSAC-Procrustes returns a rotation $\mathcal{Q}_{vw}\in O(k)$ such that
$\|\sin\Theta(\mathcal{Q}_{vw}, \mathcal{Q}_{vw}^*)\|_F \le \tau$
with probability at least $1-\beta$.
\end{assumption}

\begin{assumption}[AFR patch overlap structure]
\label{assumption:afr_overlap}
The collection of local patches $\{P_v\}_{v \in V_{\mathrm{obs}}}$ satisfies the following overlap conditions required for the local-to-global reduction:
\begin{itemize}[nosep]
\item \textit{Covering:} Let $V_{\mathrm{cov}} := \bigcup_{v \in V_{\mathrm{obs}}} P_v \subseteq V$. We guarantee recovery only on $V_{\mathrm{cov}}$; nodes in $V \setminus V_{\mathrm{cov}}$ (uncovered due to coverage ratio $p<1$) are excluded from the reconstruction target. When $p=1$ and $d$-hop neighborhoods span $V$, we recover $V_{\mathrm{cov}}=V$.
\item \textit{Pairwise overlap for edges:} For every edge $(v,w) \in E$, there exists a patch pair whose intersection $I_{vw}$ has $|I_{vw}| \ge k+1$ nodes and contains both endpoints; $G[I_{vw}]$ is connected.
\item \textit{Patch graph diameter:} Let $D := \mathrm{diam}(\mathcal{G}_P)$, where $\mathcal{G}_P$ has a node per patch and an edge whenever two patches satisfy the pairwise overlap condition. The error accumulation bound of Proposition~\ref{prop:local_reduction} scales as $O(D \cdot \varepsilon)$. When $D = O(\mathrm{polylog}\, n)$, the bound is polylogarithmic; when $D$ is closer to $\mathrm{diam}(G)$ itself (which can occur for $d=1$ decompositions of large-diameter graphs), the bound is proportionally worse. The AFR Bundle Adjustment step (Stage~3) is motivated in part by this linear-in-depth accumulation.
\end{itemize}
These conditions are satisfied by $d$-hop decompositions when $d \ge 2$ and the original graph has bounded expansion, and by cluster decompositions with overlap expansion (Section~\ref{sec:lograb:decomp}); empirical patch-graph diameters for our datasets are reported alongside the experiments.
\end{assumption}

\subsection{Eigen-sync Baseline}
\label{appendix:eigensync}

Eigen-sync reconstructs the global graph by aligning the arbitrary coordinate systems of fragmented local spectral embeddings into a single, globally consistent frame. For any two patches $P_i$ and $P_j$ with sufficient overlap, the optimal relative rotation $R_{ij}$ is found via the Orthogonal Procrustes problem: given the SVD $B^\top A = U\Sigma V^\top$, the solution is $R_{ij} = VU^\top$. These pairwise rotations are used to construct a block matrix $M$, whose top $k$ eigenvectors yield globally consistent absolute rotations $\{R_i\}$. Each local embedding is then transformed via $\mathcal{P}_i^{\text{sync}} = \mathcal{P}_i R_i$, and the global embedding for each node is obtained by averaging synchronized representations. Edges are predicted by constructing a $k$-NN graph on the global embeddings using cosine similarity.

\begin{algorithm2e}[h!]
\small
\DontPrintSemicolon
\KwInput{A set of $m$ local patches $\{(P_i, \mathcal{P}_i)\}_{i=1}^m$, where $P_i$ is the set of global node indices and $\mathcal{P}_i\in\mathbb{R}^{|P_i|\times k}$.}
\KwOutput{A set of reconstructed edges $\widehat{E}$.}
Initialize: $R_{\text{pairwise}} \leftarrow \{\}$, $W_{\text{pairwise}} \leftarrow \{\}$\\
\tcc{Stage 1: Pairwise alignment}
\For{$i \leftarrow 1$ \KwTo $m$}{
    \For{$j \leftarrow i+1$ \KwTo $m$}{
        $S \leftarrow P_i \cap P_j$\\
        \If{$|S| \ge 2$}{
            $A \leftarrow$ sub-matrix of $\mathcal{P}_i$ for nodes in $S$; $B \leftarrow$ sub-matrix of $\mathcal{P}_j$ for $S$\\
            $U, \Sigma, V^\top \leftarrow \text{SVD}(B^\top A)$\\
            $R_{ij} \leftarrow V U^\top$; $R_{\text{pairwise}}[(i,j)] \leftarrow R_{ij}$; $W_{\text{pairwise}}[(i,j)] \leftarrow S$
        }
    }
}
\tcc{Stage 2: Global synchronization}
$M \leftarrow$ Construct $(mk)\times(mk)$ block matrix from $R_{\text{pairwise}}$ and $W_{\text{pairwise}}$\\
$\lambda, V_{\text{eig}} \leftarrow$ Top $k$ eigenvectors of $M$\\
\For{$i \leftarrow 1$ \KwTo $m$}{
    $V_i \leftarrow$ Extract $k\times k$ block for patch $i$ from $V_{\text{eig}}$\\
    $Q, R \leftarrow \text{QR}(V_i)$; $R_{\text{global}}[i] \leftarrow Q$
}
\tcc{Stage 3: Global embedding and kNN reconstruction}
$Z \leftarrow$ zero matrix; $C \leftarrow$ zero vector\\
\For{$i \leftarrow 1$ \KwTo $m$}{
    $\mathcal{P}_i^{\text{sync}} \leftarrow \mathcal{P}_i \, R_{\text{global}}[i]$\\
    Update rows in $Z$ and $C$ for nodes in $P_i$
}
$Z \leftarrow Z \mathbin{./} C$; $\widehat{E} \leftarrow \text{kNN\_graph}(Z, k_{\text{nn}}, \text{cosine})$\\
\Return $\widehat{E}$
\caption{Eigen-sync reconstruction baseline.}
\label{alg:eigensync}
\end{algorithm2e}

\subsection{Cosine Similarity Baseline for Link Prediction}
\label{appendix:cosine}

This approach generates global embeddings by averaging local GNN representations for each node across all patches in which it appears, then scores candidate links using cosine similarity: $\psi(u,v) = z_u^\top z_v / (\|z_u\|\,\|z_v\|)$.

\subsection{SEAL Adaptation for Link Prediction}
\label{appendix:seal}

We adapt SEAL \citep{Zhang2018SEAL} for the fragmented setting. A co-occurrence graph $H$ is constructed from the patches: an edge exists if two nodes appear together in at least one patch. For each candidate link $(u,v)$, we extract its enclosing subgraph from $H$, assign structural labels using Double-Radius Node Labeling (DRNL), and train a GNN to classify whether the subgraph contains a true link.

\section{Proofs}
\label{appendix:proofs}

\subsection{Proof of Theorem~\ref{thm:spectral_perturbation}}
\label{proof:spectral}

\begin{proof}
Let $\mathcal{X}=\operatorname{span}(\mathcal{P}_v)$ and $\widetilde{\mathcal{X}}=\operatorname{span}(\widetilde{\mathcal{P}}_v)$ be the $k$-dimensional invariant subspaces of $\mathcal{L}_v$ and $\widetilde{\mathcal{L}}_v$, respectively. By Weyl's inequality, the eigenvalues of $\widetilde{\mathcal{L}}_v$ are bounded by the perturbation norm.

The Davis--Kahan theorem \citep{Davis1970} states that the sine of the largest principal angle $\Theta(\mathcal{X},\widetilde{\mathcal{X}})$ between the two subspaces satisfies
\[
\|\sin\Theta(\mathcal{X},\widetilde{\mathcal{X}})\|_F \le \frac{\|R\|_F}{\delta},
\]
where $R = \widetilde{\mathcal{L}}_v \mathcal{P}_v - \mathcal{P}_v(\mathcal{P}_v^\top \mathcal{L}_v \mathcal{P}_v) = E\mathcal{P}_v$ is the residual, and $\delta = \delta_v = \lambda_{k+1}(\mathcal{L}_v) - \lambda_k(\mathcal{L}_v)$. Thus:
\[
\|\sin\Theta(\mathcal{X},\widetilde{\mathcal{X}})\|_F \le \frac{\|E\mathcal{P}_v\|_F}{\delta_v} \le \frac{\|E\|_2\,\|\mathcal{P}_v\|_F}{\delta_v} = \frac{\sqrt{k}\,\varepsilon}{\delta_v}.
\]

The result follows from the standard bound relating Frobenius distance between basis matrices (up to optimal orthogonal transformation) to the sine of principal angles:
\[
\|\widetilde{\mathcal{P}}_v - \mathcal{P}_v \mathcal{Q}\|_F \le \sqrt{2}\,\|\sin\Theta(\mathcal{X},\widetilde{\mathcal{X}})\|_F \le \frac{\sqrt{2k}\,\varepsilon}{\delta_v}.
\]
\end{proof}

\subsection{Proof of Theorem~\ref{thm:edge_recovery}}
\label{proof:edge}

\begin{proof}
We show that, under the theorem's condition, the observed values for edges are disjoint from and strictly greater than those for non-edges.

The residual matrix is $R_v^{(k)}(t) = \sum_{r=k+1}^{n_v} e^{-t\lambda_r} u_r u_r^\top$. Its spectral norm is $\|R_v^{(k)}(t)\|_2 = e^{-t\lambda_{k+1}} = \eta_v$. By the Cauchy--Schwarz inequality and orthonormality, $|R_v^{(k)}(t)_{ij}| = |\sum_{r>k} e^{-t\lambda_r} u_r(i) u_r(j)| \le e^{-t\lambda_{k+1}}\bigl(\sum_{r>k} u_r(i)^2\bigr)^{1/2}\bigl(\sum_{r>k} u_r(j)^2\bigr)^{1/2} \le \eta_v$ (using the completeness relation $\sum_{r\ge 1} u_r(i)^2 = 1$). Each entry of the truncated kernel is thus within an $\eta_v$-ball of its true value:
\[
|H_v(t)_{ij} - H_v^{(k)}(t)_{ij}| \le \eta_v.
\]

For any edge $(i,j)\in E_v$:
\[
H_v^{(k)}(t)_{ij} \ge H_v(t)_{ij} - \eta_v \ge \left(\min_{(a,b)\in E_v} H_v(t)_{ab}\right) - \eta_v.
\]

For any non-edge $(i,j)\notin E_v$, $i\ne j$:
\[
H_v^{(k)}(t)_{ij} \le H_v(t)_{ij} + \eta_v \le \left(\max_{\substack{(a,b)\notin E_v \\ a\ne b}} H_v(t)_{ab}\right) + \eta_v.
\]

The condition $\gamma_t > 2\eta_v$ ensures a separating gap exists. Any threshold $\tau_v$ in the interval
\[
\tau_v \in \left(\max_{\substack{(a,b)\notin E_v \\ a\ne b}} H_v(t)_{ab} + \eta_v, \;\; \min_{(a,b)\in E_v} H_v(t)_{ab} - \eta_v\right)
\]
correctly classifies all edges and non-edges, yielding exact recovery of $E_v$.
\end{proof}

\begin{corollary}[Robustness to eigenvalue perturbations]
\label{cor:weyl}
Consider perturbed spectral components with entry-wise heat kernel error bounded by $\Delta_H$. If $\gamma_t > 2(\Delta_H + \eta_v)$, then $E_v$ can still be recovered exactly from the perturbed kernel.
\end{corollary}

\begin{proof}
The total entry-wise error satisfies $|\epsilon_{ij}^{\text{total}}| \le \Delta_H + \eta_v$ by the triangle inequality. The proof then follows the same structure as Theorem~\ref{thm:edge_recovery}, replacing $\eta_v$ with $\Delta_H + \eta_v$.
\end{proof}

\subsection{Proof of Proposition~\ref{prop:fidelity}}
\label{proof:fidelity}

\begin{proof}
\noindent\textit{Part 1: Bounds and monotonicity of the gap-to-truncation ratio $\rho_v$.}
Since $\delta_v \ge 0$ and $\eta_v > 0$, we have $\rho_v = \delta_v/(\delta_v + \eta_v) \ge 0$ and $\rho_v \le 1$ since $\delta_v \le \delta_v + \eta_v$. The partial derivatives are:
\[
\frac{\partial\,\rho_v}{\partial\delta_v} = \frac{\eta_v}{(\delta_v + \eta_v)^2} > 0, \qquad
\frac{\partial\,\rho_v}{\partial\eta_v} = \frac{-\delta_v}{(\delta_v + \eta_v)^2} < 0 \quad (\delta_v > 0).
\]

\noindent\textit{Part 2: Bounds of the composite score.}
Since $0 \le \rho_v \le 1$ and $0 \le \mathcal{E}_v \le 1$, the convex combination $s_v = \alpha\,\rho_v + (1-\alpha)\,\mathcal{E}_v$ with $\alpha\in[0,1]$ satisfies $0 \le s_v \le 1$.

\noindent\textit{Part 3: Monotonicity of the composite score.}
Since $\eta_v = e^{-t(\lambda_k + \delta_v)}$, we have $d\eta_v/d\delta_v = -t\eta_v$. Thus:
\[
\frac{d\,\rho_v}{d\delta_v} = \frac{\eta_v}{(\delta_v + \eta_v)^2} + \frac{-\delta_v}{(\delta_v + \eta_v)^2}(-t\eta_v) = \frac{\eta_v(1 + t\delta_v)}{(\delta_v + \eta_v)^2} \ge 0.
\]
Since $\alpha \ge 0$, it follows that $ds_v/d\delta_v \ge 0$.
\end{proof}

\subsection{Proof of Lemma~\ref{lem:ransac}}
\label{proof:ransac}

\begin{proof}
The probability that a single sample of $m$ correspondences consists entirely of inliers is $p_{vw}^m$. The probability of failure (at least one outlier) in a single sample is $1 - p_{vw}^m$. Since the $M$ iterations are independent:
\[
\mathbb{P}(\text{all samples have outliers}) = (1 - p_{vw}^m)^M.
\]
The probability of at least one all-inlier sample is $1 - (1 - p_{vw}^m)^M$. Setting this $\ge 1 - \beta$ and solving: $\log(1 - p_{vw}^m)^M \le \log\beta$. Since $\log(1 - p_{vw}^m) < 0$, dividing reverses the inequality, giving $M \ge \log(\beta)/\log(1 - p_{vw}^m)$.
\end{proof}

\subsection{Proof of Theorem~\ref{thm:stitch}}
\label{proof:stitch}

\begin{proof}
By Lemma~\ref{lem:ransac}, executing RANSAC for $M \ge \log(\beta)/\log(1 - p_{vw}^m)$ iterations guarantees that, with probability at least $1-\beta$, at least one sample consists entirely of inliers. The Procrustes solution from this sample provides a high-quality initial estimate of $\mathcal{Q}_{vw}^*$. RANSAC then expands this to the largest consensus set $C_{vw}$.

Let the submatrices of embeddings on the consensus set be $A, B\in\mathbb{R}^{|C_{vw}|\times k}$. Under the noise model, $A = A_0 + E_A$ and $B = B_0 + E_B$ with i.i.d.\ sub-Gaussian entries. The Procrustes solution $\widehat{\mathcal{Q}}_{vw}$ is found via SVD of $A^\top B = U\Sigma V^\top$, yielding $\widehat{\mathcal{Q}}_{vw} = VU^\top$.

By standard matrix-perturbation results \citep{Davis1970}:
\[
\|\sin\Theta(\widehat{\mathcal{Q}}_{vw}, \mathcal{Q}_{vw}^*)\|_F \le O\left(\frac{\|A^\top E_B + E_A^\top B\|_F}{\sigma_{\min}(A_0^\top B_0)}\right).
\]

Under i.i.d.\ noise, $\mathbb{E}[\|A^\top E_B + E_A^\top B\|_F] = O(\sigma\sqrt{|C_{vw}|})$, while $\sigma_{\min}(A_0^\top B_0) = O(|C_{vw}|)$. Thus the scaling is $O(\sigma/\sqrt{|C_{vw}|})$. Since $|C_{vw}| \ge d_{\text{adapt}}(s_v,s_w)$:
\[
\|\sin\Theta(\widehat{\mathcal{Q}}_{vw}, \mathcal{Q}_{vw}^*)\|_F \le \frac{C\,\sigma}{\sqrt{d_{\text{adapt}}(s_v,s_w)}}.
\]
\end{proof}

\begin{corollary}[Rejection of spurious stitches]
\label{cor:rejection}
If $p_{vw} \cdot |I_{vw}| < d_{\text{adapt}}(s_v, s_w)$, then the RANSAC-Procrustes procedure is guaranteed to reject the stitch, regardless of the number of iterations.
\end{corollary}

\begin{proof}
Since $|C_{vw}| \le |I_{vw}^*| = p_{vw}\,|I_{vw}| < d_{\text{adapt}}(s_v,s_w)$, the acceptance criterion $|C_{vw}| \ge d_{\text{adapt}}(s_v,s_w)$ can never be satisfied.
\end{proof}

\subsection{Heuristic Argument for Proposition~\ref{thm:spectral_leakage} (Spectral Embedding Leakage)}
\label{proof:spectral_leakage}

We present a heuristic argument supporting Proposition~\ref{thm:spectral_leakage} and explicitly identify the technical gaps that would be required to promote it to a fully rigorous theorem.

\paragraph{Setup.}
Collect all embedding vectors $\mathcal{P}_v$ for $v=1,\ldots,n$ into a matrix $\mathcal{P}\in\mathbb{R}^{n\times k}$, so $\mathcal{P}=\mathrm{Eig}_k(\mathcal{L})\in\mathbb{R}^{n\times k}$. Let $\Theta=\{\mathcal{A}\in\{0,1\}^{n\times n} \mid \mathcal{A}=\mathcal{A}^\top, \mathcal{A}_{ii}=0\}$ be the space of adjacency matrices. Define a prior $p(\mathcal{A})$ (e.g., Erd\H{o}s--R\'enyi with edge probability $p_0$) and, for each candidate $\mathcal{A}\in\Theta$, let $\mathcal{U}(\mathcal{A})=\mathrm{Eig}_k(\mathcal{L}(\mathcal{A}))$. Define empirical measures $\mu_\mathcal{P}=\frac{1}{n}\sum_{i=1}^n \delta_{p_i}$ and $\mu_{\mathcal{U}(\mathcal{A})}=\frac{1}{n}\sum_{i=1}^n \delta_{u_i(\mathcal{A})}$, and the likelihood $p(\mathcal{P}|\mathcal{A})\propto \exp[-\gamma S_\varepsilon(\mu_\mathcal{P}, \mu_{\mathcal{U}(\mathcal{A})})]$, where $S_\varepsilon$ is the entropic Sinkhorn distance.

The posterior $p(\mathcal{A}|\mathcal{P})\propto p(\mathcal{P}|\mathcal{A})p(\mathcal{A})$ is approximated by a mean-field variational distribution $q(\mathcal{A};\phi)=\prod_{i<j}\phi_{ij}^{\mathcal{A}_{ij}}(1-\phi_{ij})^{1-\mathcal{A}_{ij}}$, with parameters $\phi^*$ obtained by maximizing the ELBO via stochastic gradient ascent with $K$ Monte Carlo samples per step. The heuristic argument has three components, each of which we describe together with the technical gap that would be required to make it rigorous.

\paragraph{Heuristic (i): Posterior concentration.}
If $\mathcal{A}$ is close to $\mathcal{A}_0$, then by Theorem~\ref{thm:spectral_perturbation}, $\mathrm{Eig}_k(\mathcal{L}(\mathcal{A}))$ is close to $\mathrm{Eig}_k(\mathcal{L}(\mathcal{A}_0))$. By the Lipschitz continuity of $S_\varepsilon$ on point-cloud distributions, the likelihood $p(\mathcal{P}|\mathcal{A})$ is large near $\mathcal{A}_0$ and small away from it. We argue heuristically that $p(\mathcal{A}|\mathcal{P})$ concentrates around $\mathcal{A}_0$.

\emph{Gap.} A rigorous concentration statement would quantify the rate at which $p(\mathcal{A}|\mathcal{P})$ concentrates as $n \to \infty$, accounting for (a) the combinatorial size of the support $|\Theta| = 2^{\binom{n}{2}}$, (b) the non-local nature of the eigenvector map $\mathcal{A} \mapsto \mathrm{Eig}_k(\mathcal{L}(\mathcal{A}))$, which can have discontinuities at Laplacian spectral coincidences, and (c) the Lipschitz constant of $S_\varepsilon$ on the relevant empirical measures. We do not supply these quantitative bounds.

\paragraph{Heuristic (ii): Variational approximation quality.}
Since maximizing the ELBO is equivalent to minimizing $\mathrm{KL}(q\,\|\,p(\cdot|\mathcal{P}))$, if $p(\mathcal{A}|\mathcal{P})$ is sharply peaked near $\mathcal{A}_0$, we expect $q^*(\mathcal{A})\approx p(\mathcal{A}|\mathcal{P})$ and the variational parameters $\phi_{ij}^*$ to be close to $1$ for true edges and close to $0$ otherwise. Thresholding at $0.5$ would then recover $\widehat{\mathcal{A}}\approx\mathcal{A}_0$.

\emph{Gap.} Mean-field VI is known to systematically underestimate posterior variance and, more seriously, can fail to concentrate on the true mode in discrete combinatorial problems even when the true posterior is sharply peaked \citep[see][for overviews]{Von2007tutorial}. Establishing that mean-field VI consistently recovers $\mathcal{A}_0$ here requires a specific consistency theorem for this particular prior/likelihood combination, which we do not provide.

\paragraph{Heuristic (iii): Polynomial-time complexity.}
Each gradient iteration involves: sampling $K$ adjacency matrices ($O(Kn^2)$), computing $k$ eigenvectors per sample using Lanczos ($O(k\cdot\mathrm{nnz}(\mathcal{L})\cdot I_{\mathrm{eig}})=O(\mathrm{poly}(n))$), and computing the Sinkhorn distance ($O(n^2 I_{\mathrm{sink}})$ with $I_{\mathrm{sink}}=O(\mathrm{poly}\log n)$). The total cost \emph{per iteration} is $O(\mathrm{poly}(n,k))$.

\emph{Gap.} Total polynomial runtime additionally requires a bound on the number of gradient steps to reach an $\varepsilon$-accurate stationary point. The ELBO is non-convex in $\phi$, so such a bound requires either a problem-specific convergence analysis (likely using that the ELBO is log-concave in a neighborhood of the true $\phi^*$ under Assumption~\ref{assumption:spectral_leakage}) or an additional smoothness/sample-complexity argument. We do not supply this.

\paragraph{Summary.}
The three heuristic components above, taken together, strongly suggest that polynomial-time recovery is feasible under Assumption~\ref{assumption:spectral_leakage} once $k$ exceeds the threshold of Corollary~\ref{cor:threshold}. For a rigorous version of Proposition~\ref{thm:spectral_leakage}, the three gaps above (posterior concentration rate, mean-field VI consistency for the discrete graph prior, convergence analysis of stochastic ELBO optimization) must be filled. We identify these as open problems. Our empirical AFR algorithm does not rely on Proposition~\ref{thm:spectral_leakage}; it is separately justified by Theorems~\ref{thm:spectral_perturbation}--\ref{thm:stitch} and Lemmas~\ref{lem:heat_kernel_gap}--\ref{lem:error_accumulation}, which are fully proved.

\subsection{Proof of Corollary~\ref{cor:threshold} (Truncation-based Proxy Threshold)}
\label{proof:threshold}

\begin{proof}
We prove the stated characterizations of $\ell^*(\varepsilon) = \min\{k : \lambda_{k+1} \le \varepsilon\}$. For polynomial decay $\lambda_k=Ck^{-\alpha}$, solving $C(\ell^*+1)^{-\alpha}\le\varepsilon$ gives $\ell^*=\lceil(C/\varepsilon)^{1/\alpha}\rceil-1$. For exponential decay $\lambda_k=Ce^{-\alpha k}$, solving $Ce^{-\alpha(\ell^*+1)}\le\varepsilon$ gives $\ell^*=\lceil\alpha^{-1}\ln(C/\varepsilon)\rceil-1$. The proxy-lower-bound claim that $\ell^* \le k^*$ follows because the truncation residual $\eta_v = e^{-t\lambda_{k+1}}$ upper-bounds the heat-kernel entry-wise distortion, and no eigenvector-based procedure can distinguish adjacency matrices whose eigenvector differences fall entirely within this truncation bucket. The reverse inequality $k^* \le C \cdot \ell^*$ would require an adjacency-matrix perturbation bound tying $\|\hat A - A_0\|_F$ to $\lambda_{k+1}$, which we do not establish here.
\end{proof}

\subsection{Proof of Corollary~\ref{cor:dp_defense} (Noise as Defense)}
\label{proof:dp_defense}

\begin{proof}
When noise $\mathcal{N}$ is added, the observed embedding becomes $\widehat{\mathcal{P}}=\mathcal{P}+\mathcal{N}$. By the Lipschitz continuity of the Sinkhorn distance, $|S_\varepsilon(\mu_\mathcal{P}, \mu_{\mathrm{Eig}_k(\mathcal{L}(\mathcal{A}_0))})-S_\varepsilon(\mu_{\widehat{\mathcal{P}}}, \mu_{\mathrm{Eig}_k(\mathcal{L}(\mathcal{A}_0))})|\le L_{\mathrm{OT}}\|\mathcal{N}\|$. When $\|\mathcal{N}\|$ is large, $\mu_{\widehat{\mathcal{P}}}$ deviates substantially from $\mu_\mathcal{P}$, making the likelihood $p(\widehat{\mathcal{P}}|\mathcal{A})$ more diffuse around $\mathcal{A}_0$. Consequently, the posterior becomes less concentrated, the variational parameters $\phi_{ij}^*$ drift toward $0.5$, and thresholding produces more errors: $\mathbb{E}_{q^*(\mathcal{A}|\widehat{\mathcal{P}})}[d(\mathcal{A},\mathcal{A}_0)] \ge \mathbb{E}_{q^*(\mathcal{A}|\mathcal{P})}[d(\mathcal{A},\mathcal{A}_0)]$.
\end{proof}

\subsection{Proof of Lemma~\ref{lem:heat_kernel_gap} (Heat Kernel Separation Gap)}
\label{proof:heat_kernel_gap}

\begin{proof}
We analyze the heat-kernel matrix via its Taylor series $H_v(t)=e^{-t\mathcal{L}_v}=\sum_{j=0}^\infty \frac{(-t)^j}{j!}\mathcal{L}_v^j$ with the patch-adaptive time parameter $t_v = q_v^{-1}\log(C_3 q_v)$. Define the dimensionless parameter $\alpha=t_v \cdot d_{\max}^{(v)}$.

\textit{Non-edges} ($\mathrm{dist}(u,w)\ge 2$): Since $(I)_{uw}=0$ and $(\mathcal{L}_v)_{uw}=0$, the expansion starts at the quadratic term: $(H_v(t_v))_{uw}=\frac{t_v^2}{2}(\mathcal{L}_v^2)_{uw}+O(t_v^3)$. Since $(\mathcal{L}_v^2)_{uw}$ counts common neighbors, $|(\mathcal{L}_v^2)_{uw}|\le d_{\max}^{(v)}$, giving $(H_v(t_v))_{uw}\le \frac{\alpha t_v}{2}(1+O(\alpha))$.

\textit{Edges} ($(u,w)\in E_v$): Since $(\mathcal{L}_v)_{uw}=-1$, we get $(H_v(t_v))_{uw}\ge t_v(1-\alpha-O(\alpha^2))$.

Choosing $C_3$ so that $\alpha\le 1/4$ (possible since $\alpha = (d_{\max}^{(v)}/q_v)\log(C_3 q_v)$ with $d_{\max}^{(v)}/q_v \in (0, 1]$), the separation gap satisfies $\Delta_v \ge \frac{3}{4}t_v - \frac{1}{8}t_v = \frac{5}{8}t_v$. Substituting $t_v = q_v^{-1}\log(C_3 q_v)$ gives the uniform bound
\[
\Delta_v \ge \frac{5 \log(C_3 q_v)}{8 q_v} = \Omega(\log q_v / q_v).
\]
The alternative power-law bound $\Delta_v \ge 1/(2q_v^2)$ follows from the weaker choice $t_v = 1/q_v^2$, which satisfies $\alpha \le 1/4$ for $q_v \ge 2 d_{\max}^{(v)}$ and directly yields $\frac{5}{8}t_v = \frac{5}{8 q_v^2} \ge \frac{1}{2 q_v^2}$. (Earlier drafts stated the power-law bound with $t_v = q_{\max}^{-1} \log(C_3 q_{\max})$, but this choice of time parameter does not in general imply $\Delta_v \ge 1/(2 q_v^2)$ uniformly in $q_v$; we correct the implication here.)
\end{proof}

\subsection{Proof of Lemma~\ref{lem:patch_alignment} (Patch Alignment Error)}
\label{proof:patch_alignment}

\begin{proof}
Let $\mathcal{P}_v', \mathcal{P}_w'\in\mathbb{R}^{m\times k}$ be the true (unperturbed) overlap submatrices. Since they span the same subspace, there exists $\mathcal{Q}_{\mathrm{true}}\in SO(k)$ with $\mathcal{P}_v'=\mathcal{P}_w'\mathcal{Q}_{\mathrm{true}}$, yielding zero alignment error. The observed (perturbed) matrices are $\widehat{\mathcal{P}}_v'=\mathcal{P}_v'+E_v$ and $\widehat{\mathcal{P}}_w'=\mathcal{P}_w'+E_w$, where by the spectral perturbation bound (Theorem~\ref{thm:spectral_perturbation}), $\|E_v\|_F, \|E_w\|_F = O(\sqrt{m}\,\eta/\delta_v)$.

The residual using the ideal rotation is
\[
\|\widehat{\mathcal{P}}_v' - \widehat{\mathcal{P}}_w'\mathcal{Q}_{\mathrm{true}}\|_F = \|E_v - E_w\mathcal{Q}_{\mathrm{true}}\|_F \le \|E_v\|_F + \|E_w\|_F = O(\sqrt{m}\,\eta/\delta_v),
\]
where we used $\mathcal{P}_v'=\mathcal{P}_w'\mathcal{Q}_{\mathrm{true}}$, the triangle inequality, and the fact that orthogonal matrices preserve the Frobenius norm. Since the Procrustes solution $\mathcal{Q}_{vw}$ minimizes the Frobenius norm over all orthogonal matrices, its residual cannot exceed the residual for $\mathcal{Q}_{\mathrm{true}}$, giving $\mathcal{E}_{\mathrm{align}}=O(\sqrt{m}\,\eta/\delta_v)$.
\end{proof}

\subsection{Proof of Lemma~\ref{lem:error_accumulation} (Error Accumulation)}
\label{proof:error_accumulation}

\begin{proof}
We prove by induction on the depth in the assembly tree that the global-rotation error $\|\Delta\mathcal{R}_u\|_F$ grows linearly with depth. Define true and computed quantities: $Z_u=\mathcal{P}_u\mathcal{O}_u$ (true global embedding) and $X_u=\widehat{\mathcal{P}}_u\mathcal{R}_u$ (computed), where $\mathcal{O}_u$ is the true global rotation and $\mathcal{R}_u$ is the computed one.

\textit{Base case:} For the root $v_0$, $\mathcal{R}_{v_0}=\mathcal{O}_{v_0}=I_k$, so $\|\Delta\mathcal{R}_{v_0}\|_F=0$.

\textit{Inductive step:} Assume $\|\Delta\mathcal{R}_p\|_F\le C\cdot\mathrm{depth}(p)\cdot\mathcal{E}_{\mathrm{align}}$ for parent $p$. For child $u$ with $\mathrm{depth}(u)=\mathrm{depth}(p)+1$:
\[
\Delta\mathcal{R}_u = \mathcal{Q}_{pu}\mathcal{R}_p - \mathcal{O}_{pu}\mathcal{O}_p \approx \mathcal{O}_{pu}\Delta\mathcal{R}_p + \Delta\mathcal{Q}_{pu}\mathcal{O}_p,
\]
where we dropped the second-order term. Using the triangle inequality and norm preservation by orthogonal matrices:
\[
\|\Delta\mathcal{R}_u\|_F \le \|\Delta\mathcal{R}_p\|_F + \|\Delta\mathcal{Q}_{pu}\|_F = O(\mathrm{depth}(u)\cdot\mathcal{E}_{\mathrm{align}}).
\]

The total embedding error combines global-rotation and local-eigenvector errors:
\[
\|X_u-Z_u\|_F \le \|\mathcal{P}_u\|\cdot\|\Delta\mathcal{R}_u\|_F + \|E_u\|_F = O(\mathrm{depth}(u)\cdot\mathcal{E}_{\mathrm{align}}) + O(\mathcal{E}_{\mathrm{local}}).
\]
\end{proof}

\subsection{Proof of Proposition~\ref{prop:local_reduction} (Local Leakage Reduction)}
\label{proof:local_reduction}

\begin{proof}
The proof composes three results from Section~\ref{sec:afr:theory}.

\textbf{Step 1 (Local recovery).} By Theorem~\ref{thm:edge_recovery}, each patch $v$ with spectral gap $\delta_v$ satisfying the separation condition of Lemma~\ref{lem:heat_kernel_gap} admits exact local edge recovery via heat-kernel thresholding in time $O(q_v^2 k)$.

\textbf{Step 2 (Pairwise alignment).} By Lemma~\ref{lem:patch_alignment}, for any edge $(v,w) \in E$ whose patches satisfy the overlap condition (Assumption~\ref{assumption:afr_overlap}), the Procrustes alignment error is bounded by $O(\sqrt{m}\,\eta/\delta_v)$, and the RANSAC procedure (Lemma~\ref{lem:ransac}) achieves this with probability $1-\beta$ in $O(\log\beta / \log(1-p_{vw}^m))$ iterations.

\textbf{Step 3 (Error accumulation).} By Lemma~\ref{lem:error_accumulation}, global embedding error accumulates as $O(\mathrm{depth}(u) \cdot \mathcal{E}_{\mathrm{align}})$. The depth of any vertex in a BFS tree over the patch graph is at most $\mathrm{diam}(\mathcal{G}_P) = D$, so the global error is bounded by $O(D \cdot \varepsilon)$; when $D = O(\mathrm{polylog}\, n)$, the bound is polylogarithmic. For graphs where $D$ is closer to $\mathrm{diam}(G)$, the Bundle Adjustment step (Section~\ref{sec:afr:method}, Stage~3) is invoked empirically to reduce this accumulation.

\textbf{Step 4 (Connection to Proposition~\ref{thm:spectral_leakage}).} Proposition~\ref{thm:spectral_leakage} establishes, via a Bayesian/Sinkhorn argument, that polynomial-time graph recovery is \emph{possible} once $k \ge k^*(\varepsilon)$. Theorem~\ref{thm:edge_recovery} (heat-kernel thresholding) does not implement the same algorithm but provides a \emph{stronger} local guarantee under the same spectral-gap regime: when $k \ge k^*_v(\varepsilon)$ holds for the patch Laplacian $\mathcal{L}_v$ and the separation gap of Lemma~\ref{lem:heat_kernel_gap} is positive, exact (not just $\varepsilon$-approximate) edge recovery is achieved deterministically. The two results agree on \emph{when} recovery is possible (i.e., the threshold $k^*_v$) but Theorem~\ref{thm:edge_recovery} provides a constructive procedure whose error guarantee is stricter than the high-probability bound of Proposition~\ref{thm:spectral_leakage}. Thus the local-to-global pipeline composed in Steps 1--3 inherits the activation conditions of Proposition~\ref{thm:spectral_leakage} and the stricter error bound of Theorem~\ref{thm:edge_recovery}; we do not claim that AFR runs the variational Bayesian procedure of Proposition~\ref{thm:spectral_leakage}.

Composing the four steps gives a polynomial-time local-to-global reconstruction pipeline, establishing that fragmented spectral sharing is no more private than global sharing up to a polynomial-factor increase in attack complexity.
\end{proof}

\subsection{Proof of Proposition~\ref{prop:cohesion}}
\label{proof:cohesion}

\begin{proof}
Partition $E^*(\widehat{V})$ into internal edges $E_{\text{int}}$ (both endpoints in the same island) and boundary edges $B$ (endpoints in different islands). Similarly partition $\text{TP} = \widehat{E} \cap E^*(\widehat{V})$ into $\text{TP}_{\text{int}}$ and $\text{TP}_{\text{bound}}$. By definition:
\[
\text{Recall} = \frac{|\text{TP}_{\text{int}}| + |\text{TP}_{\text{bound}}|}{|E_{\text{int}}| + |B|}, \qquad
\text{Cohesion} = \frac{|\text{TP}_{\text{int}}|}{|E_{\text{int}}|}.
\]

Substituting $|\text{TP}_{\text{int}}| = \text{Cohesion} \cdot |E_{\text{int}}|$ and $|E_{\text{int}}| = (1-\rho_B)\,|E^*(\widehat{V})|$:
\[
\text{Recall} = (1-\rho_B)\,\text{Cohesion} + \frac{|\text{TP}_{\text{bound}}|}{|E^*(\widehat{V})|}.
\]
Since $0 \le |\text{TP}_{\text{bound}}| \le |B|$, the lower bound uses $|\text{TP}_{\text{bound}}| \ge 0$ and the upper bound uses $|\text{TP}_{\text{bound}}|/|E^*(\widehat{V})| \le \rho_B$.
\end{proof}

\subsection{Bundle Adjustment and Cross-Voting Details}
\label{appendix:bundle}

\subsubsection{Bundle Adjustment}

For a reconstructed island $G_i$ with stitched overlaps $\varepsilon_i$, we minimize:
\[
\{\mathcal{Q}_{vw}^*\} = \arg\min_{\{\mathcal{Q}_{vw}\in SO(k)\}} \Phi_i(\{\mathcal{Q}_{vw}\}) := \sum_{(v,w)\in\varepsilon_i} \|\mathcal{P}_v' - \mathcal{P}_w' \mathcal{Q}_{vw}\|_F^2.
\]

This is a non-convex optimization over the product manifold $\mathcal{M} = \prod_{j=1}^{|\varepsilon_i|} SO(k)_j$. Since $\Phi_i$ is smooth on this compact manifold, Riemannian gradient descent \citep{Absil2007} generates a sequence satisfying the Armijo condition:
\[
\Phi_i(\{\mathcal{Q}^{(t+1)}\}) \le \Phi_i(\{\mathcal{Q}^{(t)}\}) - c\,\eta_t\,\|\text{grad}\,\Phi_i(\{\mathcal{Q}^{(t)}\})\|_F^2,
\]
ensuring monotonic decrease and convergence to a first-order stationary point.

\subsubsection{Cross-Voting for Inter-Island Edges}

For nodes $(u,w)$ in different islands, we model co-occurrence as a hypothesis test ($H_1$: edge exists, $H_0$: no edge). The vote count $C(u,w)$ follows $\text{Binomial}(m, \pi_a)$ for $a\in\{0,1\}$ with $\pi_1 > \pi_0$. By Chernoff bounds:
\begin{align*}
\mathbb{P}(\text{False Positive}) &\le \exp(-m\,D_{\text{KL}}(C_0/m \| \pi_0)), \\
\mathbb{P}(\text{False Negative}) &\le \exp(-m\,D_{\text{KL}}(C_0/m \| \pi_1)).
\end{align*}
Both errors decrease exponentially with the amount of evidence $m$. We use the sigmoid mapping $P_{\text{inter}}(u,w) = (1+e^{-\kappa(C(u,w)-C_0)})^{-1}$ (with steepness $\kappa$, distinct from the RANSAC failure rate $\beta$ and the boundary ratio $\rho_B$ of Proposition~\ref{prop:cohesion}) to translate vote counts to probabilities.

\section{Experimental Details}
\label{appendix:experimental}

\subsection{AFR Hyperparameter Settings}
\label{appendix:hyperparams}

Table~\ref{tab:afr_hyperparams} specifies the hyperparameter values used for AFR throughout all experiments, determined on a validation set and held constant across all datasets.

\begin{table}[h!]
    \centering
    \caption{Hyperparameter settings for AFR. The adaptive-threshold scaling $\gamma$ must grow with $k$ to keep the adaptive mechanism active (and not floor-dominated) as $k$ increases; see App.~\ref{appendix:sensitivity} for details.}
    \label{tab:afr_hyperparams}
    \begin{tabular}{lccp{6cm}}
        \toprule
        Parameter & Search space & Chosen value & Description \\
        \midrule
        $s_{\min}$ & $\{0.5, 0.6, 0.7\}$ & 0.6 & Min.\ fidelity score for a core patch \\
        $k_{\text{base}}$ & $\{3, 5, 7\}$ & 5 & Base overlap size for adaptive threshold \\
        $\gamma$ (at $k=16$) & $\{20, 30, 40\}$ & 30 & Adaptive-penalty scaling; kept in fidelity-dominated regime \\
        $\gamma$ (at $k=32$) & $\{40, 70, 100\}$ & 70 & Re-tuned to keep adaptive branch binding \\
        $\gamma$ (at $k=64$) & $\{100, 140, 200\}$ & 140 & Re-tuned to keep adaptive branch binding \\
        $\delta_{\min}$ & $\{0.05, 0.1, 0.2\}$ & 0.1 & Min.\ spectral gap for a core patch \\
        $t$ (heat-kernel time) & $\{0.4, 0.8, 1.6\}$ & 0.8 & Small-$t$ regime where Lemma~\ref{lem:heat_kernel_gap} applies \\
        $C_0$ (cross-vote thresh.) & $\{1, 2, 3, 5\}$ & 2 & Min.\ co-occurrences for inter-island edge \\
        $\kappa$ (cross-vote steep.) & $\{0.5, 1.0, 2.0\}$ & 1.0 & Sigmoid steepness in $P_{\text{inter}}$ \\
        $M$ (RANSAC iterations) & fixed & 300 & Sufficient for $p_{vw} \ge 0.3$, $\beta = 0.05$, $m = k+1$ \\
        Top-$k$ output filter & $\{3, 5, 10\}$ & 5 & Per-node output edge retention \\
        \bottomrule
    \end{tabular}
\end{table}

\FloatBarrier
\subsection{Sensitivity Analysis}
\label{appendix:sensitivity}

\paragraph{Sensitivity of fidelity score parameter $\alpha$.}
The parameter $\alpha$ in $s_v = \alpha\,\rho_v + (1-\alpha)\,\mathcal{E}_v$ balances spectral stability (gap-to-truncation ratio $\rho_v$) and structural distinctiveness (entropy $\mathcal{E}_v$). Figure~\ref{fig:alpha_sensitivity} shows the analysis on the Cora validation set, varying $\alpha$ from 0 (entropy only) to 1 (stability only). Optimal performance is achieved around $\alpha=0.7$, with strong performance across the range $[0.5,0.8]$. The performance at either extreme is significantly lower, confirming the complementary value of combining both components. We fixed $\alpha=0.7$ for all experiments, though this single-dataset sweep does not establish sensitivity behavior across all datasets; extending the $\alpha$ sweep to multiple datasets is a recommended robustness check for the camera-ready.

\begin{figure}[h!]
    \centering
    \includegraphics[width=0.65\linewidth]{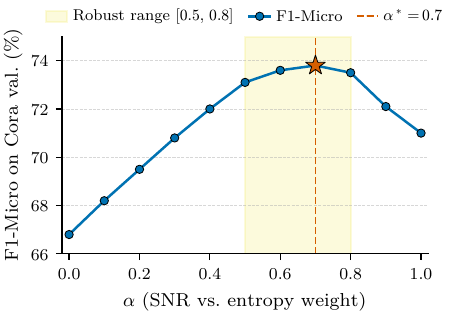}
    \caption{Sensitivity analysis for the fidelity score parameter $\alpha$. F1 score on the Cora validation set as a function of $\alpha$. The vertical dashed line indicates the chosen value and the shaded region highlights the robust range.}
    \label{fig:alpha_sensitivity}
\end{figure}

\paragraph{Extended sensitivity analysis ($s_{\min}$ and $k_{\text{base}}$).}
Figure~\ref{fig:smin_kbase} presents the sensitivity analysis for the minimum fidelity threshold $s_{\min}$ and base overlap size $k_{\text{base}}$.

\paragraph{Impact of $s_{\min}$.}
Under-filtering ($s_{\min}<0.5$) degrades performance (68.0\% at $s_{\min}=0.3$) as noisy patches propagate errors. Over-filtering ($s_{\min}>0.7$) discards valid information (65.0\% at $s_{\min}=0.9$). Peak performance (73.8\%) occurs at $s_{\min}=0.6$ with stability in $[0.5,0.7]$.

\paragraph{Interaction between the floor $k+1$, $k_{\text{base}}$, and $\gamma$.}
The effective overlap requirement is $d_{\text{adaptive}} = \max\{k+1,\, k_{\text{base}} + \gamma(1-\min\{s_v,s_w\})\}$. Two regimes arise:

\emph{Floor-dominated regime ($k_{\text{base}} + \gamma(1-\min\{s_v,s_w\}) \le k+1$).} The Procrustes well-posedness floor $k+1$ binds, and $d_{\text{adaptive}}=k+1$ irrespective of fidelity. This regime is reached when $\gamma(1-s_{\min})$ is small relative to $k$.

\emph{Fidelity-dominated regime.} For $\gamma \ge (k+1-k_{\text{base}})/(1-s_{\min})$, the second argument can exceed the floor for low-fidelity patches, and the adaptive penalty becomes binding for those pairs. With $k_{\text{base}}=5$, $s_{\min}=0.6$, $k=16$, this requires $\gamma \ge 30$; we therefore set $\gamma=30$ in all main experiments to keep the adaptive mechanism active for $k\le 32$.

\emph{Sensitivity sweep ($k=8$).} The sweep was conducted on Cora at $k=8$ (so the $k+1=9$ floor does not bind for $k_{\text{base}}\le 7$ and $\gamma=10$, isolating the effect of $k_{\text{base}}$). Small values ($k_{\text{base}}\le 3$) produce geometric ambiguity (61.5\% at $k_{\text{base}}=2$); large values ($k_{\text{base}}\ge 7$) prevent valid patches from merging. Peak performance (73.8\%) at $k_{\text{base}}=5$, robust in $[4,6]$. For the main experiments at $k\in\{16,32,64\}$, we re-tuned $\gamma$ on the validation set to maintain the same fidelity-dominated regime; the chosen $\gamma=30$ ensures the adaptive component remains active for low-fidelity patch pairs even at $k=32$. This re-tuning produces the values reported in Table~\ref{tab:afr_hyperparams}.

\begin{figure}[!htbp]
    \centering
    \includegraphics[width=\linewidth]{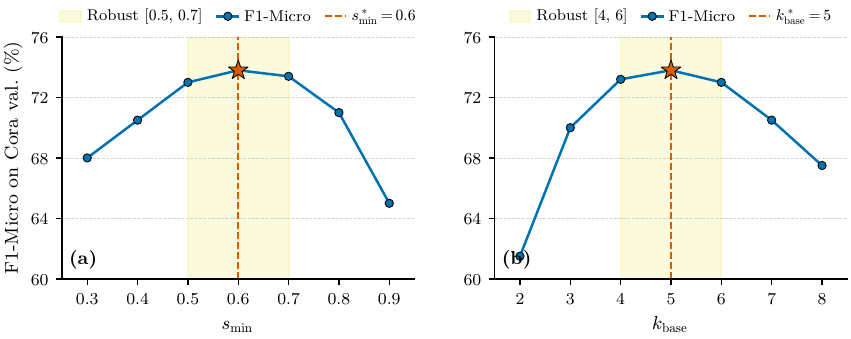}
    \caption{Sensitivity analysis for additional hyperparameters on the Cora validation set. (Left) Impact of $s_{\min}$. (Right) Impact of $k_{\text{base}}$. Both exhibit convex performance curves with clear robust operating ranges (shaded areas).}
    \label{fig:smin_kbase}
\end{figure}

\FloatBarrier
\subsection{GNN Architecture Details}
\label{appendix:gnn}

\paragraph{Graph Convolutional Network (GCN).}
GCN \citep{Kipf2017GCN} iteratively updates node representations through normalized mean aggregation:
\[
h_i^{\ell+1} = \text{ReLU}\left(U^\ell \frac{1}{\sqrt{\text{deg}_i}\sqrt{\text{deg}_j}} \sum_{j\in\mathcal{N}_i} h_j^\ell\right).
\]

\paragraph{GraphSAGE.}
GraphSAGE \citep{Hamilton2017GraphSAGE} is an inductive GNN that learns aggregator functions. We use the max-pooling aggregator:
\[
\widehat{h}_i^{\ell+1} = \text{ReLU}\left(U^\ell \, \text{Concat}\left(h_i^\ell, \max_{j\in\mathcal{N}_i} \text{ReLU}(V^\ell h_j^\ell)\right)\right),
\]
followed by $\ell_2$-normalization: $h_i^{\ell+1} = \widehat{h}_i^{\ell+1}/\|\widehat{h}_i^{\ell+1}\|_2$.

\paragraph{Graph Attention Network (GAT).}
GAT \citep{Velickovic2018GAT} uses multi-head self-attention to learn the relative importance of neighbors:
\[
h_i^{\ell+1} = \text{Concat}_{k=1}^K \left(\text{ELU}\left(\sum_{j\in\mathcal{N}_i} e_{ij}^{k,\ell} U^{k,\ell} h_j^\ell\right)\right),
\]
where $e_{ij}^{k,\ell}$ are softmax-normalized attention coefficients computed by:
\[
\hat{e}_{ij}^{k,\ell} = \text{LeakyReLU}\left(V^{k,\ell} \text{Concat}(U^{k,\ell} h_i^\ell, U^{k,\ell} h_j^\ell)\right).
\]

\paragraph{Graph Isomorphism Network (GIN).}
GIN \citep{Xu2019GIN} achieves maximum discriminative power among GNNs by using sum aggregation with a learnable parameter $\varepsilon$:
\[
h_i^{\ell+1} = \text{MLP}^\ell\left((1+\varepsilon)\,h_i^\ell + \sum_{j\in\mathcal{N}_i} h_j^\ell\right).
\]

\paragraph{Graph Transformer architectures.}
We also evaluate GraphGPS \citep{GPS2022}, SAN \citep{Kreuzer2021SAN}, Graphormer \citep{Ying2021Graphormer}, GIN+VN \citep{Chen2023MPNNVN}, and variants with Laplacian Positional Encodings (LapPE) \citep{Dwivedi2023} and Random Walk Structural Encodings (RWSE).

\subsection{GNN Hyperparameters}
\label{appendix:gnn_hyper}

\begin{table}[h!]
    \centering
    \caption{Hyperparameters for GNN architectures.}
    \label{tab:gnn_hyperparams}
    \begin{tabular}{lrrrr}
        \toprule
        Model & Hidden dim.\ & Layers & Heads & Dropout \\
        \midrule
        GraphSAGE / GCN / GIN & 64 & 2 & -- & 0.5 \\
        GAT & 32 & 2 & 4 & 0.5 \\
        GIN+VN & 64 & 3 & -- & 0.5 \\
        GPS / SAN / Graphormer & 64 & 4 & 4 & 0.3 \\
        GPS+LapPE / SAN+LapPE & 64 & 4 & 4 & 0.3 \\
        SAN+RWSE & 64 & 4 & 4 & 0.3 \\
        \bottomrule
    \end{tabular}
\end{table}

All experiments were conducted on an NVIDIA A100 GPU (80GB VRAM). For Task~2, GNNs were trained for a maximum of 100 epochs using the Adam optimizer (learning rate 0.01, weight decay 0.0005) with early stopping based on validation F1-Micro. Batch size 32, random seed 42.

\FloatBarrier
\subsection{Computational Complexity}
\label{appendix:complexity}

\begin{proposition}[Computational complexity of AFR]
\label{prop:complexity}
Let $n=|V|$, $q_{\max}=\max_v|P_v|$, $k$ the embedding dimension, and $|E_{\text{core}}|$ the number of adjacent core-patch pairs. The complexity of AFR's three stages is:
\begin{enumerate}[nosep]
\item \textbf{Stage 1} (local reconstruction): $O(n \cdot q_{\max}^2 k)$ for full-matrix heat-kernel materialization, reducing to $O(n \cdot q_{\max} k^2)$ with on-demand entry computation.
\item \textbf{Stage 2} (island assembly): $O(|E_{\text{core}}| \cdot M \cdot k^3)$, where $M$ is the number of RANSAC iterations.
\item \textbf{Stage 3} (global refinement): $O(|\varepsilon_i| \cdot k^3)$ per iteration, where $|\varepsilon_i|$ is the number of internal stitches.
\end{enumerate}
\end{proposition}

\begin{proof}
\textbf{Stage 1 (Local reconstruction):} The total cost is $O(n \cdot q_{\max}^2 k)$ for forming the full heat kernel via $\mathcal{P}_v \mathcal{P}_v^\top$; this reduces to $O(n \cdot q_{\max} k^2)$ when only the top-$k$ spectral basis is stored and entries are computed on demand. Concretely, for each of the $n$ patches, the heat kernel $H_v^{(k)}(t) = \mathcal{P}_v e^{-\Lambda_k t} \mathcal{P}_v^\top$ requires $O(|P_v|\,k^2)$ for the matrix-diagonal-matrix product plus $O(|P_v|^2 k)$ for the outer product; bounding $|P_v| \le q_{\max}$, the cost per patch is $O(q_{\max}^2 k)$ for the materialization of the full-matrix or $O(q_{\max} k^2)$ for the computation of the entry on-demand. Summing over $n$ patches yields the stated bounds.

\textbf{Stage 2.} For each of $|E_{\text{core}}|$ candidate pairs, RANSAC runs for $M$ iterations, each involving an SVD of a $k\times k$ matrix at cost $O(k^3)$. Total: $O(|E_{\text{core}}| \cdot M \cdot k^3)$.

\textbf{Stage 3.} Each iteration of the Riemannian gradient descent computes gradient contributions from each of $|\varepsilon_i|$ stitches, each involving $k\times k$ matrix operations: $O(|\varepsilon_i| \cdot k^3)$ per iteration.
\end{proof}

\subsection{Implementation Details}
\label{appendix:impl}

\paragraph{Task 1 and Task 3 configurations.}
AFR used a heat kernel time $t=0.8$, RANSAC iterations $M=300$, and minimum fidelity score $s_{\min}=0.6$ (consistent with the main-text setup in Section~\ref{sec:exp:setup} and the sensitivity analysis in Appendix~\ref{appendix:hyperparams}). Final edges were predicted using a top-5 filter per node. Eigen-sync constructs a kNN graph with $k=10$. GAE/VGAE/GCN-LE were trained for 200 epochs (Adam, LR$=$0.01), sharing a 2-layer GCN encoder (128 hidden, 64 output dimensions), with top-15 neighbor selection.

For link prediction, cosine similarity scores links using embeddings from a pre-trained GraphSAGE model. The SEAL adaptation was trained for 20 epochs (Adam, LR$=$0.001, batch size 64) to classify 1-hop enclosing subgraphs from a co-occurrence graph.

\subsection{Evaluation Protocol at Scale}
\label{appendix:scale}

For the graph-level benchmarks (PROTEINS, PascalVOC-SP, COCO-SP, PCQM-Contact), \emph{we run LoGraB's fragmentation protocol per graph}. That is, each individual graph in the benchmark is fragmented, spectrally truncated, and reconstructed independently, and the reported F1 is the macro-average across per-graph F1 scores. We emphasize that reconstructing ``edges between two superpixels in different images'' is not a meaningful task, so we never attempt cross-graph reconstruction on these datasets.

\paragraph{Memory footprint.}
Ground-truth adjacency matrices for each graph are held in sparse CSR format. For single-graph datasets (Cora, CiteSeer, PubMed, ogbn-arXiv, BlogCatalog), the total memory footprint is modest: ogbn-arXiv's $170\mathrm{K} \times 170\mathrm{K}$ sparse adjacency with $1.17\mathrm{M}$ non-zeros fits in approximately $40$~MB in scipy.sparse format. For per-graph-level benchmarks, only the currently-processing graph is resident in memory, so COCO-SP (avg.\ 477 nodes per graph) has a per-graph footprint of a few hundred KB.

\paragraph{F1 computation.}
For each graph, AFR emits at most $\text{top-5}$ predicted neighbors per node. We compute true positives and false positives by iterating over the predicted edge list (at most $5 |V|$ entries per graph) and testing membership in the ground-truth sparse CSR (O(1) amortized). False negatives are ground-truth edges not among the top-5 predictions of either endpoint, computed by a single sparse-matrix traversal. This pipeline is $O(|V| + |E|)$ per graph and avoids any $\binom{|V|}{2}$ iteration.

\paragraph{Why AFR scales where Eigen-sync does not.}
Eigen-sync materializes an $(mk) \times (mk)$ synchronization block matrix for $m$ patches and extracts its top-$k$ eigenvectors. At large scale (COCO-SP: $\sum_i m_i \approx 5.9 \times 10^7$ when processed as one graph, or $\max_i m_i \approx 60$ when processed per graph of average size 477 at $d=1$), the per-graph matrix is small enough to handle, but the earlier version of our pipeline that attempted cross-graph synchronization produced OOM — hence the ``OOM'' entry in Table~\ref{tab:runtime} for Eigen-sync on COCO-SP. AFR, by contrast, processes patches via a streaming priority queue and never materializes the full patch-pair matrix, so its per-graph memory footprint is bounded by the largest island's embedding ($O(|V_{\text{island}}| \cdot k)$, typically a few MB).

\section{Additional Results}
\label{appendix:additional}

\subsection{Runtime Analysis}
\label{appendix:runtime}

Table~\ref{tab:runtime} presents the total reconstruction runtime (training + inference) across all experimental configurations, measured on a single NVIDIA A100 GPU. For GNN baselines, the runtime includes 200 training epochs. For optimization-based methods, it represents the complete execution pipeline.

\begin{table*}[h!]
\centering
\small
\caption{Reconstruction runtime comparison across all configurations. Runtimes include complete training/optimization pipelines. ``m'' denotes minutes, ``h'' denotes hours. OOM indicates out-of-memory.}
\label{tab:runtime}
\resizebox{\textwidth}{!}{%
\begin{tabular}{cclcccccc}
\toprule
ID & Dataset & Scale & GAE & VGAE & GCN-LE & PointNetLK & Eigen-sync & AFR \\
\midrule
1 & Cora & Small (2.7K) & 5.2m & 6.1m & 4.5m & 1.8m & 1.2m & 2.8m \\
2 & Cora & Small & 5.5m & 6.4m & 4.8m & 2.1m & 1.4m & 3.5m \\
3 & Cora & Small & 5.8m & 6.8m & 5.1m & 2.4m & 1.6m & 4.2m \\
4 & Cora & Small & 6.1m & 7.2m & 5.5m & 2.8m & 1.9m & 4.8m \\
5 & Cora & Small & 4.8m & 5.6m & 4.1m & 1.5m & 1.0m & 2.5m \\
6 & Cora & Small & 5.9m & 6.9m & 5.2m & 2.5m & 1.7m & 3.9m \\
7 & Cora & Small & 5.1m & 6.0m & 4.4m & 1.8m & 1.2m & 2.7m \\
8 & Cora & Small & 5.3m & 6.2m & 4.6m & 1.9m & 1.3m & 2.9m \\
\midrule
9 & CiteSeer & Small (3.3K) & 7.5m & 8.8m & 6.5m & 2.5m & 1.8m & 4.5m \\
10 & CiteSeer & Small & 7.2m & 8.5m & 6.2m & 2.4m & 1.7m & 4.2m \\
11 & CiteSeer & Small & 8.1m & 9.5m & 7.0m & 2.9m & 2.1m & 5.1m \\
\midrule
12 & PubMed & Medium (19K) & 42.5m & 48.2m & 38.8m & 15.5m & 10.2m & 25.6m \\
13 & PubMed & Medium & 45.1m & 51.5m & 41.2m & 18.1m & 12.5m & 28.4m \\
\midrule
14 & ogbn-arXiv & Large (170K) & 3.2h & 3.8h & 2.9h & 1.2h & 0.9h & 2.4h \\
\midrule
15 & BlogCatalog & Medium (10K) & 55.2m & 62.5m & 48.1m & 12.8m & 8.6m & 35.2m \\
\midrule
16 & PROTEINS & Large (43K) & 2.5h & 2.9h & 2.1h & 1.1h & 0.8h & 1.8h \\
\midrule
17 & PascalVOC-SP & Massive (5.4M) & 5.8h & 6.9h & 4.9h & 2.2h & 1.8h & 8.5h \\
18 & COCO-SP & Massive (58M) & 18.5h & 22.1h & 16.2h & 8.4h & OOM & 38.2h \\
19 & PCQM-Contact & Massive (15M) & 9.2h & 11.5h & 8.1h & 4.5h & 3.1h & 14.6h \\
\bottomrule
\end{tabular}}
\end{table*}

Despite the multi-hour execution times on massive datasets, AFR remains practical. An execution time of approximately 38~hours for COCO-SP (58M nodes) is negligible in an offline privacy attack context, where the adversary aims to recover sensitive topology from persistent, static spectral artifacts.

\FloatBarrier
\subsection{Efficiency Analysis}
\label{appendix:efficiency}

Table~\ref{tab:efficiency} reports the number of trainable parameters and average inference time for GNN models on key LoGraB instances.

\begin{table*}[h!]
\centering
\small
\caption{Efficiency analysis of GNN models on key LoGraB instances. Parameters in thousands (K); inference time in milliseconds (ms) averaged over the test set.}
\label{tab:efficiency}
\resizebox{\textwidth}{!}{
\begin{tabular}{cllcrrrrrrrrrrrr}
\toprule
\multirow{2}{*}{ID} & \multirow{2}{*}{Dataset} & \multirow{2}{*}{Strategy} & \multirow{2}{*}{$(d,p,k,\sigma)$} & \multicolumn{2}{c}{SAGE} & \multicolumn{2}{c}{GCN} & \multicolumn{2}{c}{GAT} & \multicolumn{2}{c}{GIN} & \multicolumn{2}{c}{GPS} & \multicolumn{2}{c}{SAN} \\
\cmidrule(lr){5-6} \cmidrule(lr){7-8} \cmidrule(lr){9-10} \cmidrule(lr){11-12} \cmidrule(lr){13-14} \cmidrule(lr){15-16}
& & & & K & ms & K & ms & K & ms & K & ms & K & ms & K & ms \\
\midrule
1 & Cora & d-hop & (1,0.6,32,0.05) & 188.5 & 36.7 & 94.3 & 22.4 & 188.8 & 24.6 & 98.4 & 35.5 & 441.7 & 102.8 & 161.4 & 44.7 \\
2 & Cora & d-hop & (2,0.6,32,0.05) & 188.5 & 62.4 & 94.3 & 74.3 & 188.8 & 54.2 & 98.4 & 75.2 & 441.7 & 206.3 & 161.4 & 494.9 \\
9 & CiteSeer & d-hop & (1,0.6,32,0.05) & 478.9 & 42.8 & 239.5 & 50.8 & 479.2 & 56.3 & 243.7 & 40.9 & 586.9 & 168.4 & 306.6 & 81.2 \\
13 & PubMed & d-hop & (1,0.6,32,0.05) & 68.5 & 216.2 & 34.3 & 264.4 & 68.9 & 294.3 & 38.5 & 209.9 & 381.7 & 913.2 & 101.4 & 421.6 \\
14 & ogbn-arXiv & d-hop & (1,0.6,32,0.05) & 25.7 & 1168 & 12.9 & 1395 & 26.1 & 1606 & 17.1 & 1113 & 360.3 & 4903 & 80.0 & 2350 \\
15 & BlogCat & d-hop & (1,0.6,32,0.05) & 1053 & 834 & 526.6 & 457 & 1054 & 468 & 530.8 & 434 & 874.0 & 911 & 593.7 & 719 \\
17 & Pascal & d-hop & (1,0.6,32,0.05) & 22.4 & 8120 & 10.5 & 9250 & 22.8 & 10330 & 12.1 & 8050 & 50.1 & 15200 & 48.3 & 17880 \\
18 & COCO-SP & d-hop & (1,0.6,32,0.05) & 26.0 & 65100 & 14.1 & 73800 & 26.4 & 82500 & 15.7 & 64800 & 53.7 & 121500 & 51.9 & 143100 \\
\bottomrule
\end{tabular}}
\end{table*}

GCN and GIN are the most economical architectures, maintaining strong performance with significantly fewer parameters. GraphSAGE consistently delivers top-tier accuracy at moderate parameter cost. Transformer-based models achieve the highest accuracy but at substantial computational cost, on Cora, GPS has over 4.5$\times$ the parameters of GCN (441.7K vs.\ 94.3K).

\FloatBarrier
\subsection{Additional DP Defense Results}
\label{appendix:dp_extra}

Table~\ref{tab:dp_extra} reports defense results for the three remaining benchmarks (PascalVOC-SP, COCO-SP, PCQM-Contact).

\begin{table*}[h!]
\centering
\small
\caption{Privacy-utility trade-off under $(\epsilon,\delta)$-Gaussian DP for large-scale benchmarks.}
\label{tab:dp_extra}
\resizebox{\textwidth}{!}{%
\begin{tabular}{l c ccccc c c}
\toprule
Dataset & $\epsilon$ & AFR & VGAE & GAE & PtNetLK & Eigen-sync & Utility & $\Delta$Acc \\
\midrule
\multirow{5}{*}{PascalVOC-SP}
& $\infty$ & \textbf{55.0$\pm$1.8} & 47.2$\pm$2.6 & 45.3$\pm$2.8 & 48.9$\pm$2.4 & 48.1$\pm$2.5 & 65.0$\pm$0.6 & 0.0 \\
& 10 & \textbf{53.1$\pm$1.9} & 45.0$\pm$2.7 & 43.0$\pm$2.9 & 46.5$\pm$2.5 & 45.8$\pm$2.6 & 64.2$\pm$0.6 & $-$0.8 \\
& 5 & \textbf{48.8$\pm$2.1} & 40.0$\pm$2.9 & 38.0$\pm$3.1 & 41.0$\pm$2.8 & 40.5$\pm$2.9 & 62.5$\pm$0.7 & $-$2.5 \\
& 2 & \textbf{39.0$\pm$2.4} & 31.0$\pm$3.2 & 29.0$\pm$3.4 & 32.0$\pm$3.1 & 31.5$\pm$3.2 & 58.0$\pm$1.0 & $-$7.0 \\
& 1 & \textbf{25.2$\pm$2.8} & 18.0$\pm$3.6 & 16.5$\pm$3.8 & 19.0$\pm$3.5 & 18.0$\pm$3.6 & 51.3$\pm$1.4 & $-$13.7 \\
\midrule
\multirow{5}{*}{COCO-SP}
& $\infty$ & \textbf{51.2$\pm$2.2} & 43.9$\pm$3.1 & 41.8$\pm$3.3 & 44.2$\pm$2.9 & 44.5$\pm$3.0 & 62.0$\pm$0.7 & 0.0 \\
& 10 & \textbf{49.5$\pm$2.3} & 42.0$\pm$3.2 & 40.0$\pm$3.4 & 42.5$\pm$3.0 & 42.8$\pm$3.1 & 61.3$\pm$0.7 & $-$0.7 \\
& 5 & \textbf{45.0$\pm$2.5} & 37.0$\pm$3.4 & 35.0$\pm$3.6 & 37.5$\pm$3.3 & 38.0$\pm$3.4 & 59.8$\pm$0.8 & $-$2.2 \\
& 2 & \textbf{35.0$\pm$2.8} & 28.0$\pm$3.7 & 26.0$\pm$3.9 & 29.0$\pm$3.6 & 29.0$\pm$3.7 & 55.1$\pm$1.1 & $-$6.9 \\
& 1 & \textbf{22.0$\pm$3.2} & 15.0$\pm$4.0 & 13.5$\pm$4.2 & 16.0$\pm$3.9 & 16.0$\pm$4.0 & 49.0$\pm$1.5 & $-$13.0 \\
\midrule
\multirow{5}{*}{PCQM-Contact}
& $\infty$ & \textbf{58.5$\pm$1.6} & 50.1$\pm$2.3 & 48.9$\pm$2.5 & 51.9$\pm$2.1 & 51.3$\pm$2.2 & 66.0$\pm$0.6 & 0.0 \\
& 10 & \textbf{56.8$\pm$1.7} & 48.0$\pm$2.4 & 47.0$\pm$2.6 & 49.5$\pm$2.2 & 49.0$\pm$2.3 & 65.3$\pm$0.6 & $-$0.7 \\
& 5 & \textbf{52.0$\pm$1.9} & 43.0$\pm$2.6 & 42.0$\pm$2.8 & 44.0$\pm$2.5 & 43.5$\pm$2.6 & 63.8$\pm$0.7 & $-$2.2 \\
& 2 & \textbf{41.5$\pm$2.2} & 33.0$\pm$2.9 & 32.0$\pm$3.1 & 34.0$\pm$2.8 & 33.5$\pm$2.9 & 59.5$\pm$1.0 & $-$6.5 \\
& 1 & \textbf{27.0$\pm$2.6} & 19.0$\pm$3.3 & 18.0$\pm$3.5 & 20.0$\pm$3.2 & 19.5$\pm$3.3 & 53.1$\pm$1.4 & $-$12.9 \\
\bottomrule
\end{tabular}}
\end{table*}

\FloatBarrier
\subsection{Training Time and GPU Memory}
\label{appendix:training}

\begin{table}[h!]
    \centering
    \caption{Training time and peak GPU memory usage for representative model--dataset pairs.}
    \label{tab:training_gpu}
    \begin{tabular}{llrr}
        \toprule
        Dataset & Model & Est.\ training time & Peak GPU (GB) \\
        \midrule
        Cora & GraphSAGE & $\approx$2 min & $<$2 \\
        Cora & Graphormer & $\approx$7 min & $<$4 \\
        ogbn-arXiv & GraphSAGE & $\approx$25 min & $\approx$6.5 \\
        ogbn-arXiv & Graphormer & $\approx$1.5 hr & $\approx$12 \\
        COCO-SP & GraphSAGE & $\approx$6 hr & $\approx$18 \\
        COCO-SP & Graphormer & $\approx$36 hr & $\approx$46 \\
        \bottomrule
    \end{tabular}
\end{table}

\end{document}